\documentclass[10pt,journal,compsoc]{IEEEtran}
\usepackage{amsmath}
\usepackage{mathrsfs}
\usepackage[tight,footnotesize]{subfigure}
\usepackage{graphicx}
\usepackage{multirow}
\usepackage{amsthm}
\usepackage{graphicx}
\usepackage{array}
\usepackage{multirow}
\usepackage{mdwmath}
\usepackage{mdwtab}
\usepackage{amsfonts}
\usepackage{algorithm}
\usepackage{algorithmic}
\usepackage{url}
\usepackage{enumerate}
\usepackage[utf8]{inputenc}
\usepackage{collref}
\usepackage{balance}

\newtheorem{definition}{Definition}
\newtheorem{theorem}{Theorem}
\newtheorem{lemma}{Lemma}

\newtheorem{property}{Property}

\begin{document}

\title{Regularized Orthogonal Tensor Decompositions for Multi-Relational Learning}

\author{Fanhua~Shang,~\IEEEmembership{Member,~IEEE,}
James~Cheng,
and~Hong~Cheng
\IEEEcompsocitemizethanks{\IEEEcompsocthanksitem F.\ Shang (Corresponding author) and J.\ Cheng are with the Department of Computer Science and Engineering, The Chinese University of Hong Kong, Shatin, N. T., Hong Kong. E-mail: \{fhshang, jcheng\}@cse.cuhk.edu.hk.\protect\\
\IEEEcompsocthanksitem H.\ Cheng is with the Department of Systems Engineering and Engineering Management, The Chinese University of Hong Kong, Shatin, N. T., Hong Kong. E-mail: hcheng@se.cuhk.edu.hk.}
\thanks{Manuscript received December 6, 2014; revised July 3, 2015.}}

\markboth{Journal of \LaTeX\ Class Files,~Vol.~13, No.~9, September~2014}%
{Shell \MakeLowercase{\textit{et al.}}: Bare Advanced Demo of IEEEtran.cls for Journals}

\IEEEtitleabstractindextext{
\begin{abstract}
Multi-relational learning has received lots of attention from researchers in various research communities. Most existing methods either suffer from superlinear per-iteration cost, or are sensitive to the given ranks. To address both issues, we propose a scalable core tensor trace norm Regularized Orthogonal Iteration Decomposition (ROID) method for full or incomplete tensor analytics, which can be generalized as a graph Laplacian regularized version by using auxiliary information or a sparse higher-order orthogonal iteration (SHOOI) version. We first induce the equivalence relation of the Schatten $p$-norm ($0\!<\!p\!<\!\infty$) of a low multi-linear rank tensor and its core tensor. Then we achieve a much smaller matrix trace norm minimization problem. Finally, we develop two efficient augmented Lagrange multiplier algorithms to solve our problems with convergence guarantees. Extensive experiments using both real and synthetic datasets, even though with only a few observations, verified both the efficiency and effectiveness of our methods.
\end{abstract}

}

\maketitle

\IEEEdisplaynontitleabstractindextext
\IEEEpeerreviewmaketitle

\IEEEraisesectionheading{\section{Introduction}\label{sec:introduction}}

\IEEEPARstart{R}{elational} learning is becoming increasingly important because of the high value hidden in relational data and also of its many applications in various domains such as social networks, the semantic web, bioinformatics, and the linked data cloud~\cite{nickel:twm}. A class of relational learning methods focus mostly on the problem of modeling a single relation type, such as relational learning from latent attributes~\cite{kok:spi, singh:cmf}, which models relations between objects as resulting from intrinsic latent attributes of these objects. But in reality, relational data typically involve multiple types of relations between objects or attributes, which can themselves be similar. For example, in social networks~\cite{jenatton:lfm}, relationships between individuals may be personal, familial, or professional. This type of relational data learning is often referred to as multi-relational learning (MRL), which needs to model large-scale sparse relational databases efficiently~\cite{getoor:srl}.

People usually make use of the semantic web's RDF formalism to represent relational data, where relations are modeled as triples of the form (subject, relation, object), and a relation either denotes the relationship between two entries or between an entity and an attribute value. Considering the multiple types of relationships, it is a more natural way stacking the matrices of observed relationships into one big sparse three-order tensor. Fig.\ \ref{fig1} shows an illustration of this modeling method. In recent years, tensors have become ubiquitous such as multi-channel images and videos, and become popular due to the ability to discover complex and interesting latent structures and correlations of data~\cite{kolda:tda, sun:multivis, acar:stf, xu:bcdm}. Recently there is a growing interest in tensor methods for link prediction tasks, partially due to their natural representation of multi-relational data.

\begin{figure}[t]
\centering
\includegraphics[width=0.36\linewidth]{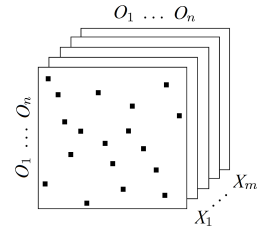}
\vspace{-2mm}
\caption{Tensor model for sparse multi-relational data. $O_{1},\ldots,O_{n}$ denote the objects, while $X_{1},\ldots, X_{m}$ denote the relations.}
\label{fig1}
\end{figure}

Tensor decomposition~\cite{tucker:tmfa, harshman:fpp, harshman:par}, \cite{yilmaz:gctf} is a popular tool for multi-relational prediction problems~\cite{acar:mda}, \cite{liu:fmtnm}. For example, Bader \emph{et al.}~\cite{bader:tasg} proposed a three-way component decomposition model for analyzing intrinsically asymmetric relationships. In addition, Nickel \emph{et al.}~\cite{nickel:twm} incorporated collective learning into the tensor factorization, which is designed to account for the inherent structure of relational data. Two of the most popular tensor factorizations are the Tucker decomposition~\cite{tucker:tmfa} and the CANDECOMP/PARAFAC (CP) decomposition~\cite{harshman:fpp}. To address incomplete tensor estimation, two weighted alternating least-squares methods~\cite{acar:stf, filipovic:tftc} were proposed. However, these methods require the ability to reliably estimate the rank of the involved tensor~\cite{gandy:tc, liu:ghooi}.

Recently, the low rank tensor recovery problem has been intensively studied. Liu \emph{et al.}~\cite{liu:tc} first extended the trace norm (also known as the nuclear norm~\cite{fazel:rmh} or the Schatten $1$-norm~\cite{liu:ghooi}) regularization for partially observed low multi-linear rank tensor recovery. Then the tensor recovery problem is transformed into a convex combination of trace norm minimization of the matrix unfolding along each mode. More recently, in Liu \emph{et al}.'s subsequent paper~\cite{liu:tcem}, they proposed three efficient algorithms to solve the low multi-linear rank tensor completion problem. Some similar algorithms can also be found in~\cite{gandy:tc}, \cite{signoretto:lt}, \cite{yang:fpim}, \cite{tomioka:ctd}. In addition, there are some theoretical developments that guarantee the reconstruction of a low rank tensor from partial measurements by solving trace norm minimization under some reasonable conditions~\cite{tomioka:ctd, huang:lrtr, mu:sd, liu:tld}. However, the tensor trace norm minimization problems have to be solved iteratively and involve multiple singular value decompositions (SVDs) in each iteration. Therefore, existing algorithms suffer from high computational cost, making them impractical for real-world applications~\cite{liu:ghooi, shang:hotd}.

To address both of the issues mentioned above, i.e., the \emph{robustness} of given ranks and the computational \emph{efficiency}, we propose a scalable core tensor trace norm Regularized Orthogonal Iteration Decomposition (ROID) method for full or incomplete tensor analytics. We first induce the equivalence relation of the Schatten $p$-norm ($0\!<\!p\!<\!\infty$) of a low multi-linear rank tensor and its core tensor. We use the trace norm of the core tensor to replace that of the whole tensor, and then achieve a much smaller scale matrix trace norm minimization problem. In particular, our ROID method is generalized as a graph Laplacian regularized version by using auxiliary information from the relationships or a sparse higher-order orthogonal iteration (SHOOI) version. Finally, we develop two efficient augmented Lagrange multiplier (ALM) algorithms for our problems. Moreover, we theoretically analyze the convergence property of our algorithms. Our experimental results on real-world datasets verified both the efficiency and effectiveness of our methods.

The rest of the paper is organized as follows. We review preliminaries and related work in Section 2. In Section 3, we propose two novel core tensor trace norm regularized tensor decomposition models, and develop two efficient ALM algorithms and extend one algorithm to solve the SHOOI problem in Section 4. We provide the theoretical analysis of our algorithms in Section 5. We report the experimental results in Section 6. In Section 7, we conclude this paper and point out some potential extensions for future work.

\section{Notations and Problem Formulations}
A third-order tensor is denoted by a calligraphic letter, e.g., $\mathcal{X}\!\in\! \mathbb{R}^{{I_{1}}\times{I_{2}}\times{I_{3}}}$, and its entries are denoted as $x_{{i_{1}}{i_{2}}{i_{3}}}$, where $i_{n}\!\in\!\{1,\ldots,I_{n}\}$ for $1\!\leq\! n\!\leq\! 3$. Fibers are the higher-order analogue of matrix rows and columns. The mode-\emph{n} fibers of a third-order tensor are $\mathrm{x}_{{:}{i_{2}}{i_{3}}}$, $\mathrm{x}_{{i_{1}}{:}{i_{3}}}$ and $\mathrm{x}_{{i_{1}}{i_{2}}{:}}$, respectively.

The mode-$n$ unfolding, also known as matricization, of a third-order tensor $\mathcal{X}\!\in\! \mathbb{R}^{{I_{1}}\times{I_{2}}\times{I_{3}}}$ is denoted by $\mathcal{X}_{(n)}\!\in\!\mathbb{R}^{{{I_{n}}\times{\Pi_{j\neq{n}}}{I_{j}}}}$ and arranges the mode-\emph{n} fibers to be the columns of the resulting matrix $\mathcal{X}_{(n)}$ such that the mode-$n$ fiber becomes the row index and all other two modes become column indices. The tensor element $(i_{1},i_{2},i_{3})$ is mapped to the matrix element $(i_{n}, j)$, where
\vspace{-2mm}
\begin{displaymath}
j=1+\sum^{3}_{k=1,k\neq n}(i_{k}-1)J_{k}\;\;\textup{with}\;\;J_{k}=\prod^{k-1}_{m=1, m\neq n}I_{m}.
\end{displaymath}
\vspace{-2mm}

The inner product of two same-sized tensors $\mathcal{A}\!\in\! \mathbb{R}^{{I_{1}}\times{I_{2}}\times{I_{3}}}$ and $\mathcal{B}\!\in\! \mathbb{R}^{{I_{1}}\times{I_{2}}\times{I_{3}}}$ is the sum of the product of their entries, $\langle{\mathcal{A},\,\mathcal{B}}\rangle=\sum_{i_{1},i_{2},{i_{3}}} a_{{i_{1}}i_{2}{i_{3}}}b_{{i_{1}}i_{2}{i_{3}}}$. The Frobenius norm of a third-order tensor $\mathcal{X}$ is defined as:
\begin{displaymath}
\|\mathcal{X}\|_{F}:= \sqrt{\langle{\mathcal{X}, \mathcal{X}}\rangle}=\sqrt{\sum^{I_{1}}_{i_{1}=1}\sum^{I_{2}}_{i_{2}=1}\sum^{I_{3}}_{i_{3}=1}x^{2}_{{i_{1}}{i_{2}}{i_{3}}}}.
\end{displaymath}
The $1$-mode product of a tensor $\mathcal{X}\!\in\!\mathbb{R}^{{I_{1}}\times{I_{2}}\times{I_{3}}}$ with a matrix $U\!\in\!\mathbb{R}^{J\times{I_{1}}}$, denoted by $\mathcal{X}\!\times_{1}\!U\!\in\! \mathbb{R}^{{J}\times{I_{2}}\times{I_{3}}}$, is defined as:
\vspace{-2mm}
\begin{displaymath}
(\mathcal{X}\!\times_{1}\!{U})_{{j}{i_{2}}{i_{3}}}=\sum^{I_{1}}_{i_{1}=1}{x}_{{i_{1}}{i_{2}}{i_{3}}}{u}_{{j}{i_{1}}}.
\end{displaymath}

\subsection{Tensor Trace Norm}
With an exact analogue to the definition of the matrix rank, the rank of a tensor $\mathcal{X}$ is defined as the smallest number of rank-one tensors that generate $\mathcal{X}$ as their sum. However, there is no straightforward way to determine the rank of a tensor. In fact, the problem is NP-hard~\cite{kolda:tda, hillar:np}. Fortunately, the multi-linear rank (also called the Tucker rank in~\cite{mu:sd, goldfarb:tr}) of a tensor $\mathcal{X}$ is easy to compute, and consists of the ranks of all mode-\emph{n} unfoldings.

\begin{definition}
The multi-linear rank of a third-order tensor $\mathcal{X}$ is the tuple of the ranks of the mode-n unfoldings,
\begin{displaymath}
\textup{multi-linear rank}\!=\!\left[\textup{rank}({\mathcal{X}}_{(1)}),\textup{rank}({\mathcal{X}}_{(2)}),\textup{rank}({\mathcal{X}}_{(3)})\right]\!.
\end{displaymath}
\end{definition}

In order to keep problems simple, the (weighted) sum of the ranks of all unfoldings along each mode is used to take the place of the multi-linear rank of the tensor, and is relaxed into the following definition.

\begin{definition}
The Schatten $p$-norm ($0\!<\!p\!<\!\infty$) of a third-order tensor $\mathcal{X}$ is the average of the Schatten $p$-norms of all unfoldings $\mathcal{X}_{(n)}$, i.e.,
\vspace{-2mm}
\begin{displaymath}
\|\mathcal{X}\|_{\mathcal{S}_{p}}=\frac{1}{3}\sum^{3}_{n=1}{\|\mathcal{X}_{(n)}\|_{\mathcal{S}_{p}}}
\end{displaymath}
where $\|\mathcal{X}_{(n)}\|_{\mathcal{S}_{p}}\!=\!\left(\sum_{i}\!\sigma^{p}_{i}\right)^{1/p}$ denotes the Schatten $p$-norm of the unfolding $\mathcal{X}_{(n)}$, and $\sigma_{i}$ is the $i$-th singular value of $\mathcal{X}_{(n)}$. When $p\!=\!1$, the Schatten ${1}$-norm is the well-known trace norm, $\|\mathcal{X}\|_{*}$.
\end{definition}

For some imbalance sparse tensor decomposition problems (e.g., the size of the YouTube data used in Section 6.2 is $4,117\!\times\!4,117\!\times\!5$), the trace norm of the tensor can be incorporated by some pre-specified weights $\alpha_{n}\!\geq\!0$, $n=1,\,2,\,3$, which satisfy $\sum_{n}\!\alpha_{n}=1$.

\subsection{Weighted Tensor Decompositions}
We will introduce two of the most often used tensor decomposition models for MRL problems. In~\cite{acar:stf}, Acar \emph{et al.} presented a weighted CANDECOMP/PARAFAC (WCP) decomposition model for sparse third-order tensors:
\vspace{-2mm}
\begin{equation}\label{equ1}
\min_{A,B,C}\,\frac{1}{2}\sum_{i,j,k}{w}_{ijk}\!\left({t}_{ijk}-\sum^{R}_{r=1}a_{ir}b_{jr}c_{kr}\right)^2
\end{equation}
where $R$ is a positive integer, $\mathcal{W}$ denotes a non-negative indicator tensor of the same size as an incomplete tensor $\mathcal{T}$: ${w}_{ijk}=1$ if ${t}_{ijk}$ is observed and ${w}_{ijk}=0$ otherwise, and $A\in\mathbb{R}^{{I_{1}}\times{R}},B\in\mathbb{R}^{{I_{2}}\times{R}},C\in\mathbb{R}^{{I_{3}}\times{R}}$ are referred to as the factor matrices which are the combination of the vectors from the rank-one components (e.g., $A=[{a}_{:,1}, {a}_{:,2},\ldots,{a}_{:,R}]$).

In~\cite{filipovic:tftc}, the weighted Tucker decomposition (WTucker) model is formulated as follows:
\begin{equation}\label{equ2}
\min_{\mathcal{G},U,V,W}\,\frac{1}{2}\left\|\mathcal{W}\ast(\mathcal{T}-\mathcal{G}\!\times_{1}\!U\!\times_{2}\!V\!\times_{3}\!W)\right\|^{2}_{F}
\end{equation}
where $\ast$ denotes the Hadamard (element-wise) product, $U\!\in\!\mathbb{R}^{{I_{1}}\times{R_{1}}}$, $V\!\in\!\mathbb{R}^{{I_{2}}\times{R_{2}}}$, $W\!\in\mathbb{R}^{{I_{3}}\times{R_{3}}}$, and $\mathcal{G}\!\in\! \mathbb{R}^{{R_{1}}\times R_{2}\times{R_{3}}}$ is a core tensor with the given multi-linear rank $({R_{1}},{R_{2}},{R_{3}})$. Since the decomposition rank $R_{n}\,(n\!=\!1,2, 3)$ is in general much smaller than $I_{n}$, in this sense, the storage of the Tucker decomposition form can be significantly smaller than that of the original tensor. Moreover, unlike the rank of the tensor $R$, the multi-linear rank $({R_{1}},{R_{2}},{R_{3}})$ is clearly computable. If the factor matrices of the Tucker decomposition are constrained orthogonal, the classical decomposition methods are referred to as the higher-order singular value decomposition (HOSVD)~\cite{lathauwer:msvd} or higher-order orthogonal iteration (HOOI)~\cite{athauwer:hooi}, where the latter leads to the estimation of best rank-$(R_{1},R_{2},R_{3})$ approximations while the truncation of HOSVD may achieve a good rank-$(R_{1},R_{2},R_{3})$ approximation but in general not the best possible one~\cite{athauwer:hooi}. Hence, we are particularly interested in extending the HOOI method for sparse MRL problems.

In addition, several extensions of both tensor decomposition models are developed for tensor estimation problems, such as~\cite{kressner:lrtc}, \cite{chen:stdc}, \cite{jain:ff}. However, for all those methods, a suitable rank value needs to be given, and it has been shown that both WTucker and WCP models are usually sensitive to the given ranks due to their least-squares formulations~\cite{gandy:tc, liu:ghooi}, and they have poor performance when the data have a high rank~\cite{liu:tcem}.

\subsection{Problem Formulations}
For multi-relational prediction, the sparse tensor trace norm minimization problem is formulated as follows:
\vspace{-2mm}
\begin{equation}\label{equ3}
\min_{\mathcal{X}}\,\sum^{3}_{n=1}{\alpha_{n}\|\mathcal{X}_{(n)}\|_{*}},\quad \textup{s.t.},\,\mathcal{X}_{\Omega}=\mathcal{T}_{\Omega}
\end{equation}
where $\alpha_{n}$'s are pre-specified weights, and $\Omega$ is the set of indices of observed entries. Liu \emph{et al.}~\cite{liu:tcem} proposed three efficient algorithms (e.g., the HaLRTC algorithm) to solve (\ref{equ3}). In addition, there are some similar convex tensor completion algorithms in~\cite{gandy:tc, signoretto:lt, yang:fpim}. Tomioka and Suzuki~\cite{tomioka:ctd} proposed a latent trace norm minimization model,
\vspace{-2mm}
\begin{equation}\label{equ4}
\min_{\mathcal{X}_{n}}\,\frac{1}{\lambda}\sum^{N}_{n=1}{\|\mathcal{X}_{n,(n)}\|_{*}}+\frac{1}{2}{\|\mathcal{P}_{\Omega}(\sum^{N}_{n=1}\mathcal{X}_{n})-\mathcal{P}_{\Omega}(\mathcal{T})\|^{2}_{F}}
\end{equation}
where $\mathcal{P}_{\Omega}$ is the projection operator: $\mathcal{P}_{\Omega}(\mathcal{T})_{ijk}\!=\!\mathcal{T}_{ijk}$ if $(i,j,k)\!\in\!\Omega$ and $\mathcal{P}_{\Omega}(\mathcal{T})_{ijk}\!=\!0$ otherwise, and $\lambda>0$ is a regularization parameter.

More recently, it has been shown that the tensor trace norm minimization models mentioned above can be substantially suboptimal~\cite{mu:sd, romera:crtc}. However, if the order of the involved tensor is no more than three, the models (\ref{equ3}) and (\ref{equ4}) often perform better than the more balanced (square) matrix model in~\cite{mu:sd}. Indeed each unfolding $\mathcal{X}_{(n)}$ shares the same entries, and thus cannot be optimized independently. Therefore, we must apply variable splitting and introduce multiple additional equal-sized variables to all unfoldings of $\mathcal{X}$. Moreover, existing algorithms involve multiple SVDs in each iteration and suffer from high computational cost $O(I^{4})$, where the assumed size of the tensor is $I\times I\times I$.

\section{Core Tensor Trace Norm Regularized Tensor Decomposition}
To address the poor scalability of existing low multi-linear rank tensor recovery algorithms, we present two scalable core tensor trace norm (or together with graph Laplacian) regularized orthogonal decomposition models, and then achieve three smaller-scale matrix trace norm minimization problems. Then in Section 4, we will develop some efficient algorithms for solving the problems.

\subsection{Core Tensor Trace Norm Minimization Models}
Assume that $\mathcal{X}\!\in\!\mathbb{R}^{{I_{1}}\times{I_{2}}\times{I_{3}}}$ is a multi-relational tensor with multi-linear rank $(r_{1},r_{2},r_{3})$, $\mathcal{X}$ can be decomposed as:
\vspace{-1mm}
\begin{equation}\label{equ5}
\mathcal{X}=\mathcal{G}\!\times_{1}\!U\!\times_{2}\!V\!\times_{3}\!W
\end{equation}
where $U\!\in\!\mathbb{R}^{{I_{1}}\times{r_{1}}}$, $V\!\in\!\mathbb{R}^{{I_{2}}\times{r_{2}}}$ and $W\!\in\!\mathbb{R}^{{I_{3}}\times{r_{3}}}$ are the column-wise orthonormal matrices, and can be thought of as the principal components in each mode. The entries of the core tensor $\mathcal{G}\!\in\!\mathbb{R}^{{r_{1}}\times{r_{2}}\times{r_{3}}}$ show the level of interaction between the different components. For $r_{n}$ ($n\!=\!1, 2, 3$), we recommend a matrix rank estimation approach recently developed in~\cite{julia:rem} to compute some good values $(r'_{1}, r'_{2},r'_{3})$ for the multi-linear rank of the involved tensor. Then we can give some relatively large integers $(d_{1},d_{2},d_{3})$ satisfying $d_{n}\geq r'_{n}$ and $d_{n}\geq r_{n}$, $n=1, 2, 3$.

\begin{theorem}\label{theo1}
Let $\mathcal{X}\!\in\!\mathbb{R}^{{I_{1}}\times{I_{2}}\times{I_{3}}}$ with multi-linear rank $({r_{1}},{r_{2}},{r_{3}})$  and $\mathcal{G}\in\mathbb{R}^{{d_{1}}\times{d_{2}}\times{d_{3}}}$ satisfy $\mathcal{X}=\mathcal{G}\!\times_{1}\!U\!\times_{2}\!V\!\times_{3}\!W$, and $U^{T}U=I_{d_{1}}$, $V^{T}V=I_{d_{2}}$ and $W^{T}W=I_{d_{3}}$, then
\begin{equation*}
\|\mathcal{X}\|_{\mathcal{S}_{p}}=\|\mathcal{G}\|_{\mathcal{S}_{p}}
\end{equation*}
where $\|\mathcal{X}\|_{\mathcal{S}_{p}}$ and $\|\mathcal{G}\|_{\mathcal{S}_{p}}$ denote the Schatten $p$-norm of $\mathcal{X}$ and its core tensor $\mathcal{G}$, respectively.
\end{theorem}

The proof of Theorem~\ref{theo1} is given in APPENDIX A. Since the trace norm (i.e., the Schatten $1$-norm) is the tightest convex surrogate to the rank function~\cite{fazel:rmh}, \cite{candes:crmc}, we mainly consider the trace norm case in this paper. According to the equivalence relation of the trace norm of a low multi-linear rank tensor and its core tensor, the tensor completion model (\ref{equ3}) is formulated into the following form:
\begin{equation}\label{equ6}
\begin{aligned}
&\min_{\mathcal{G},U,V,W,\mathcal{X}}\frac{1}{\lambda}\|\mathcal{G}\|_{*}+\frac{1}{2}\|\mathcal{X}-\mathcal{G}\!\times_{1}\!U\!\times_{2}\!V\!\times_{3}\!W\|^{2}_{F},\\
&\quad\;\textup{s.t.},\mathcal{X}_{\Omega}\!=\!\mathcal{T}_{\Omega},U^{T}U\!=\!I_{d_{1}},V^{T}V\!=\!I_{d_{2}},W^{T}W\!=\!I_{d_{3}}.
\end{aligned}
\end{equation}
When all entries of $\mathcal{T}$ are observed, the model (6) degenerates to the following core tensor trace norm regularized tensor decomposition problem~\cite{shang:hotd}:
\begin{equation}\label{equ7}
\begin{aligned}
&\min_{\mathcal{G},U,V,W}\frac{1}{\lambda}\|\mathcal{G}\|_{*}+\frac{1}{2}\|\mathcal{T}-\mathcal{G}\!\times_{1}\!U\!\times_{2}\!V\!\times_{3}\!W\|^{2}_{F},\\
&\quad\;\textup{s.t.},\,U^{T}U\!=\!I_{d_{1}},V^{T}V\!=\!I_{d_{2}},W^{T}W\!=\!I_{d_{3}}.
\end{aligned}
\end{equation}

It is clear that the core tensor $\mathcal{G}$ of size $({d_{1}},{d_{2}},{d_{3}})$ has much smaller size than the whole tensor $\mathcal{T}$, i.e., $d_{n}\!\ll\! I_{n}$ for all $n\!\in\!\{1,2,3\}$. Therefore, our core tensor trace norm regularized orthogonal tensor decomposition models (\ref{equ6}) and (\ref{equ7}) can alleviate the SVD computational burden of much larger unfoldings in both models (\ref{equ3}) and (\ref{equ4}). Besides, the core tensor trace norm term promotes low multi-linear rank tensor decompositions, and enhances the robustness of the multi-linear rank selection, while those traditional tensor decomposition methods are usually sensitive to the given multi-linear rank~\cite{liu:tcem, shang:hotd}.

\subsection{Sparse HOOI Model}
When $\lambda\!\rightarrow\! \infty$, the model (\ref{equ6}) degenerates to the following sparse tensor HOOI (SHOOI) problem,
\vspace{-1mm}
\begin{equation}\label{equ8}
\begin{aligned}
&\min_{\mathcal{G},U,V,W,\mathcal{Z}}\;\frac{1}{2}\|\mathcal{W}\ast(\mathcal{Z}-\mathcal{T})\|^{2}_{F},\\
&\textup{s.t.},\mathcal{Z}\!\!=\!\mathcal{G}\!\!\times_{\!1}\!\!U\!\!\times_{\!2}\!\!V\!\!\times_{\!3}\!\!W,U^{T}\!U\!\!=\!\!I_{d_{1}}\!,V^{T}\!V\!\!=\!\!I_{d_{2}}\!,W^{T}\!W\!\!=\!\!I_{d_{3}}.
\end{aligned}
\end{equation}
In a sense, the SHOOI model (\ref{equ8}) is a special case of our ROID method (see the Supplementary Materials for detailed discussion). When all entries of $\mathcal{T}$ are observed, the SHOOI model (\ref{equ8}) becomes a traditional HOOI problem in~\cite{athauwer:hooi}.

\subsection{Graph Regularized Model}
Inspired by the work in~\cite{narita:tf}, \cite{gu:nmf}, \cite{shang:nmf}, we also exploit the auxiliary information given as link-affinity matrices in a graph regularized ROID (GROID) model:
\vspace{-1mm}
\begin{equation}\label{equ9}
\begin{aligned}
&\min_{\mathcal{G},U,V,W,\mathcal{X}}\,\frac{1}{\lambda}\|\mathcal{G}\|_{*}+\frac{1}{2}\|\mathcal{X}-\mathcal{G}\!\times_{1}\!U\!\times_{2}\!V\!\times_{3}\!W\|^{2}_{F}\\
&\;\;\;+\frac{\mu}{2}[\textup{Tr}(U^{T}\!L_{1}U)+\textup{Tr}(V^{T}\!L_{2}V)+\textup{Tr}(W^{T}\!L_{3}W)],\\
&\textup{s.t.},\,\mathcal{X}_{\Omega}=\mathcal{T}_{\Omega},U^{T}U=I_{d_{1}},V^{T}V=I_{d_{2}},W^{T}W=I_{d_{3}}
\end{aligned}
\end{equation}
where $\mu\geq0$ is a regularization constant, $\textup{Tr}(\cdot)$ denotes the matrix trace, $L_{n}$ is the graph Laplacian matrix, i.e., $L_{n}\!=\!D_{n}\!-\!W_{n}$, $W_{n}$ is the weight matrix for the object set $O_{n}$ or different relations, and $D_{n}$ is the diagonal matrix whose entries are column sums of $W_{n}$, i.e., $(D_{n})_{ii}=\sum_{j}(W_{n})_{ij}$.

\section{Optimization Algorithms}
In this section, we propose an efficient method of augmented Lagrange multipliers (ALM) to solve our ROID problem (\ref{equ6}), and then extend the proposed algorithm to solve (\ref{equ7})-(\ref{equ9}). As a variant of the standard ALM, the alternating direction method of multipliers (ADMM) has received much attention recently due to the tremendous demand from large-scale machine learning applications~\cite{boyd:admm, ouyang:sadmm}. Similar to (\ref{equ3}), the proposed problem (\ref{equ6}) is difficult to solve due to the interdependent tensor trace norm term $\|\mathcal{G}\|_{*}$. Therefore, we first introduce three much smaller auxiliary variables ${G}_{n}\!\in\!\mathbb{R}^{{{d_{n}}\times{\Pi_{j\neq{n}}}{d_{j}}}}$ into (\ref{equ6}), and then reformulate it into the following equivalent form:
\vspace{-1mm}
\begin{small}
\begin{equation}\label{equ10}
\begin{aligned}
&\min_{\mathcal{G},{U},{V},W,\{G_{\!n}\!\},\mathcal{X}}\sum^{3}_{n=1}\!\frac{\|G_{n}\|_{*}}{3\lambda}\!+\!\frac{1}{2}\|\mathcal{X}-\mathcal{G}\!\times_{1}\!U\!\!\times_{2}\!V\!\!\times_{3}\!W\|^{2}_{F},\\
&\textup{s.t.},\mathcal{X}_{\Omega}\!=\!\mathcal{T}_{\Omega},\mathcal{G}_{(n)}\!\!=\!G_{\!n},U^{T}\!U\!\!=\!I_{d_{1}},V^{T}\!V\!\!=\!I_{d_{2}},W^{T}\!W\!\!=\!I_{d_{3}}.
\end{aligned}
\end{equation}
\end{small}
The partial augmented Lagrangian function of (\ref{equ10}) is
\vspace{-1mm}
\begin{equation}\label{equ11}
\begin{aligned}
&\quad\qquad \mathcal{L}_{\rho}(\{{G}_{n}\},\,\mathcal{G},\,{U},\,{V},\,W,\,\mathcal{X},\,\{{Y}_{n}\})=\\
&\sum^{3}_{n=1}\!\left(\frac{\|G_{n}\|_{*}}{3\lambda}+\langle{Y}_{n},\;{\mathcal{G}}_{(n)}-{G}_{n}\rangle+\frac{\rho}{2}\|\mathcal{G}_{(n)}-{G}_{n}\|^{2}_{F}\right)\\
&\quad\quad\;\;\;\; +\frac{1}{2}\|\mathcal{X}-\mathcal{G}\!\times_{1}\!U\!\times_{2}\!V\!\times_{3}\!W\|^{2}_{F}
\end{aligned}
\end{equation}
where ${Y}_{n}\!\in\!\mathbb{R}^{{{d_{n}}\times{\Pi_{j\neq{n}}}{d_{j}}}}\,(n\!=\!1,2,3)$ are the matrices of Lagrange multipliers (or dual variables), and $\rho\!>\!0$ is called the penalty parameter. Our ADMM iterative scheme for solving (\ref{equ10}) is derived by successively minimizing $\mathcal{L}_{\rho}$ over $(\{{G}_{n}\},\mathcal{G},{U},{V},W,\mathcal{X})$, and then updating $(Y_{1},Y_{2},Y_{3})$.

\subsection{Updating $\{{G}^{k+1}_{1},{G}^{k+1}_{2},{G}^{k+1}_{3}\}$}
By keeping all the other variables fixed, ${G}^{k+1}_{n}$ is updated by solving the following problem,
\vspace{-1mm}
\begin{equation}\label{equ12}
\min_{G_{n}}\frac{\|G_{n}\|_{*}}{3\lambda}+\frac{\rho^{k}}{2}\|\mathcal{G}^{k}_{(n)}-G_{n}+Y^{k}_{n}/\rho^{k}\|^{2}_{F}.
\end{equation}
For solving (\ref{equ12}), we give the shrinkage operator~\cite{cai:svt} below.

\begin{definition}
For any matrix ${M}\in \mathbb{R}^{{m}\times{n}}$, the singular vector thresholding \emph{(SVT)} operator is defined as:
\vspace{-1mm}
\begin{equation*}
\textup{SVT}_{\mu}({M}):={\overline{U}}\textup{diag}(\textup{max}\{\overline{\sigma}-\mu,0\}){\overline{V}}^{T}
\end{equation*}
where $\max\{\cdot,\cdot\}$ should be understood element-wise, ${\overline{U}}\!\in\! \mathbb{R}^{m\times r}$, ${\overline{V}}\!\in\! \mathbb{R}^{n\times{r}}$ and $\overline{\sigma}\!=\!(\sigma_{1},\sigma_{2},\ldots,\sigma_{r})^T\!\in\! \mathbb{R}^{r\times 1}$  are obtained by SVD of $M$, i.e., $M={\overline{U}}\textup{diag}(\overline{\sigma}){\overline{V}}^{T}$.
\end{definition}

Therefore, a closed-form solution to (\ref{equ12}) is given by:
\vspace{-1mm}
\begin{equation}\label{equ13}
G^{k+1}_{n}=\textup{SVT}_{1/(3\lambda\rho^{k})}(\mathcal{G}^{k}_{(n)}+Y^{k}_{n}/\rho^{k}).
\end{equation}
It it clear that only some smaller size matrices ${M}_{n}\!=$ $\mathcal{G}^{k}_{(n)}\!\!+\!Y^{k}_{n}/\rho^{k}\!\in\!\mathbb{R}^{{{d_{n}}\times{\Pi_{j\neq{n}}}{d_{j}}}}\;(d_{n}\!\ll\! I_{n},n\!=\!1,2,3)$ in (\ref{equ13}) need to perform SVD. Therefore, our shrinkage operator has a significantly lower computational complexity ${O}(\sum_{n}\!{d^{2}_{n}}{\Pi_{j\neq{n}}}d_{j})$ while the computational complexity of those algorithms for solving (\ref{equ3}) and (\ref{equ4}) is ${O}(\sum_{n}\!{{I^{2}_{n}}{\Pi_{j\neq{n}}}{I_{j}}})$ for each iteration. Hence, our algorithm has a much lower complexity than those as in~\cite{gandy:tc, liu:tcem, signoretto:lt, yang:fpim, tomioka:ctd}.

\subsection{Updating $\{{U}^{k+1},{V}^{k+1},{W}^{k+1},\mathcal{G}^{k+1}\}$}
The optimization problem (\ref{equ10}) with respect to ${U}$, ${V}$, ${W}$ and $\mathcal{G}$ is formulated as follows:
\vspace{-1mm}
\begin{equation}\label{equ14}
\begin{aligned}
\min_{\mathcal{G},U,V,W}&\sum^{3}_{n=1}\frac{\rho^{k}}{2}\|\mathcal{G}_{(n)}-{G}^{k+1}_{n}+{Y}^{k}_{n}/\rho^{k}\|^{2}_{F}\\
&+\frac{1}{2}\|\mathcal{X}^{k}-\mathcal{G}\!\times_{1}\!U\!\times_{2}\!V\!\times_{3}\!W\|^{2}_{F},\\
\textup{s.t.},\;&U^{T}U=I_{d_{1}},\,V^{T}V=I_{d_{2}},\,W^{T}W=I_{d_{3}}.
\end{aligned}
\end{equation}
Unlike the HOOI algorithm in~\cite{athauwer:hooi}, we propose a new orthogonal iteration scheme to update the matrices $U$, $V$ and $W$ for the optimization of (\ref{equ14}). Moreover, the conventional HOOI can be seen as a special case of (\ref{equ14}) when $\rho^{k}\!=\!0$. For any estimate of these matrices, the optimal solution with respect to $\mathcal{G}$ is given by the following theorem.

\begin{theorem}\label{theo2}
For given matrices ${U}$, ${V}$ and ${W}$, the optimal core tensor $\mathcal{G}$ of the optimization problem (\ref{equ14}) is given by
\vspace{-1mm}
\begin{equation}\label{equ15}
\mathcal{G}=\frac{1}{1+3\rho^{k}}\left(\mathcal{A}+\rho^{k}\mathcal{B}\right)
\end{equation}
where $\mathcal{A}\!=\!\mathcal{X}^{k}\!\!\times_{\!1}\!\!U^{T}\!\!\!\times_{\!2}\!\!V^{T}\!\!\!\times_{\!3}\!\!W^{T}$\!, $\mathcal{B}\!=\!\!\sum^{3}_{n\!=\!1}\!\!\textup{refold}(G^{k\!+\!1}_{n}\!-\!Y^{k}_{n}\!/\!\rho^{k})$ and $\textup{refold}(\cdot)$ denotes the refolding of the matrix into a tensor.
\end{theorem}
Please see APPENDIX B for the proof of Theorem~\ref{theo2}. Moreover, we propose an orthogonal iteration scheme for solving $U$, $V$ and $W$, which is an alternating orthogonal procrustes scheme to solve the rank-$(d_{1},d_{2},d_{3})$ problem. Analogous with Theorem 4.2 in~\cite{athauwer:hooi}, we first state that the minimization problem (\ref{equ14}) can be formulated as follows:

\begin{theorem}\label{theo3}
Assume a real third-order tensor $\mathcal{X}^{k}$, then the minimization problem (\ref{equ14}) is equivalent to the maximization (over these matrices $U$, $V$ and $W$ having orthonormal columns) of the following function
\vspace{-0.5mm}
\begin{equation}\label{equ16}
g(U,\,V,\,W)=\|\mathcal{A}+\rho^{k}\mathcal{B}\|^{2}_{F}.
\end{equation}
\end{theorem}

The detailed proof of Theorem~\ref{theo3} is given in APPENDIX C. According to the theorem, an orthogonal iteration scheme is proposed to successively solve $U$, $V$ and $W$ by fixing the other variables. Imagine that the matrices $V$ and $W$ are fixed and that the optimization problem (\ref{equ16}) is merely a quadratic function of the unknown matrix $U$. Consisting of orthonormal columns, we have
\begin{equation}\label{equ17}
\max_{U,\,U^{T}\!U=I_{d_{1}}}\!\|\mathcal{M}^{k}_{1}\!\times_{\!1}\!U^{T}\!+\!\rho^{k}\mathcal{B}\|^{2}_{F}\!=\!\|(\mathcal{M}^{k}_{1})^{T}_{(1)}U\!+\!\rho^{k}\mathcal{B}^{T}_{(1)}\|^{2}_{F}
\end{equation}
where $\mathcal{M}^{k}_{1}\!=\!\mathcal{X}^{k}\!\times_{2}\!(V^{k})^{T}\!\!\times_{3}\!(W^{k})^{T}$. This is actually the well-known orthogonal procrustes problem~\cite{nick:mpp}. Hence, we have
\vspace{-1mm}
\begin{equation}\label{equ18}
U^{k+1}=\textup{ORT}\left((\mathcal{M}^{k}_{1})_{(1)}\mathcal{B}^{T}_{(1)}\right)
\end{equation}
where $\textup{ORT}(A):=\widehat{U}\widehat{V}^{T}$, and $\widehat{U}$ and $\widehat{V}$ are the left singular vector and right singular vector matrices obtained by the tight SVD of the matrix $A$. Repeating the above procedure for $V$ and $W$, we have
\vspace{-1mm}
\begin{equation}\label{equ19}
\begin{aligned}
{V}^{k+1}&=\textup{ORT}\left((\mathcal{M}^{k}_{2})_{(2)}\mathcal{B}^{T}_{(2)}\right),\\
{W}^{k+1}&=\textup{ORT}\left((\mathcal{M}^{k}_{3})_{(3)}\mathcal{B}^{T}_{(3)}\right)
\end{aligned}
\end{equation}
where $\mathcal{M}^{k}_{2}\!=\!\mathcal{X}^{k}\!\times_{1}\!(U^{k+1})^{T}\!\times_{3}\!(W^{k})^{T}$ and $\mathcal{M}^{k}_{3}\!=\!\mathcal{X}^{k}\!\times_{1}\!(U^{k+1})^{T}\!\times_{2}\!(V^{k+1})^{T}$.

For the updated matrices ${U}^{k+1}$, ${V}^{k+1}$ and ${W}^{k+1}$, then $\mathcal{G}$ is updated by
\vspace{-1mm}
\begin{equation}\label{equ20}
\begin{aligned}
\mathcal{G}^{k+1}=&\frac{\rho^{k}\sum^{3}_{n=1}\textup{refold}({G}^{k+1}_{n}-{Y}^{k}_{n}/\rho^{k})}{1+3\rho^{k}}\\
&+\frac{\mathcal{M}^{k}_{3}\!\times_{3}\!(W^{k+1})^{T}}{1+3\rho^{k}}.
\end{aligned}
\end{equation}

\subsection{Updating $\mathcal{X}^{k+1}$}
The optimization problem (\ref{equ10}) with respect to $\mathcal{X}$ is formulated as follows:
\vspace{-1mm}
\begin{equation}\label{equ21}
\begin{aligned}
&\min_{\mathcal{X}}\|\mathcal{X}-\mathcal{G}^{k+1}\!\times_{1}\!U^{k+1}\!\times_{2}\!V^{k+1}\!\times_{3}\!W^{k+1}\|^{2}_{F},\\
&\;\textup{s.t.},\,\mathcal{X}_{\Omega}=\mathcal{T}_{\Omega}.
\end{aligned}
\end{equation}
By introducing a Lagrangian multiplier $\mathcal{Y}\!\in\!\mathbb{R}^{{I_{1}}\times{I_{2}}\times{I_{3}}}$ for $\mathcal{X}_{\Omega}\!=\!\mathcal{T}_{\Omega}$, the Lagrangian function of (\ref{equ21}) is given by
\vspace{-1mm}
\begin{displaymath}
\begin{aligned}
\mathcal{H}(\mathcal{X},\mathcal{Y})=&\|\mathcal{X}-\mathcal{G}^{k+1}\!\times_{1}\!U^{k+1}\!\times_{2}\!V^{k+1}\!\times_{3}\!W^{k+1}\|^{2}_{F}\\
&+\langle \mathcal{Y},\;\mathcal{P}_{\Omega}(\mathcal{X})-\mathcal{P}_{\Omega}(\mathcal{T})\rangle.
\end{aligned}
\end{displaymath}
Letting $\nabla_{(\mathcal{X},\mathcal{Y})}\mathcal{H}\!=\!0$, we then obtain the following Karush-Kuhn-Tucker (KKT) conditions:
\vspace{-1mm}
\begin{displaymath}
\begin{aligned}
2(\mathcal{X}-\mathcal{G}^{k+\!1}\!\!\times_{1}\!U^{k+\!1}\!\!\times_{2}\!V^{k+\!1}\!\!\times_{3}\!W^{k+\!1})+\mathcal{P}_{\Omega}(\mathcal{Y})&=0,\\
\mathcal{X}_{\Omega}-\mathcal{T}_{\Omega}&=0.
\end{aligned}
\end{displaymath}
By deriving simply the KKT conditions, we have the optimal solution as follows:
\vspace{-1mm}
\begin{equation}\label{equ22}
\mathcal{X}^{k+1}\!=\!\mathcal{P}_{\Omega}(\mathcal{T})+\mathcal{P}^{\perp}_{\Omega}(\mathcal{G}^{k+1}\!\times_{1}\!U^{k+1}\!\times_{2}\!V^{k+1}\!\times_{3}\!W^{k+1})
\end{equation}
where $\mathcal{P}^{\perp}_{\Omega}$ is the complementary operator of $\mathcal{P}_{\Omega}$.

Based on the above analysis, we develop an efficient ADMM algorithm for solving (\ref{equ10}), as outlined in \textbf{Algorithm \ref{alg1}}. Moreover, Algorithm~\ref{alg1} can be extended to solve (\ref{equ7}) and the SHOOI problem (\ref{equ8}) (the details can be found in the Supplementary Materials). For instance, with the tensor of Lagrange multipliers $\mathcal{Y}^{k}$, the iterations of ADMM for solving (\ref{equ8}) go as follows:
\vspace{-1mm}
\begin{equation}\label{equ23}
\begin{aligned}
&\min_{\mathcal{G},U,V,W}\frac{1}{2}\|\mathcal{Z}^{k}-\mathcal{G}\!\times_{1}\!U\!\times_{2}\!V\!\times_{3}\!W+\mathcal{Y}^{k}/\rho^{k}\|^{2}_{F},\\
&\;\;\;\;\textup{s.t.},\,U^{T}U=I_{d_{1}},\,V^{T}V=I_{d_{2}},\,W^{T}W=I_{d_{3}},
\end{aligned}
\end{equation}
\vspace{-1mm}
\begin{equation}\label{equ24}
\begin{aligned}
\min_{\mathcal{Z}}&\frac{\rho^{k}}{2}\|\mathcal{Z}\!-\!\mathcal{G}^{k+\!1}\!\!\times_{1}\!U^{k+\!1}\!\!\times_{2}\!V^{k+\!1}\!\!\times_{3}\!W^{k+\!1}\!\!+\!\mathcal{Y}^{k}/\rho^{k}\|^{2}_{F}\\
&+\frac{1}{2}\|\mathcal{W}\ast(\mathcal{Z}-\mathcal{T})\|^{2}_{F}.
\end{aligned}
\end{equation}

To monitor convergence of Algorithm~\ref{alg1}, the adaptively adjusting strategy of the penalty parameter $\rho^{k}$ in~\cite{boyd:admm} is introduced. The necessary optimality conditions for (\ref{equ10}) are primal feasibility
\vspace{-1mm}
\begin{equation}\label{equ25}
\begin{aligned}
\mathcal{X}^{*}_{\Omega}=\mathcal{T}_{\Omega},\;G^{*}_{n}&=\mathcal{G}^{*}_{(n)},\,n=1,\,2,\,3,\\
(U^{*})^{T}U^{*}=I_{d_{1}},(V^{*})^{T}&V^{*}=I_{d_{2}},(W^{*})^{T}W^{*}=I_{d_{3}}
\end{aligned}
\end{equation}
and dual feasibility
\vspace{-1mm}
\begin{equation}\label{equ26}
\begin{aligned}
0\in\partial&\|G^{*}_{n}\|_{*}/(3\lambda)-Y^{*}_{n},\\
\mathcal{G}^{*}\!\!-\!\mathcal{X}^{*}\!\!\times_{1}\!\!(\!U^{*}\!)^{T}\!\times_{2}&(\!V^{*}\!)^{T}\!\!\times_{3}\!(\!W^{*}\!)^{T}\!\!+\!\!\sum^{3}_{n=1}\textup{refold}(Y^{*}_{n})\!=\!0
\end{aligned}
\end{equation}
where $(\{G^{*}_{n}\},\mathcal{G}^{*}\!,U^{*}\!,V^{*}\!,W^{*}\!,\mathcal{X}^{*})$ is a KKT point of (\ref{equ10}). By the optimal conditions of (\ref{equ12}) and (\ref{equ14}) and $Y^{k+1}_{n}\!=\!Y^{k}_{n}\!+\!\rho^{k}(\mathcal{G}^{k+1}_{(n)}\!-\!G^{k+1}_{n})$, we have
\vspace{-1mm}
\begin{displaymath}
\begin{aligned}
&\quad\;\; 0\in \partial\|G^{k+\!1}_{n}\|_{*}/(3\lambda)-Y^{k+\!1}_{n}+\rho^{k}(\mathcal{G}^{k+\!1}_{(n)}-\mathcal{G}^{k}_{(n)}),\\
&\!\!\mathcal{G}^{k\!+\!1}\!\!-\!\!\mathcal{X}^{k\!+\!1}\!\!\times_{1}\!\!(\!U^{k\!+\!1}\!)^{T}\!\!\!\times_{2}\!\!(\!V^{k\!+\!1}\!)^{T}\!\!\!\times_{3}\!\!(\!W^{k\!+\!1}\!)^{T}\!\!+\!\!\sum^{3}_{n=1}\!\textup{refold}(\!Y^{\!k\!+\!1}_{n}\!)\\
&\quad\;+\!(\!\mathcal{X}^{k+\!1}\!\!-\!\!\mathcal{X}^{k}\!)\!\!\times_{1}\!\!(\!U^{k+\!1}\!)^{T}\!\!\times_{2}\!\!(\!V^{k+\!1}\!)^{T}\!\!\times_{3}\!\!(\!W^{k+\!1}\!)^{T}\!\!=\!0.
\end{aligned}
\end{displaymath}

Let $r^{k\!+\!1}\!:=\!\max(\|\mathcal{G}^{k\!+\!1}_{(n)}\!-\!G^{k\!+\!1}_{n}\|_{F}$, $n\!=\!1,2,3)$ be the primal residual and $s^{k+\!1}\!:=\!\max(\rho^{k}\|\mathcal{G}^{k+\!1}_{(n)}\!-\!\mathcal{G}^{k}_{(n)}\|_{F},\,\rho^{k}\|(\mathcal{X}^{k+\!1}\!-\!\mathcal{X}^{k})\!\times_{1}\!(U^{k+\!1})^{T}\!\!\times_{2}\!(V^{k+\!1})^{T}\!\!\times_{3}\!(W^{k+\!1})^{T}\|_{F})$ be the dual residual at iteration $(k\!+\!1)$, we require the primal and dual residuals at the $(k\!+\!1)$-iteration to be small such that they satisfy the optimal conditions in (\ref{equ25}) and (\ref{equ26}). Following~\cite{boyd:admm}, an efficient strategy is to let $\rho\!=\!\rho^{0}$ (the initialization in Algorithm~\ref{alg1}) and update $\rho^{k}$ iteratively by:
\vspace{-1mm}
\begin{equation}\label{equ27}
\rho^{k+1}=\left\{ \begin{array}{ll}
\gamma\rho^{k},  & r^{k}>10s^{k},\\
\rho^{k}/\gamma, & s^{k}>10r^{k},\\
\rho^{k}, &\textup{otherwise},\\
\end{array}\right.
\end{equation}
where $\gamma>1$.

\begin{algorithm}[t]
\caption{ADMM for ROID problem (\ref{equ10})}
\label{alg1}
\renewcommand{\algorithmicrequire}{\textbf{Input:}}
\renewcommand{\algorithmicensure}{\textbf{Initialize:}}
\renewcommand{\algorithmicoutput}{\textbf{Output:}}
\begin{algorithmic}[1]
\REQUIRE $\mathcal{T}_{\Omega}$, multi-linear rank $({d_{1}},d_{2},{d_{3}})$, $\lambda$ and $\textup{tol}$.
\WHILE {not converged}
\STATE {Update $G^{k+1}_{n}$ by (\ref{equ13}).}
\STATE {Update $U^{k+1}$, $V^{k+1}$ and $W^{k+1}$ by (\ref{equ18}) and (\ref{equ19}).}
\STATE {Update $\mathcal{G}^{k+1}$ and $\mathcal{X}^{k+1}$ by (\ref{equ20}) and (\ref{equ22}).}
\STATE {Update the multipliers $Y^{k+1}_{n}$ by\\
$\;\;\;\; Y^{k+1}_{n}=Y^{k}_{n}+\rho^{k}(\mathcal{G}^{k+1}_{(n)}-G^{k+1}_{n}),\,n=1,2,3$.}
\STATE {Update $\rho^{k+1}$ by (\ref{equ27}).}
\STATE {Check the convergence condition,\\
$\;\;\max(\|\mathcal{G}^{k+1}_{(n)}-G^{k+1}_{n}\|_{F}/\|\mathcal{T}\|_{F},\,n=1,2,3)<\textup{tol}$.}
\ENDWHILE
\OUTPUT $\mathcal{G}^{k+1}$, $U^{k+1}$, $V^{k+1}$ and $W^{k+1}$.
\end{algorithmic}
\end{algorithm}

\subsection{Extension for GROID}
Algorithm~\ref{alg1} can be extended to solve our GROID problem (\ref{equ9}), where the main difference is that the subproblem with respect to ${U}$, ${V}$, ${W}$ and $\mathcal{G}$ is formulated as follows:
\vspace{-1mm}
\begin{equation}\label{equ28}
\begin{aligned}
&\min_{\mathcal{G},U,V,W}\;\sum^{3}_{n=1}\frac{\rho^{k}}{2}\|\mathcal{G}_{(n)}-{G}^{k+1}_{n}+{Y}^{k}_{n}/\rho^{k}\|^{2}_{F}\\
&\;\;+\frac{1}{2}\|\mathcal{X}^{k}-\mathcal{G}\!\times_{1}\!U\!\times_{2}\!V\!\times_{3}\!W\|^{2}_{F}+\frac{\mu}{2}h(U,V,W),\\
&\;\;\textup{s.t.},\;U^{T}U=I_{d_{1}},\,V^{T}V=I_{d_{2}},\,W^{T}W=I_{d_{3}}
\end{aligned}
\end{equation}
where $h(U,V,W)\!:=\!\textup{Tr}(U^{T}\!L_{1}U)\!+\!\textup{Tr}(V^{T}\!L_{2}V)\!+\!\textup{Tr}(W^{T}\!L_{3}W)$. Similar to Algorithm 1, $U$, $V$ and $W$ can be solved by minimizing the following cost function,
\vspace{-1mm}
\begin{equation*}
\mathcal{F}(U,V,W)=-\frac{g(U,V,W)}{2(1\!+\!3\rho^{k})}+\frac{\mu}{2}h(U,V,W).
\end{equation*}
Let $\nabla \mathcal{F}(U,V^{k},W^{k})$ be the derivative of the function $\mathcal{F}(U,V^{k},W^{k})$, and $\nabla \mathcal{F}(U,V^{k},W^{k})$ be Lipschitz continuous with the constant $\tau^{k}$, i.e., $\|\nabla \mathcal{F}(U,V^{k},W^{k})-\nabla \mathcal{F}(\widehat{U},V^{k},W^{k})\|_{F}\leq\tau^{k}\|U-\widehat{U}\|_{F}$, $\forall\,U,\,\widehat{U}\in \mathbb{R}^{I_{1}\times d_{1}}$. To update $U^{k+1}$, an approximate procedure is given by the following linearization technique
\vspace{-1mm}
\begin{small}
\begin{equation*}
\begin{aligned}
&\mathcal{F}(U,V^{k}\!,W^{k})=-\frac{g(U,V^{k},W^{k})}{2(1\!+\!3\rho^{k})}+\frac{\mu}{2} h(U,V^{k},W^{k})\\
\approx& \mathcal{F}(U^{k}\!,V^{k}\!,W^{k})\!+\!\langle\nabla \mathcal{F}(U^{k}\!,V^{k}\!,W^{k}),U\!-\!U^{k}\rangle\!+\!\frac{\tau^{k}}{2}\|U\!-\!U^{k}\|^{2}_{F}\\
=&\mathcal{F}(U^{k}\!,V^{k}\!,W^{k})\!-\!\frac{\langle U,\mathcal{F}_{1}(U^{k})\!+\!\mathcal{F}_{2}(U^{k})\!+\!Q\!+\!\tau^{k} U^{k}\rangle}{1+3\rho^{k}}\!+\!c
\end{aligned}
\end{equation*}
\end{small}

\noindent
where $\tau^{k}>0$ is the Lipschitz constant of $\nabla \mathcal{F}(U,V^{k}\!,W^{k})$, $c$ is a constant, $Q\!=\!\rho^{k}(\mathcal{M}^{k}_{1})_{(1)}\mathcal{B}^{T}_{(1)}$, $\mathcal{F}_{1}(U^{k})\!=\!\frac{1}{2}[(\mathcal{M}^{k}_{1})_{(1)}(\mathcal{M}^{k}_{1})^{T}_{(1)}\!-\!\mu(1\!+\!3\rho^{k}) L_{1}]U^{k}$, and $\mathcal{F}_{2}(U^{k})\!=\!\frac{1}{2}[(\mathcal{M}^{k}_{1})_{(1)}(\mathcal{M}^{k}_{1})^{T}_{(1)}\!-\!\mu(1\!+\!3\rho^{k}) L_{1}]^{T}U^{k}$. Thus, the minimization of this problem is transformed into the following maximization
\vspace{-1mm}
\begin{equation}\label{equ29}
\begin{aligned}
&\mathop{\max}_{U}\;\langle U,\;\mathcal{F}_{1}(U^{k})+\mathcal{F}_{2}(U^{k})+Q+\tau^{k}U^{k}\rangle,\\
&\;\textup{s.t.},\;U^{T}U=I_{d_{1}}.
\end{aligned}
\end{equation}
Following~\cite{nick:mpp}, the solution of (\ref{equ29}) is given by
\vspace{-1mm}
\begin{equation}\label{equ30}
U^{k+1}= \textup{ORT}(\mathcal{F}_{1}(U^{k})+\mathcal{F}_{2}(U^{k})+Q+\tau^{k} U^{k}).
\end{equation}

In addition, $V$ and $W$ can be updated by the similar approximation procedure. Similar to Algorithm~\ref{alg1}, we can propose an efficient ADMM algorithm (called Algorithm 2) to solve the graph regularized problem~(\ref{equ9}).

\section{Algorithm Analysis}
In this section, we provide the convergence analysis and the complexity analysis for our algorithms.

\subsection{Convergence Analysis}
With the low multi-linear rank tensor decomposition in (\ref{equ6}), the problem (\ref{equ10}) is non-convex and so we can only consider local convergence~\cite{boyd:admm}. As in~\cite{xu:adm}, \cite{hu:admm}, we show below a necessary condition for local convergence.

\begin{lemma}\label{lemm1}
Let $\{\mathscr{Z}^{k}\}\!=\!\{(\{G^{k}_{n}\},\mathcal{G}^{k}\!,U^{k}\!,V^{k}\!,W^{k}\!,\mathcal{X}^{k}\!,\{Y^{k}_{n}\})\}$ be a sequence generated by Algorithm~\ref{alg1}. If the sequences $\{Y^{k}_{n}\}$ $(n\!=\!1,2,3)$ are bounded and satisfy $\sum^{\infty}_{k=0}\|Y^{k+\!1}_{n}\!-\!Y^{k}_{n}\|\!<\!\infty$, then $\mathscr{Z}^{k+\!1}\!-\!\mathscr{Z}^{k}\!\rightarrow\! 0$, and $\{\mathscr{Z}^{k}\}$ is bounded.
\end{lemma}

\begin{proof}
First, we prove that the Lagrangian function of (\ref{equ10}) is bounded. By (\ref{equ22}), we obtain
\vspace{-1mm}
\begin{displaymath}
\begin{aligned}
&\|\mathcal{X}^{k+1}-\mathcal{G}^{k+1}\!\times_{1}\!U^{k+1}\!\times_{2}\!V^{k+1}\!\times_{3}\!W^{k+1}\|^{2}_{F}\\
=&\|\mathcal{P}_{\Omega}(\mathcal{T}-\mathcal{G}^{k+1}\!\times_{1}\!U^{k+1}\!\times_{2}\!V^{k+1}\!\times_{3}\!W^{k+1})\|^{2}_{F}\\
\leq&\|\mathcal{P}_{\Omega}(\mathcal{T}-\mathcal{G}^{k+1}\!\times_{1}\!U^{k+1}\!\times_{2}\!V^{k+1}\!\times_{3}\!W^{k+1})\|^{2}_{F}\\
&+\|\mathcal{P}^{\perp}_{\Omega}(\mathcal{X}^{k}-\mathcal{G}^{k+1}\!\times_{1}\!U^{k+1}\!\times_{2}\!V^{k+1}\!\times_{3}\!W^{k+1})\|^{2}_{F}\\
=&\|\mathcal{X}^{k}-\mathcal{G}^{k+1}\!\times_{1}\!U^{k+1}\!\times_{2}\!V^{k+1}\!\times_{3}\!W^{k+1}\|^{2}_{F}.
\end{aligned}
\end{displaymath}
Thus, we have
\vspace{-1mm}
\begin{equation*}
\begin{aligned}
&\mathcal{L}_{\rho^{k}}(\{G^{k+1}_{n}\},\mathcal{G}^{k+1},U^{k+1},V^{k+1},W^{k+1},\mathcal{X}^{k+1},\{Y^{k}_{n}\})\\
\leq&\mathcal{L}_{\rho^{k}}(\{G^{k+1}_{n}\},\mathcal{G}^{k+1},U^{k+1},V^{k+1},W^{k+1},\mathcal{X}^{k},\{Y^{k}_{n}\}).
\end{aligned}
\end{equation*}

Similarly, by the iteration procedure, we have
\vspace{-1mm}
\begin{equation*}
\begin{aligned}
&\mathcal{L}_{\rho^{k}}(\{G^{k+1}_{n}\},\mathcal{G}^{k+1},U^{k+1},V^{k+1},W^{k+1},\mathcal{X}^{k+1},\{Y^{k}_{n}\})\\
\leq&\mathcal{L}_{\rho^{k}}(\{G^{k+1}_{n}\},\mathcal{G}^{k+1},U^{k+1},V^{k+1},W^{k+1},\mathcal{X}^{k},\{Y^{k}_{n}\})\\
\leq&\mathcal{L}_{\rho^{k}}(\{G^{k+1}_{n}\},\mathcal{G}^{k},U^{k},V^{k},W^{k},\mathcal{X}^{k},\{Y^{k}_{n}\})\\
\leq&\mathcal{L}_{\rho^{k}}(\{G^{k}_{n}\},\mathcal{G}^{k},U^{k},V^{k},W^{k},\mathcal{X}^{k},\{Y^{k}_{n}\})\\
=&\mathcal{L}_{\rho^{k-1}}(\{G^{k}_{n}\},\mathcal{G}^{k},U^{k},V^{k},W^{k},\mathcal{X}^{k},\{Y^{k-1}_{n}\})\\
&+\theta_{k}\sum^{3}_{n=1}\| Y^{k}_{n}-Y^{k-1}_{n}\|^2_{F}
\end{aligned}
\end{equation*}
where $\theta_{k}\!=\!(\rho^{k-\!1}\!+\!\rho^{k})/(2(\rho^{k-\!1})^{2})$. According to $\sum^{\infty}_{k=0}\!\|Y^{k+\!1}_{n}\!-\!Y^{k}_{n}\|\!<\!\infty$, we have that
$\mathcal{L}_{\rho^{k}}(\{G^{k+\!1}_{n}\},\mathcal{G}^{k+\!1}\!,$ $U^{k+\!1}\!,V^{k+\!1}\!,W^{k+\!1}\!,\mathcal{X}^{k+\!1}\!,\{Y^{k}_{n}\})$ is upper-bounded. By (\ref{equ11}), we have
\vspace{-1mm}
\begin{displaymath}
\begin{split}
&\mathcal{L}_{\rho^{k}}(\{G^{k+1}_{n}\},\mathcal{G}^{k+1},U^{k+1},V^{k+1},W^{k+1},\mathcal{X}^{k+1},\{Y^{k}_{n}\})\\
\geq&\frac{1}{2}\sum^{3}_{n=1}\frac{\|Y^{k+1}_{n}\|^2_{F}-\|Y^{k}_{n}\|^2_{F}}{\rho^{k}}.
\end{split}
\end{displaymath}
Then $\mathcal{L}_{\rho^{k}}(\{G^{k+\!1}_{n}\},\mathcal{G}^{k+\!1}\!,U^{k+\!1}\!,V^{k+\!1}\!,W^{k+\!1}\!,\mathcal{X}^{k+\!1}\!,\{Y^{k}_{n}\})$ is lower-bounded due to the boundedness of $\{Y^{k}_{n}\}$. That is, $\{\mathcal{L}_{\rho^{k}}(\{G^{k+\!1}_{n}\},\mathcal{G}^{k+\!1}\!,U^{k+\!1}\!,V^{k+\!1}\!,W^{k+\!1}\!,\mathcal{X}^{k+\!1}\!,\{Y^{k}_{n}\})\}$ is bounded. Furthermore, by $Y^{k+\!1}_{n}\!=\!Y^{k}_{n}+\rho^{k}(\mathcal{G}^{k+\!1}_{(n)}\!-\!G^{k+\!1}_{n})$, $n\!=\!1,2,3$, then $\{\mathcal{L}_{\rho^{k+\!1}}(\mathscr{Z}^{k+\!1})\}$ is also bounded.

$\mathcal{L}_{\rho}(\cdot)$ is strong convex with respect to $G_{n}\,(n\!=\!1,2,3)$ and $\mathcal{X}$, respectively. For any $G_{n}$ and $\triangle G_{n}$, we have
\vspace{-1mm}
\begin{equation}\label{equ31}
\begin{split}
&\mathcal{L}_{\rho}(G_{n}+\triangle G_{n},\cdot)-\mathcal{L}_{\rho}(G_{n},\cdot)\\
\geq& \langle\partial_{G_{n}}\mathcal{L}_{\rho}(G_{n},\cdot),\triangle G_{n}\rangle+\frac{\rho}{2}\|\triangle G_{n}\|^{2}_{F}
\end{split}
\end{equation}
where $\mathcal{L}_{\rho}(G_{n},\cdot)$ denotes that all variables except $G_{n}$ are fixed. Moreover, $G^{*}_{n}$ is a minimizer of $\mathcal{L}_{\rho}(G_{n},\cdot)$ if
\vspace{-1mm}
\begin{equation}\label{equ32}
\langle\partial_{G_{n}}\mathcal{L}_{\rho}(G^{*}_{n},\cdot), \;\triangle G_{n}\rangle\geq 0.
\end{equation}
Note that $G^{k+1}_{n}$ is a minimizer of $\mathcal{L}_{\rho}(G_{n},\cdot)$ at the $k$-th iteration. Then, combining (\ref{equ31}) and (\ref{equ32}), we have
\vspace{-1mm}
\begin{equation}\label{equ33}
\mathcal{L}_{\rho^{k}}(G^{k}_{n},\cdot)-\mathcal{L}_{\rho^{k}}(G^{k+1}_{n},\cdot)\geq \frac{\rho^{k}}{2}\| G^{k}_{n}-G^{k+1}_{n}\|^{2}_{F}.
\end{equation}
Similarly, we obtain
\vspace{-1mm}
\begin{equation}\label{equ34}
\mathcal{L}_{\rho^{k}}(\mathcal{X}^{k},\cdot)-\mathcal{L}_{\rho^{k}}(\mathcal{X}^{k+1},\cdot)\geq \frac{1}{2}\|\mathcal{X}^{k}-\mathcal{X}^{k+1}\|^{2}_{F}.
\end{equation}
Using (\ref{equ14}), we have
\begin{equation}\label{equ35}
\begin{split}
&\mathcal{L}_{\!\rho^{k}}(\{G^{k\!+\!1}_{n}\},\mathcal{G}^{k}\!,U^{k}\!,V^{k}\!,W^{k}\!,\mathcal{X}^{k}\!,\{Y^{k}_{n}\})-\mathcal{L}_{\!\rho^{k}}(\{G^{k\!+\!1}_{n}\},\\
&\quad\quad \mathcal{G}^{k\!+\!1}\!,U^{k\!+\!1}\!,V^{k\!+\!1}\!,W^{k\!+\!1}\!,\mathcal{X}^{k}\!,\{Y^{k}_{n}\})\geq 0.
\end{split}
\end{equation}
By (\ref{equ33})-(\ref{equ35}) and $Y^{k+\!1}_{n}=Y^{k}_{n}+\rho^{k}(\mathcal{G}^{k+\!1}_{(n)}-G^{k+\!1}_{n})$, we have
\vspace{-1mm}
\begin{equation*}
\begin{split}
&\mathcal{L}_{\rho^{k}}(\mathscr{Z}^{k})-\mathcal{L}_{\rho^{k+1}}(\mathscr{Z}^{k+1})\\
=&\mathcal{L}_{\rho^{k}}(\mathscr{Z}^{k})-\mathcal{L}_{\rho^{k}}(\{G^{k+1}_{n}\},\mathcal{G}^{k},U^{k},V^{k},W^{k},\mathcal{X}^{k},\{Y^{k}_{n}\})\\
&+\mathcal{L}_{\rho^{k}}(\{G^{k+1}_{n}\},\mathcal{G}^{k},U^{k},V^{k},W^{k},\mathcal{X}^{k},\{Y^{k}_{n}\})\\
&-\mathcal{L}_{\rho^{k}}(\{G^{k+1}_{n}\},\mathcal{G}^{k+1},U^{k+1},V^{k+1},W^{k+1},\mathcal{X}^{k},\{Y^{k}_{n}\})\\
&+\mathcal{L}_{\rho^{k}}(\{G^{k+1}_{n}\},\mathcal{G}^{k+1},U^{k+1},V^{k+1},W^{k+1},\mathcal{X}^{k},\{Y^{k}_{n}\})\\
&-\mathcal{L}_{\rho^{k}}(\{G^{k+1}_{n}\},\mathcal{G}^{k+1},U^{k+1},V^{k+1},W^{k+1},\mathcal{X}^{k+1},\{Y^{k}_{n}\})\\
&+\mathcal{L}_{\rho^{k}}(\{G^{k+1}_{n}\},\mathcal{G}^{k+1},U^{k+1},V^{k+1},W^{k+1},\mathcal{X}^{k+1},\{Y^{k}_{n}\})\\
&-\mathcal{L}_{\rho^{k}}(\mathscr{Z}^{k+1})+\mathcal{L}_{\rho^{k}}(\mathscr{Z}^{k+1})-\mathcal{L}_{\rho^{k+1}}(\mathscr{Z}^{k+1})\\
\geq & \frac{\rho^{k}}{2}\sum^{3}_{n=1}\|G^{k+\!1}_{n}\!-\!G^{k}_{n}\|^{2}_{F}-\theta_{k+\!1}\sum^{3}_{n=1}\|Y^{k+\!1}_{n}\!-\!Y^{k}_{n}\|^{2}_{F}\\
&+\frac{1}{2}\|\mathcal{X}^{k+\!1}\!-\!\mathcal{X}^{k}\|^{2}_{F}.
\end{split}
\end{equation*}
By taking summation of the above inequality from $1$ to $\infty$, and by the boundedness of $\{\mathcal{L}_{\rho^{k}}(\mathscr{Z}^{k})\}$, then
\vspace{-1mm}
\begin{equation*}
\begin{split}
&\sum^{\infty}_{k=1}\left(\frac{\rho^{k}}{2}\sum^{3}_{n=1}\|G^{k+1}_{n}-G^{k}_{n}\|^{2}_{F}+\frac{1}{2}\|\mathcal{X}^{k+1}-\mathcal{X}^{k}\|^{2}_{F}\right)\\
&-\sum^{\infty}_{k=1}\left(\theta_{k+1}\sum^{3}_{n=1}\|Y^{k+1}_{n}-Y^{k}_{n}\|^{2}_{F}\right)<\infty
\end{split}
\end{equation*}
and by $\sum^{\infty}_{k=1}\!\|Y^{k+\!1}_{n}\!-\!Y^{k}_{n}\|^{2}_{F}\!<\!\infty$, we have $\sum^{\infty}_{k=1}\!\|\mathcal{X}^{k+\!1}\!-\!\mathcal{X}^{k}\|^{2}_{F}\!<\!\infty$ and $\sum^{\infty}_{k=1}\sum^{3}_{n=1}\!\|G^{k+\!1}_{n}\!-\!G^{k}_{n}\|^{2}_{F}\!<\!\infty$, i.e., $\mathcal{X}^{k+\!1}\!-\!\mathcal{X}^{k}\!\rightarrow\! 0$ and $G^{k+\!1}_{n}\!-\!G^{k}_{n}\!\rightarrow\! 0$. By (\ref{equ14}), we have $\mathcal{G}^{k+\!1}\!-\!\mathcal{G}^{k}\!\rightarrow\! 0$, $U^{k+\!1}\!-\!U^{k}\!\rightarrow\! 0$, $V^{k+\!1}\!-\!V^{k}\!\rightarrow\! 0$, and $W^{k+\!1}\!-\!W^{k}\!\rightarrow\! 0$, i.e., $\mathscr{Z}^{k+\!1}\!-\!\mathscr{Z}^{k}\!\rightarrow\! 0$.

Furthermore, we have
\vspace{-1mm}
\begin{displaymath}
\begin{split}
&\sum^{3}_{n=1}\|G^{k}_{n}\|_{*}/(3\lambda)+\frac{1}{2}\|\mathcal{X}^{k}-\mathcal{G}^{k}\!\times_{1}\!U^{k}\!\times_{2}\!V^{k}\!\times_{3}\!W^{k}\|^{2}_{F}\\
=&\mathcal{L}_{\rho^{k-\!1}}(\{G^{k}_{n}\},\mathcal{G}^{k},U^{k},V^{k},W^{k},\mathcal{X}^{k},\{Y^{k-1}_{n}\})\\
&-\frac{1}{2\rho^{k-1}}\sum^{3}_{n=1}\left(\|Y^{k}_{n}\|^2_{F}-\|Y^{k-1}_{n}\|^2_{F}\right)
\end{split}
\end{displaymath}
is upper-bounded due to the boundedness of $\{Y^{k}_{n}\}$ for all $n\!\in\!\{1,2,3\}$ and $\mathcal{L}_{\rho^{k-\!1}}(\{G^{k}_{n}\},\mathcal{G}^{k}\!, U^{k}\!,V^{k}\!,W^{k}\!,\mathcal{X}^{k}\!,\{Y^{k-\!1}_{n}\})$. According to $Y^{k}_{n}\!=\!Y^{k-\!1}_{n}\!+\!\rho^{k-\!1}(\mathcal{G}^{k}_{(n)}\!-\!G^{k}_{n})$ and (\ref{equ27}), thus $\{G^{k}_{n}\}$, $\{\mathcal{G}^{k}\}$ and $\{\mathcal{X}^{k}\}$ are all bounded, i.e., $\{\mathscr{Z}^{k}\}$ is bounded.
\end{proof}

\begin{theorem}\label{theo4}
Let $\{(\{G^{k}_{n}\},\mathcal{G}^{k},U^{k}\!,V^{k}\!,W^{k}\!,\mathcal{X}^{k}\!,\{Y^{k}_{n}\})\}$ be a sequence generated by Algorithm~\ref{alg1}. Then any accumulation point of $\{(\{G^{k}_{n}\},\mathcal{G}^{k}\!,U^{k}\!,V^{k}\!,W^{k}\!,\mathcal{X}^{k})\}$ satisfies the KKT conditions for (\ref{equ10}).
\end{theorem}

The proof of Theorem~\ref{theo4} is given in APPENDIX D. Moreover, the convergence of Algorithm 2 can also be guaranteed. We first give the following lemma~\cite{bertsekas:np}.

\begin{lemma}\label{lemm2}
Let $F:\mathbb{R}^{m\times n}\rightarrow \mathbb{R}$ be a continuously differentiable function with Lipschitz continuous gradient and Lipschitz constant $L(F)$. Then, for any $\tau\geq L(F)$,
\vspace{-1mm}
\begin{displaymath}
\begin{split}
F(X)\leq F(Y)+\langle\nabla F(X),\;&{X}-{Y}\rangle+\frac{\tau}{2}\|{X}-{Y}\|^{2}_{F},\\
\forall\, X,\,Y\in& \mathbb{R}^{m\times n}.
\end{split}
\end{displaymath}
\end{lemma}

\begin{theorem}\label{theo5}
Let $\{\mathscr{Z}^{k}\}\!=\!\{(\!\{G^{k}_{n}\},\mathcal{G}^{k}\!,U^{k}\!,V^{k}\!,W^{k}\!,\mathcal{X}^{k}\!,\{Y^{k}_{n}\}\!)\}$ be a sequence generated by Algorithm 2. If the sequences $\{Y^{k}_{n}\}$ $(n\!=\!1,2,3)$ are all bounded, and satisfy $\sum^{\infty}_{k=0}\!\|Y^{k+1}_{n}\!-\!Y^{k}_{n}\|\!<\!\infty$, then any accumulation point of $\{(\{G^{k}_{n}\},\mathcal{G}^{k}\!,U^{k}\!,V^{k}\!,W^{k}\!,\mathcal{X}^{k})\}$ satisfies the KKT conditions for (\ref{equ9}).
\end{theorem}

\begin{proof}
By Lemma~\ref{lemm2}, we have
\vspace{-1mm}
\begin{displaymath}
\begin{split}
&\mathcal{L}_{\rho^{k}}(\{G^{k+\!1}_{n}\},\mathcal{G}^{k}\!,U^{k}\!,V^{k}\!,W^{k}\!,\mathcal{X}^{k}\!,\{Y^{k}_{n}\})-\mathcal{L}_{\rho^{k}}(\{G^{k+\!1}_{n}\},\\
&\quad\quad \mathcal{G}^{k+\!1}\!,U^{k+\!1}\!,V^{k+\!1}\!,W^{k+\!1}\!,\mathcal{X}^{k}\!,\{Y^{k}_{n}\})\geq 0.
\end{split}
\end{displaymath}
According to Lemma~\ref{lemm1}, we have $\mathscr{Z}^{k+1}\!-\!\mathscr{Z}^{k}\!\rightarrow\! 0$, and $\{\mathscr{Z}^{k}\}$ is bounded. Moreover, by using the same proof procedure as for Theorem~\ref{theo4}, we can obtain the conclusion.
\end{proof}

\subsection{Complexity Analysis}
We discuss the time complexity of our ROID and GROID methods. For solving both (\ref{equ6}) and (\ref{equ9}), the main running time of our algorithms is taken for performing SVDs and some multiplications. The time complexity of performing SVDs in (\ref{equ13}), (\ref{equ18}) and (\ref{equ19}) is $O(\sum_{n}\!{d^{2}_{n}}{\Pi_{j\neq{n}}}{d_{j}}\!+\!\sum_{n}\!{d^{2}_{n}}{I_{n}})$. The time complexity of some multiplication operators in (\ref{equ18}), (\ref{equ19}) and (\ref{equ22}) is ${O}((2d_{1}\!+\!d_{2}\!+\!d_{3})\Pi_{j}I_{j}\!+\!\sum_{n}\!I_{n}\Pi_{j}d_{j})$. Thus, the total time complexity of both ROID and GROID is ${O}((2d_{1}\!+\!d_{2}\!+\!d_{3})\Pi_{j}I_{j})$ $(d_{n}\!\ll\! I_{n})$. Our algorithms are essentially the Gauss-Seidel-type schemes of ADMM, and the update strategy of the Jacobi version as in~\cite{shang:hotd} can be easily implemented, and well suited for parallel computing.

\section{EXPERIMENTAL RESULTS}
In this section, we evaluate the effectiveness and efficiency of our ROID method for low multi-linear rank tensor completion on synthetic data and multi-relational learning on real-world data such as a network data set and three popular multi-relational data sets. Except for large-scale multi-relational prediction, all the other experiments were performed on an Intel(R) Core (TM) i5-4570 (3.20 GHz) PC running Windows 7 with 8GB main memory.

\subsection{Results on Synthetic Data}
Following~\cite{liu:tcem}, we generated low multi-linear rank third-order tensors $\mathcal{T}\!\in\!\mathbb{R}^{I_{1}\times I_{2}\times I_{3}}$, which we used as the ground truth data. The generated tensor data follows the Tucker model, i.e., $\mathcal{T}\!=\!\mathcal{C}\!\times_{1}\!U_{1}\!\times_{2}\!U_{2}\!\times_{3}\!U_{3}$, where $\mathcal{C}\!\in\! \mathbb{R}^{r\times r\times r}$ is the core tensor whose entries are generated as independent and identically distributed (i.i.d.) numbers from a uniform distribution in [0, 1], and the entries of $U_{n}\!\in\!\mathbb{R}^{I_{n}\times r}$ are random samples drawn from a uniform
distribution in the range [-0.5, 0.5]. With this construction, the multi-linear rank of third-order tensors $\mathcal{T}$ equals $(r,r,r)$ almost surely.

\subsubsection{Algorithm Settings}
We compare our ROID method with the following state-of-the-art tensor estimation algorithms:
\begin{enumerate}
\item WTucker{\footnote{\url{http://www.lair.irb.hr/ikopriva/marko-filipovi.html}}}~\cite{filipovic:tftc}: In the implementation of WTucker, we set ${R}_{1}\!=\!{R}_{2}\!=\!{R}_{3}\!=\!\lfloor1.25r\rfloor$ for solving the weighted Tucker (WTucker) decomposition problem (\ref{equ2}).
\item WCP{\footnote{\url{http://www.sandia.gov/~tgkolda/TensorToolbox/}}}~\cite{acar:stf}: We set the tensor rank $R\!=\!40$ to solve the weighted CP (WCP) decomposition problem (\ref{equ1}), and the maximal number of iterations, $\textup{maxiter}\!=\!100$, for WCP and WTucker, both of which are solved by nonlinear conjugate gradient methods.
\item HaLRTC{\footnote{\url{http://pages.cs.wisc.edu/~ji-liu/}}}~\cite{liu:tcem}: The value of the weights $\alpha_{n}$ is set to be $1/3,\,n\!=\!1,2,3$ for solving (\ref{equ3}) by using the highly accurate LRTC (HaLRTC) algorithm. Other parameters of HaLRTC are set to their default values.
\item Latent{\footnote{\url{http://ttic.uchicago.edu/~ryotat/softwares/tensor/}}}~\cite{tomioka:ctd}: The regularization parameter $\lambda$ is set to $10^{6}$ for solving the latent trace norm minimization (Latent) problem (\ref{equ4}). Moreover, we set the tolerance value $\textup{tol}=10^{-5}$ and $\textup{maxiter}=500$ for HaLRTC, Latent and ROID.
\end{enumerate}

We also apply Algorithm~\ref{alg1} to solve the SHOOI problem (\ref{equ8}). Note that HaLRTC, Latent, SHOOI and ROID all apply the ADMM algorithm to solve their problems. For our ROID method, we set the regularization parameter $\lambda\!=\!10^{2}$ and ${d}_{1}\!=\!{d}_{2}\!=\!{d}_{3}\!=\!\lfloor1.5r\rfloor$. The relative square error (RSE) of the recovered tensor $\mathcal{X}$ for all these algorithms is defined by $\textup{RSE}\!:=\!\|\mathcal{X}-\mathcal{T}\|_{F}/\|\mathcal{T}\|_{F}$.

\subsubsection{Numerical Results on Sparse Tensors}
To evaluate the robustness of our ROID method with respect to multi-linear rank parameter changes, we first conduct some experiments on synthetic tensors of size $100\!\times\!100\!\times\!100$ or $200\!\times\!200\!\times\!200$, and illustrate the RSE results of all these tensor decomposition methods with 10\% sampling ratio, where the rank parameter of ROID, SHOOI, WTucker and WCP is chosen from $\{10,15,\ldots,40\}$. The average RSE results of 10 independent runs are shown in Fig.\ \ref{fig2}, from which we can see that when the number of the given rank increases, the RSE of all these tensor decomposition methods (except WCP) gradually increase, especially for SHOOI. More specifically, SHOOI gives extremely accurate solutions for exact multi-linear rank tensor completion problems. However, as the number of the given rank increases, the RSE of SHOOI increases dramatically. In contrast, ROID under all these settings consistently outperforms WTucker and WCP in terms of RSE, and performs \emph{more robust} than SHOOI. This confirms that our ROID model with core tensor trace norm regularization is reasonable, and can provide a good estimation of the observed tensor even though from only a few observations. Note that the RSE of both convex algorithms, HaLRTC and Latent, on tensors of size $100\!\times\!100\!\times\!100$ are 0.5796 and 0.3375, respectively.

\begin{figure}[t]
\centering
\subfigure[$100\!\times\!100\!\times\!100$]{\includegraphics[width=0.495\linewidth]{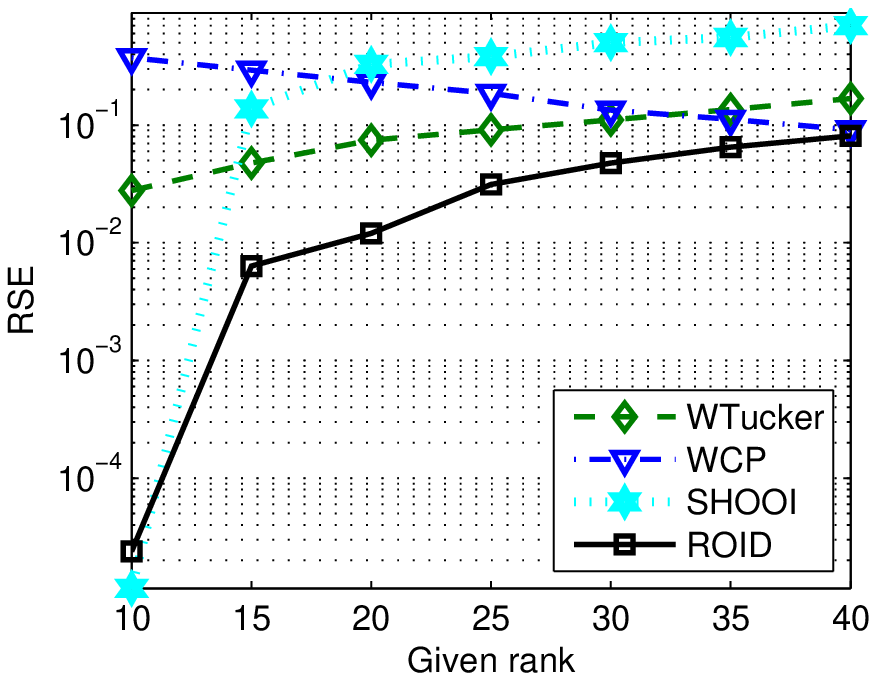}}
\subfigure[$200\!\times\!200\!\times\!200$]{\includegraphics[width=0.495\linewidth]{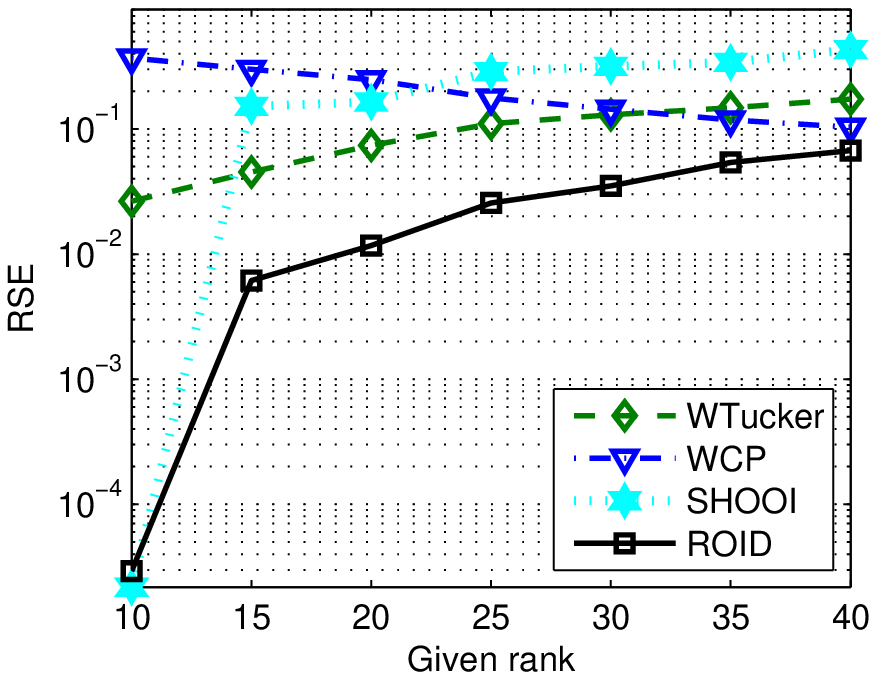}}

\vspace{-3mm}
\caption{Comparison results of WTucker, WCP, SHOOI, and ROID in terms of RSE on third-order tensors with multi-linear rank $(10,10,10)$ by varying the given rank.}
\label{fig2}
\end{figure}

We also report the recovery results of WTucker, WCP, HaLRTC, Latent and our ROID method with different fractions of observed entries and tensor multi-linear ranks on synthetic tensors of size $200\!\times\!200\!\times\!200$ in Fig.\ \ref{fig3}, where the sampling ratio varies from 5\% to 20\% with increment 2.5\%, and the multi-linear ranks $r_{n},\,n=1,2,3$ are chosen from 10 to 40 with increment 5. We can observe that in all these settings, our ROID method consistently outperforms the other methods in terms of RSE. HaLRTC is able to yield very accurate solutions using adequate large sampling ratio (e.g.\ 0.2); however, when the fraction of observed entries is low (e.g.\ 0.05), or the underlying tensor multi-linear ranks are high (e.g.\ 40), the performance of HaLRTC (and also the other convex tensor trace norm minimization method, Latent) is poor. The main reason is that WTucker and our ROID method are all multiple structured methods similar to the matrix case~\cite{haeffele:lrmf}, and need only $O(d^{N}\!+dNI)$ observations to exactly recover an \emph{N}th-order low multi-linear rank tensor $\mathcal{X}$ with high probability, while $O(rI^{N-1})$ observations are required for recovering the true tensor by both convex tensor trace norm minimization methods, HaLRTC and Latent, as stated in~\cite{mu:sd, liu:tld, tomioka:spctd}.

\begin{figure}[t]
\centering
\subfigure[$(r,r,r)=(10,10,10)$]{\includegraphics[width=0.495\linewidth]{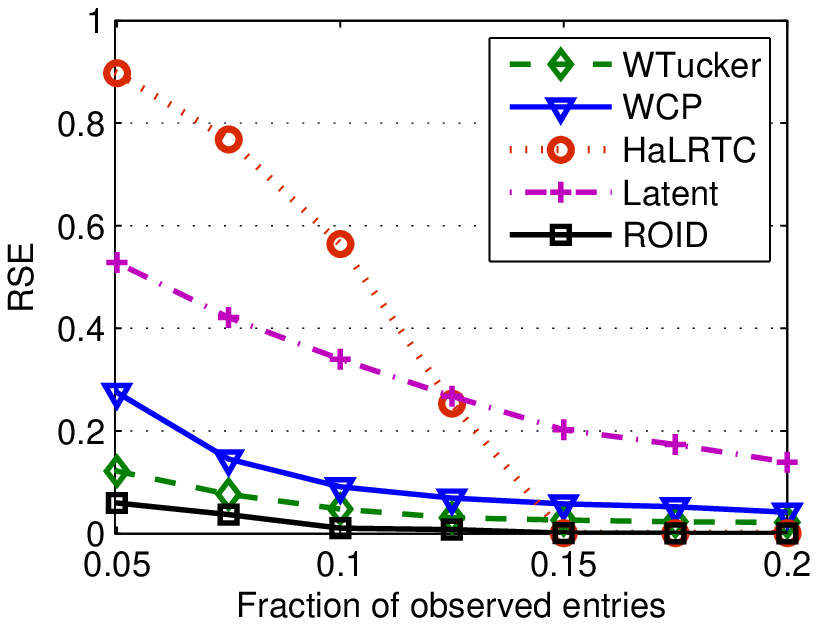}}
\subfigure[10\% sampling ratio]{\includegraphics[width=0.495\linewidth]{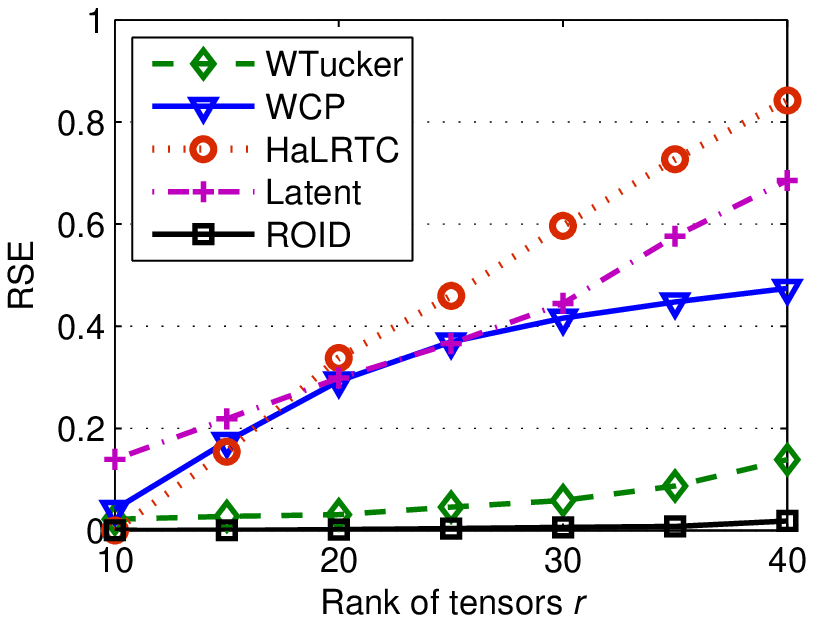}}

\vspace{-3mm}
\caption{Comparison results of WTucker, WCP, HaLRTC, Latent and our ROID method in terms of RSE on third-order tensors by varying the fraction of observed entries (a) or multi-linear rank, $(r,r,r)$ (b).}
\label{fig3}
\end{figure}

Moreover, we conduct some experiments to evaluate the robustness of our ROID method with respect to the regularization parameter $\lambda$, and report the results of Latent and our ROID method on synthetic tensors of size $100\!\times\!100\!\times\!100$ or $200\!\times\!200\!\times\!200$ in Fig.\ \ref{fig41}, where $\lambda$ is tuned from the grid $\{10^{1}, 10^{2}, \ldots, 10^{6}\}$ and the sampling ratio is set to 10\%. Note that the solid and dashed lines denote the results of ROID with ${d}_{1}\!=\!{d}_{2}\!=\!{d}_{3}\!=\!r$ and ${d}_{1}\!=\!{d}_{2}\!=\!{d}_{3}\!=\!\lfloor1.5r\rfloor$, respectively. It is clear that as $\lambda$ increases, both Latent and our ROID method with exact multi-linear rank give much better performance for tensor completion problems. As suggested in~\cite{tomioka:ctd}, setting $\lambda\rightarrow\infty$ gives more accurate solutions for the noiseless problem. In practical applications, this parameter of ROID is set to $\lambda=100$ for the following noisy problems. Moreover, our ROID method under all settings significantly outperforms Latent in terms of RSE.

Finally, we present the running time of our ROID method and the other methods with varying sizes of third-order tensors, as shown in Fig.\ \ref{fig42}, from which we can see that the running time of WTcuker, WCP, Latent and HaLRTC dramatically grows with the increase of tensor size whereas the running time of our SHOOI and ROID methods only increase slightly. In addition, WTcuker, WCP, Latent and HaLRTC could not yield experimental results on the two largest synthetic tensor completion problems with sizes of 800 and 1000, because they ran out of memory. Our ROID method is more than 10 times faster than WTcuker and WCP, more than 25 times faster than HaLRTC, and more than 150 times faster than Latent when the size of input tensors is $600\times600\times600$. This shows that our ROID method has good \emph{scalability} and can address large-scale problems. Notice that because Latent converges too slowly, we do not consider it in the following experiments. Moreover, Table 1 summarizes the time complexities of major computations in the two related weighted tensor decomposition algorithms and the two convex trace norm minimization algorithms, where the assumed sizes of the tensor and the core tensor are $I\times I\times I$ and $d\times d\times d$, respectively. From Table~\ref{tab1}, we can see that although WTucker and WCP have the computational complexity similar to our ROID method, they are much slower in practice than ROID due to their Polak-Ribiere nonlinear conjugate gradient algorithms with a time-consuming line search scheme~\cite{moller:cg}.

\begin{table}[t]
\centering
\caption{Complexities per iteration of major computations in low multi-linear rank tensor recovery algorithms.}
\label{tab1}
\vspace{-2mm}

\begin{tabular}{lc}
\hline
{Algorithms} &  {Complexity}\\
\hline
WCP \cite{acar:stf}  & $O(8dI^{3})$\\
WTucker \cite{filipovic:tftc}	& $O(8dI^{3})$\\
{HaLRTC~\cite{liu:tcem}, Latent~\cite{tomioka:ctd}} & {$O(3I^{4})$}\\
ROID and SHOOI   &$O(4dI^{3})$\\
\hline
\end{tabular}
\end{table}

\begin{figure}[!t]
\centering
\subfigure[RSE vs. $\lambda$]{\includegraphics[width=0.492\linewidth]{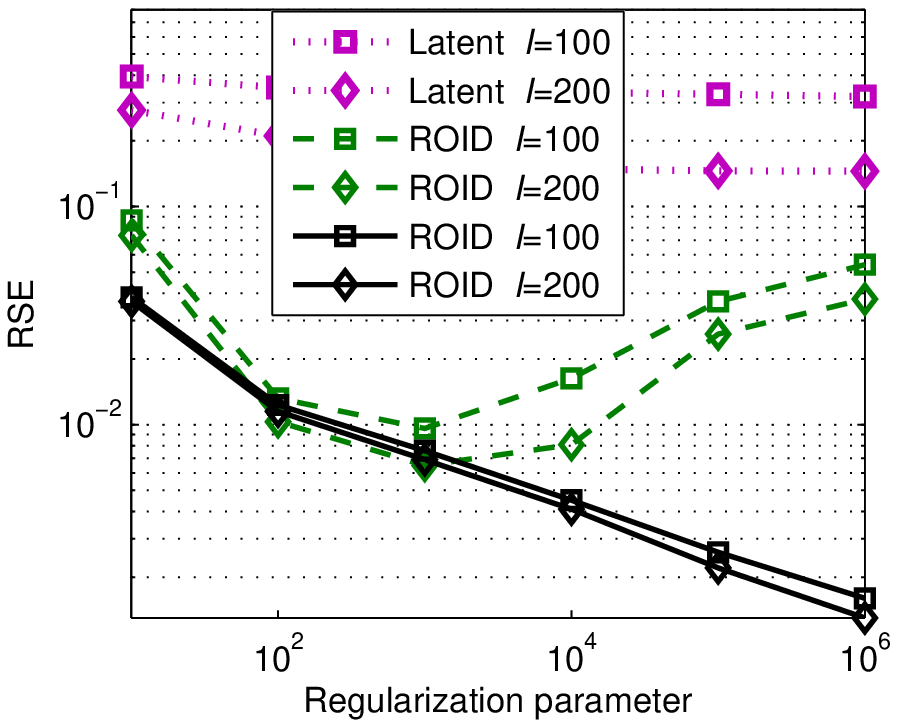}\label{fig41}}\,
\subfigure[Running time vs. size]{\includegraphics[width=0.492\linewidth]{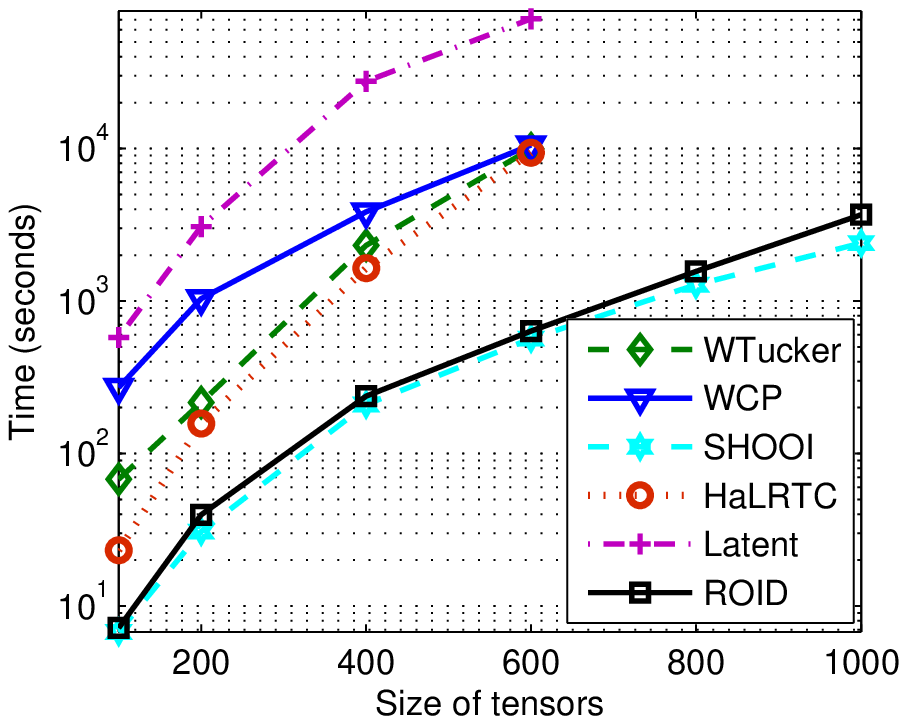}\label{fig42}}

\vspace{-3mm}
\caption{Comparison of WTucker, WCP, FaLRTC, Latent and our ROID method in terms of RSE and computational time (in seconds and in logarithmic scale) on third-order tensors by varying regularization parameter $\lambda$ (a) or tensor size (b).}
\label{fig4}
\end{figure}

\subsubsection{Numerical Results on Full Tensors}
To further evaluate the performances of our method for full tensor decomposition, we compare our ROID method with the low multi-linear rank approximation (LMLRA) method~\cite{tucker:tmfa, sorber:tensor} and HOOI~\cite{athauwer:hooi, sorber:tensor} on noisy tensors, i.e., $\mathcal{T}\!=\!\mathcal{C}\!\times_{1}\!U_{1}\!\times_{2}\!U_{2}\!\times_{3}\!U_{3}\!+\!nf\!*\!\mathcal{E}$, where $nf$ denotes the noise factor and $\mathcal{E}$ denotes the standard Gaussian random noise. Fig.\ \ref{fig51} illustrates the RSE results of LMLRA, HOOI and ROID on $200\!\times\!200\!\times\!200$ noisy tensors with different noise factors. We can observe that ROID performs more robust and stable against noise than the other methods. Moreover, we also report the running time on tensors of different sizes in Fig.\ \ref{fig52}, from which we can see that our ROID method is more than 10 times faster than the other methods. In addition, LMLRA and HOOI could not generate experimental results on the largest problem with size of $1000\!\times\!1000\!\times\!1000$, because they ran out of memory.

\begin{figure}[!t]
\centering
\subfigure[RSE vs. noise factor]{\includegraphics[width=0.492\linewidth]{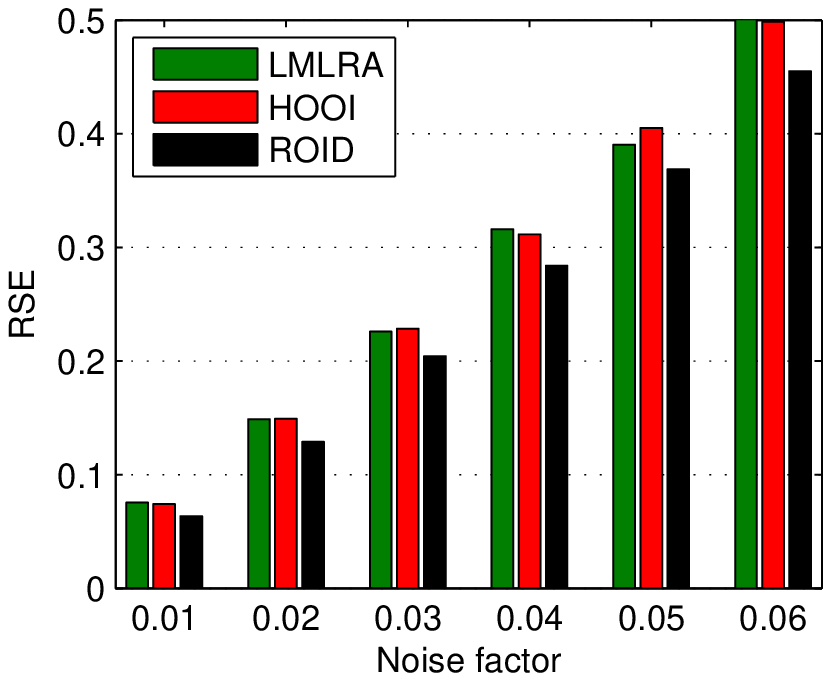}\label{fig51}}\,
\subfigure[Running time vs. size]{\includegraphics[width=0.492\linewidth]{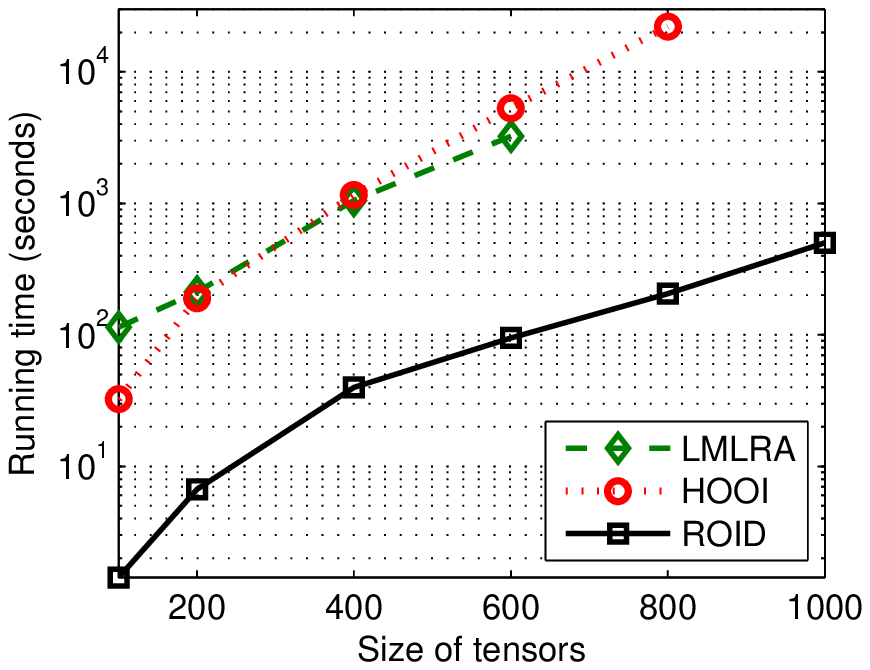}\label{fig52}}

\vspace{-3mm}
\caption{Comparison of LMLRA, HOOI and our ROID method in terms of RSE and running time (in seconds and in logarithmic scale) on third-order tensors by varying noise factor (a) or tensor size (b).}
\label{fig5}
\end{figure}

\subsection{Results on Network Data}
In this part, we examine our ROID and graph regularized (called GROID) methods on real-world network data sets, such as the YouTube data set{\footnote{\url{http://leitang.net/heterogeneous_network.html}}}~\cite{tang:hia}. YouTube is currently the most popular video sharing web site, which allows users to interact with each other in various forms such as contacts, subscriptions, sharing favorite videos, etc. In total, this data set contains 848,003 users, with 15,088 users sharing all of the information types, and includes 5-dimension of interactions: contact network, co-contact network, co-subscription network, co-subscribed network, and favorite network. Additional information about the data can be found in~\cite{tang:hia}. We run these experiments on a machine with 6-core Intel Xeon 2.4GHz CPU and 64GB memory.

\begin{figure}[!t]
\centering
\subfigure[10\% training data]{\includegraphics[width=0.492\linewidth]{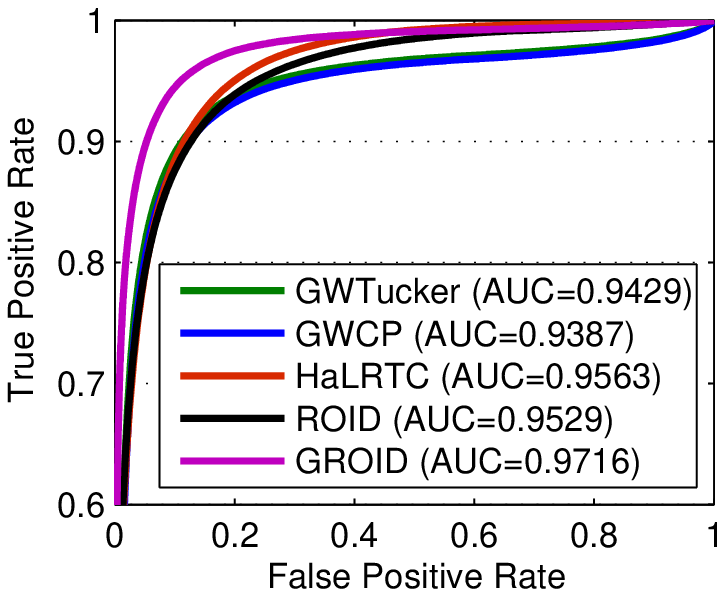}}\,
\subfigure[20\% training data]{\includegraphics[width=0.492\linewidth]{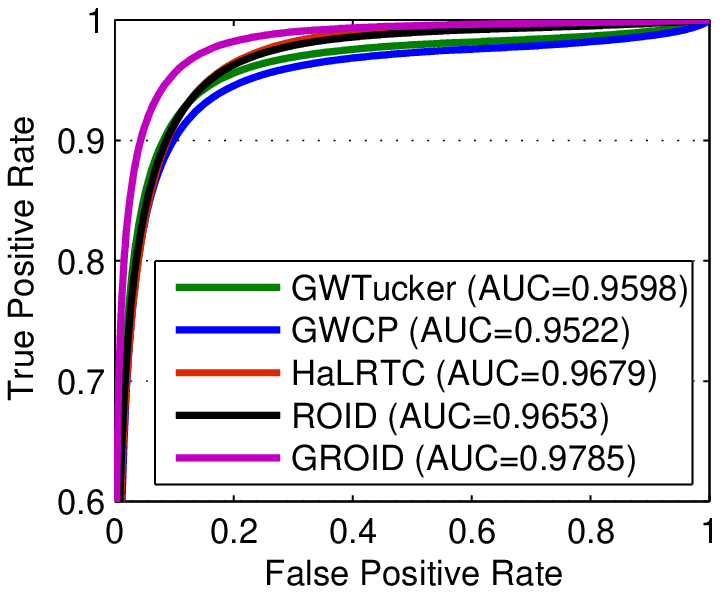}}

\vspace{-3mm}
\caption{Average ROC curves showing the performance of link prediction methods with 10\% and 20\% training data (best viewed in colors).}
\label{fig6}
\end{figure}

We address the multi-relational prediction problem as a tensor completion problem. For the graph regularized weighted CP (GWCP) decomposition~\cite{narita:tf}, graph regularized weighted Tucker (GWTucker) decomposition~\cite{narita:tf}, and our ROID and GROID methods, we set the tensor rank $R\!=\!45$ and the multi-linear rank ${d}_{1}\!=\!{d}_{2}\!=\!40$ and ${d}_{3}\!=\!5$, and the regularization parameter $\lambda=100$. For HaLRTC~\cite{liu:tcem} and our ROID and GROID methods, the weights $\alpha_{n}$ are set to $\alpha_{1}\!=\!\alpha_{2}\!=\!0.4998$ and $\alpha_{3}\!=\!0.0004$. The tolerance value of all these methods is fixed at $\textup{tol}=10^{-4}$.

As the other methods could not yield the experimental results on the whole YouTube data set, we first chose the users who have more than 10 interactions as a subset, which consists of 4,117 users and five types of interactions, i.e., $4,117\!\times\!4,117\!\times\!5$. We randomly select 10\% or 20\% entries as the training set, and the remainder as the testing data. We report the average prediction accuracy (the score Area Under the receiver operating characteristic Curve, AUC) and the average running time (seconds) over 10 independent runs in Figs.\ \ref{fig6} and \ref{fig71}, where the number of users is gradually increased. Moreover, we evaluate the robustness of our ROID method with respect to given multi-linear ranks, as shown in Fig.\ \ref{fig72}, where the given ranks of GWTucker, GWCP and our ROID and GROID methods are chosen in the range $[30,35,\ldots,70]$. We can observe that the three trace norm minimization algorithms, i.e., our ROID and GROID methods and HaLRTC, significantly outperform GWTucker and GWCP in terms of prediction accuracy. Moreover, ROID and GROID run remarkably faster than the other methods, and are more robust than GWTucker and GWCP with respect to given multi-linear ranks. The running time of ROID and GROID increase slightly when the number of users increases. In contrast, the running time of the other methods increases dramatically, and they could not complete the computation within 48 hours on the two largest problem sizes of 8,000 or 15,088 users. This shows that ROID and GROID have very good \emph{scalability} and can address large-scale problems. GROID performs \emph{significantly} better than all the other methods in terms of prediction accuracy due to the use of auxiliary information.

\begin{figure}[!t]
\centering
\subfigure[]{\includegraphics [width=0.492\linewidth]{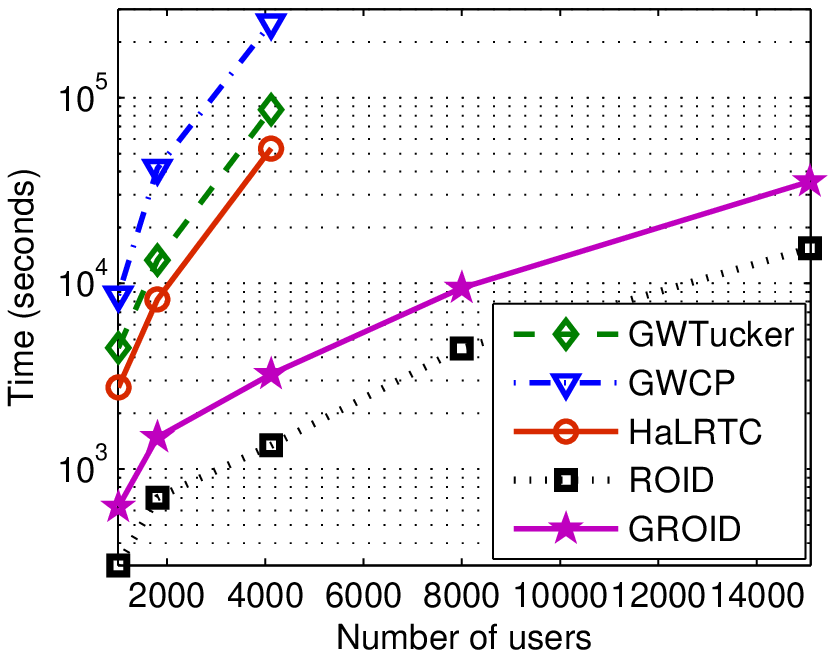}\label{fig71}}\,
\subfigure[]{\includegraphics [width=0.492\linewidth]{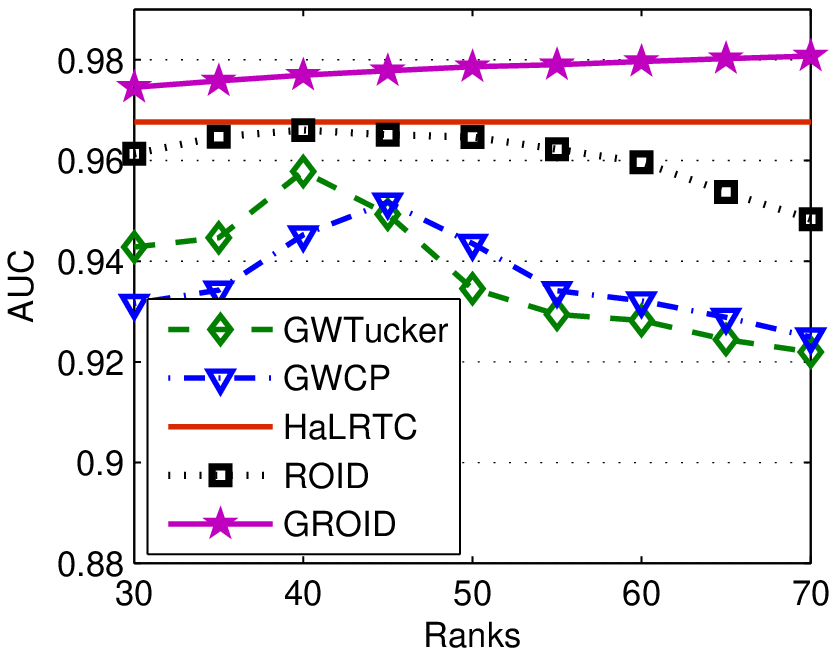}\label{fig72}}

\vspace{-3mm}
\caption{Running time (a) and prediction accuracy (b) comparison on the YouTube data set. For each dataset, we use 20\% for training. Note that the other methods could not run sizes $\{8,000, 15,088\}$ due to runtime exceptions.}
\label{fig7}
\end{figure}

\subsection{Results on Multi-Relational Data}
Finally, we examine how well our ROID method performs on the three popular multi-relational datasets, which have previously been used by Kemp \emph{et al.}~\cite{kemp:irm} for link prediction, including the Kinship, Nations and UMLS data sets. The Kinship data set consists of kinship relationships (such as ``father" or ``wife" relations) among the members of the ALyawarra tribe in Central Australia~\cite{denham:dp}. The data set contains 104 tribe members and 26 types of kinship (binary) relations, formfitting a three-order tensor of size $104\!\times\!104\!\times\!26$. The Nations data set consists of international relations among different countries in the world~\cite{rummel:dnp}. The data set contains 14 countries and 56 types of (binary) relations (such as ``Treaties" or ``Military Alliance"), and is a three-order tensor of size $14\!\times\!14\!\times\!56$. The UMLS data set is collected from the Unified Medical Language System by McCray \emph{et al.}~\cite{mcgray:ulo}. This data set includes a semantic network with 135 concepts and 49 binary predicates (such as ``affects" or ``causes"), and is a three-order tensor of size $135\!\times\!135\!\times\!49$.

\begin{figure}[t]
\centering
\subfigure[Nations]{\includegraphics[width=0.490\linewidth]{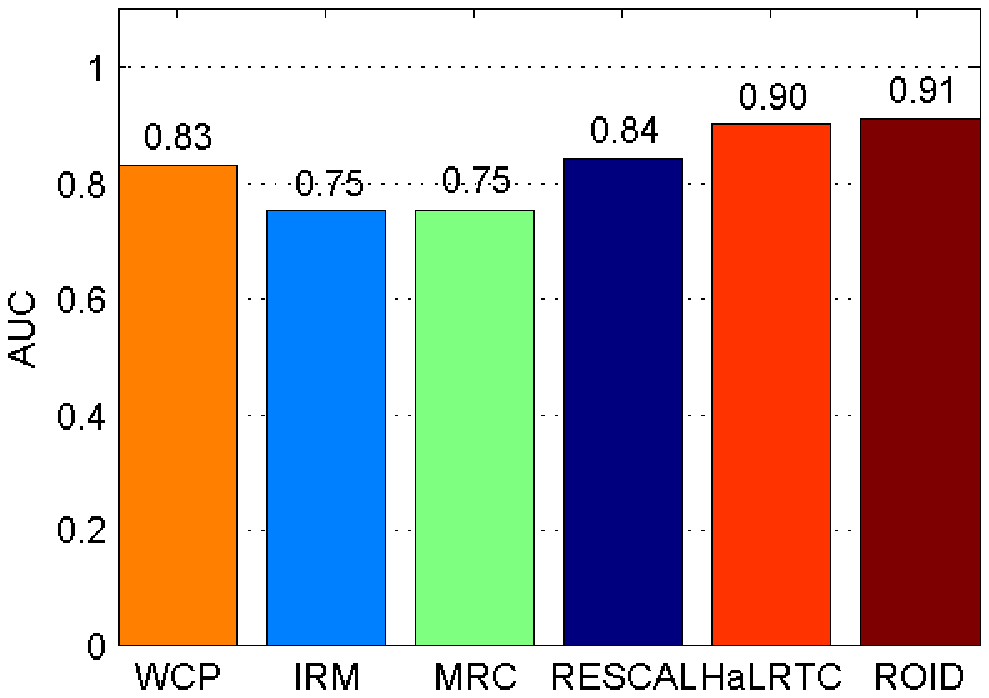}}\,
\subfigure[Kinship]{\includegraphics[width=0.490\linewidth]{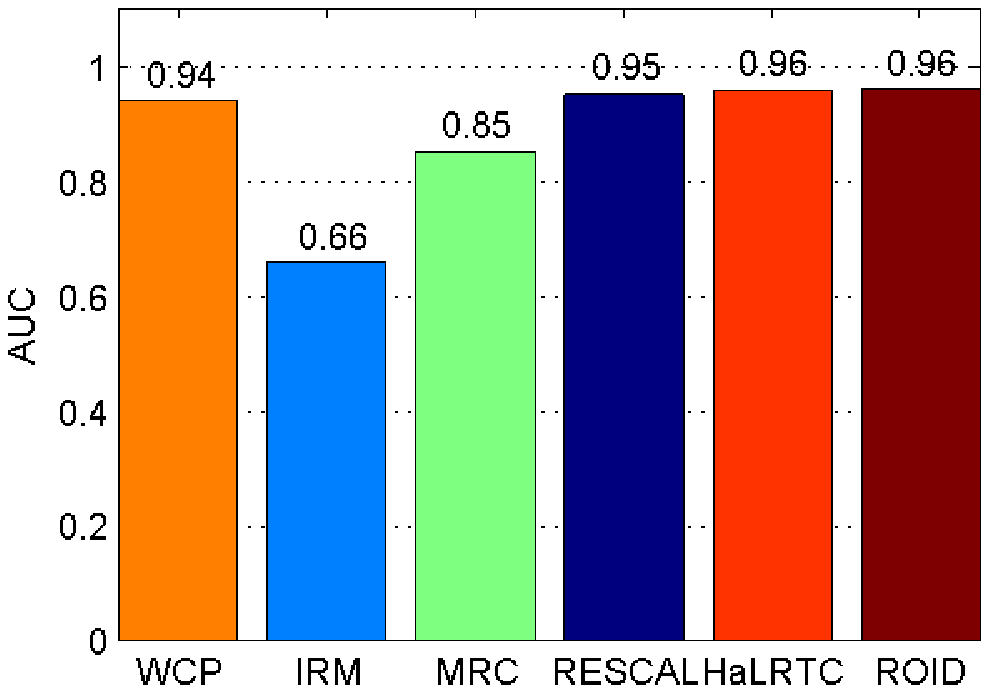}}
\subfigure[UMLS]{\includegraphics[width=0.490\linewidth]{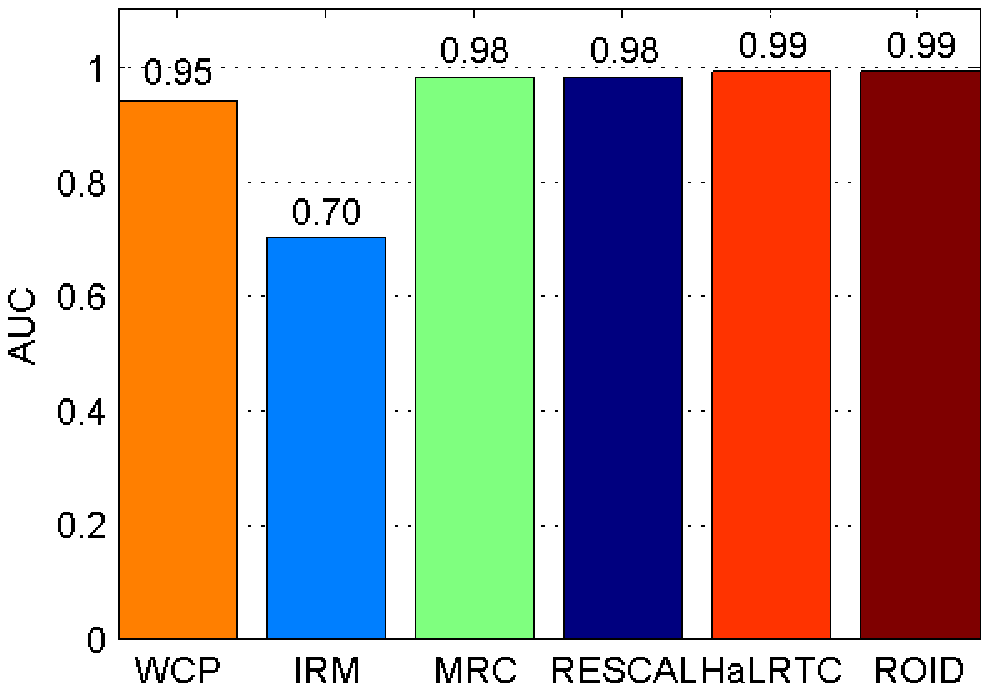}}

\vspace{-3mm}
\caption{Link predication results on the Nations, Kinship and UMLS date sets (best viewed in colors).}
\label{fig8}
\end{figure}

\begin{figure}[!t]
\centering
\subfigure[Conferences]{\includegraphics[width=0.49\linewidth]{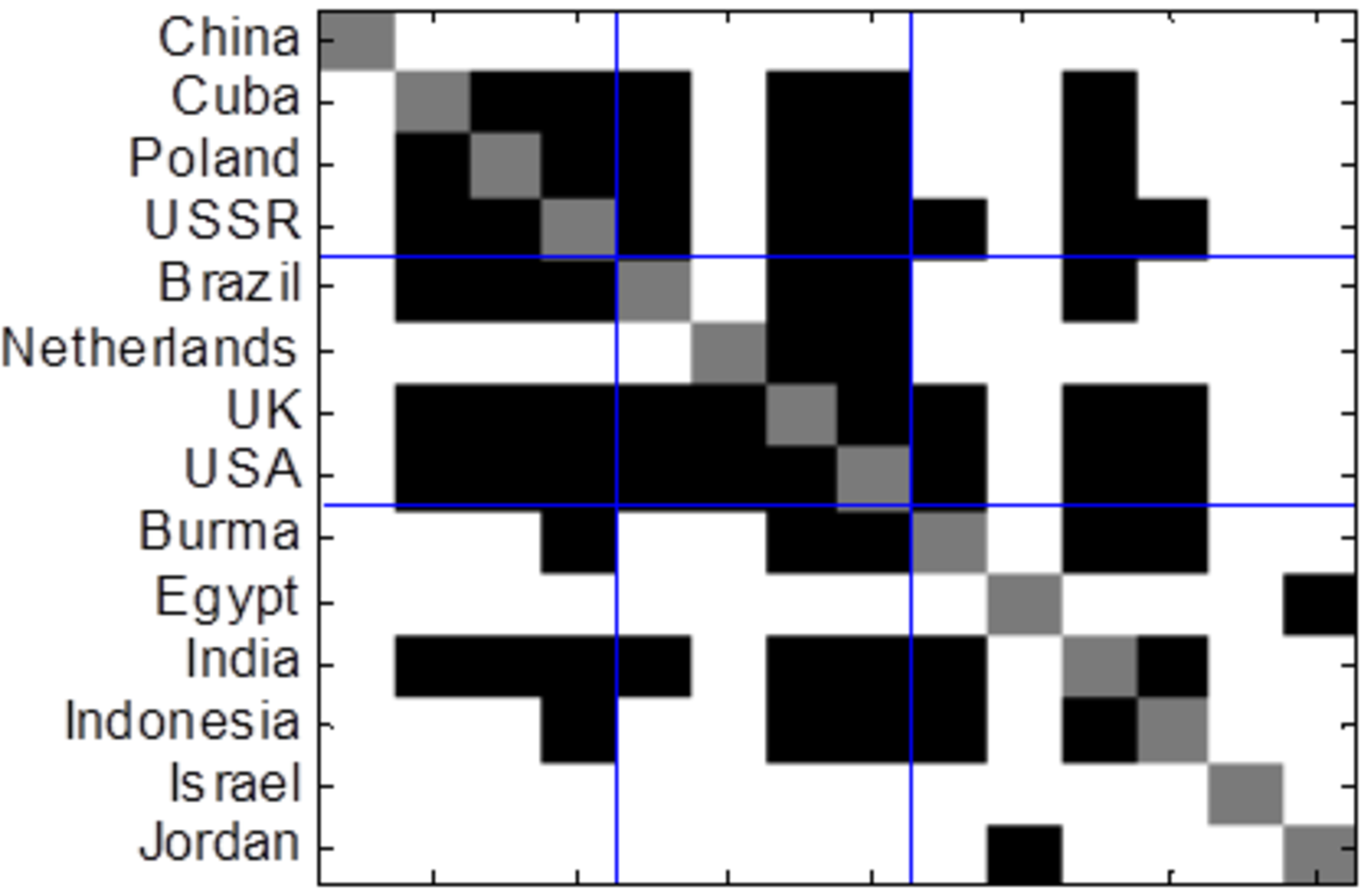}}\,\,
\subfigure[Military Alliance]{\includegraphics[width=0.432\linewidth]{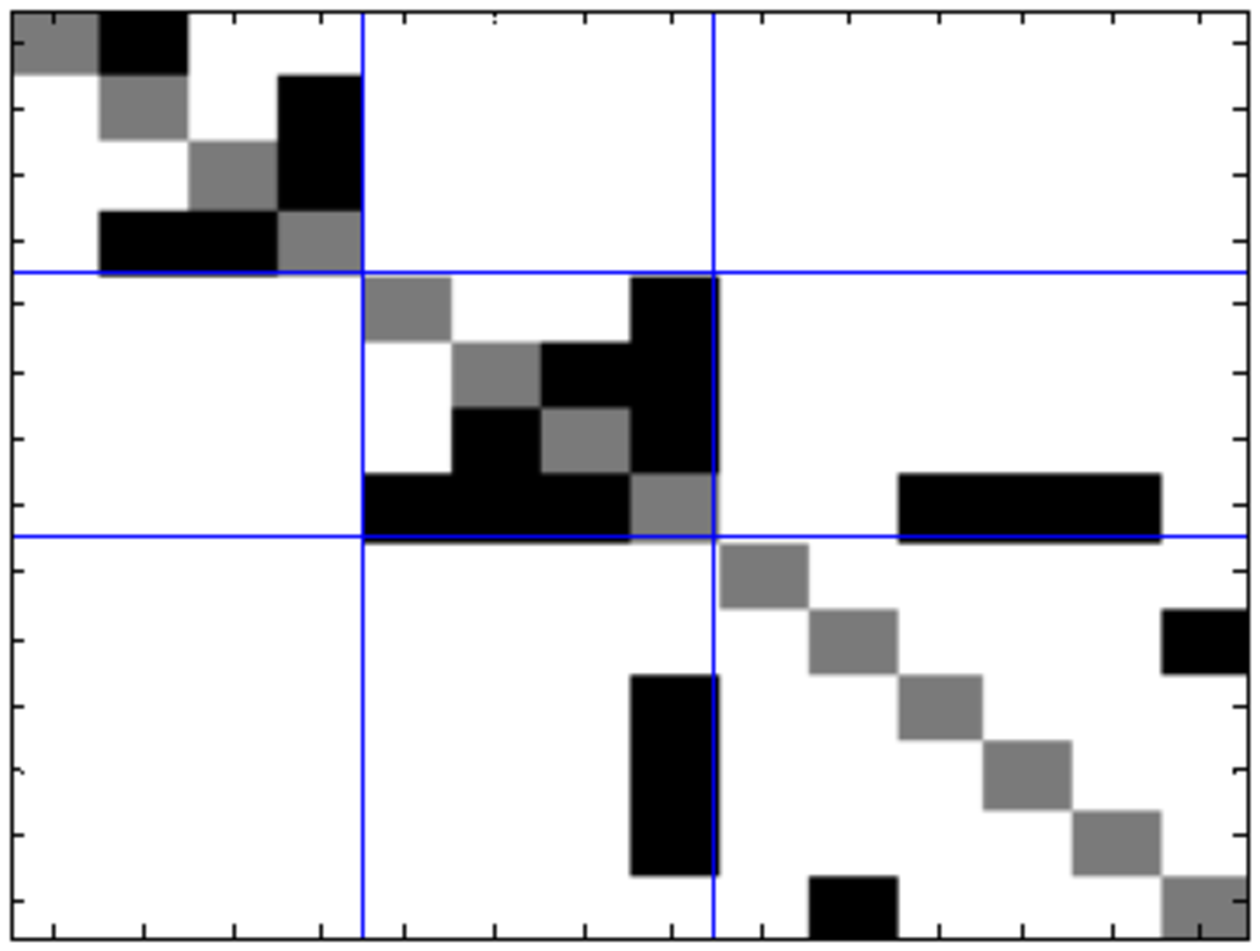}}
\subfigure[Treaties]{\includegraphics[width=0.49\linewidth]{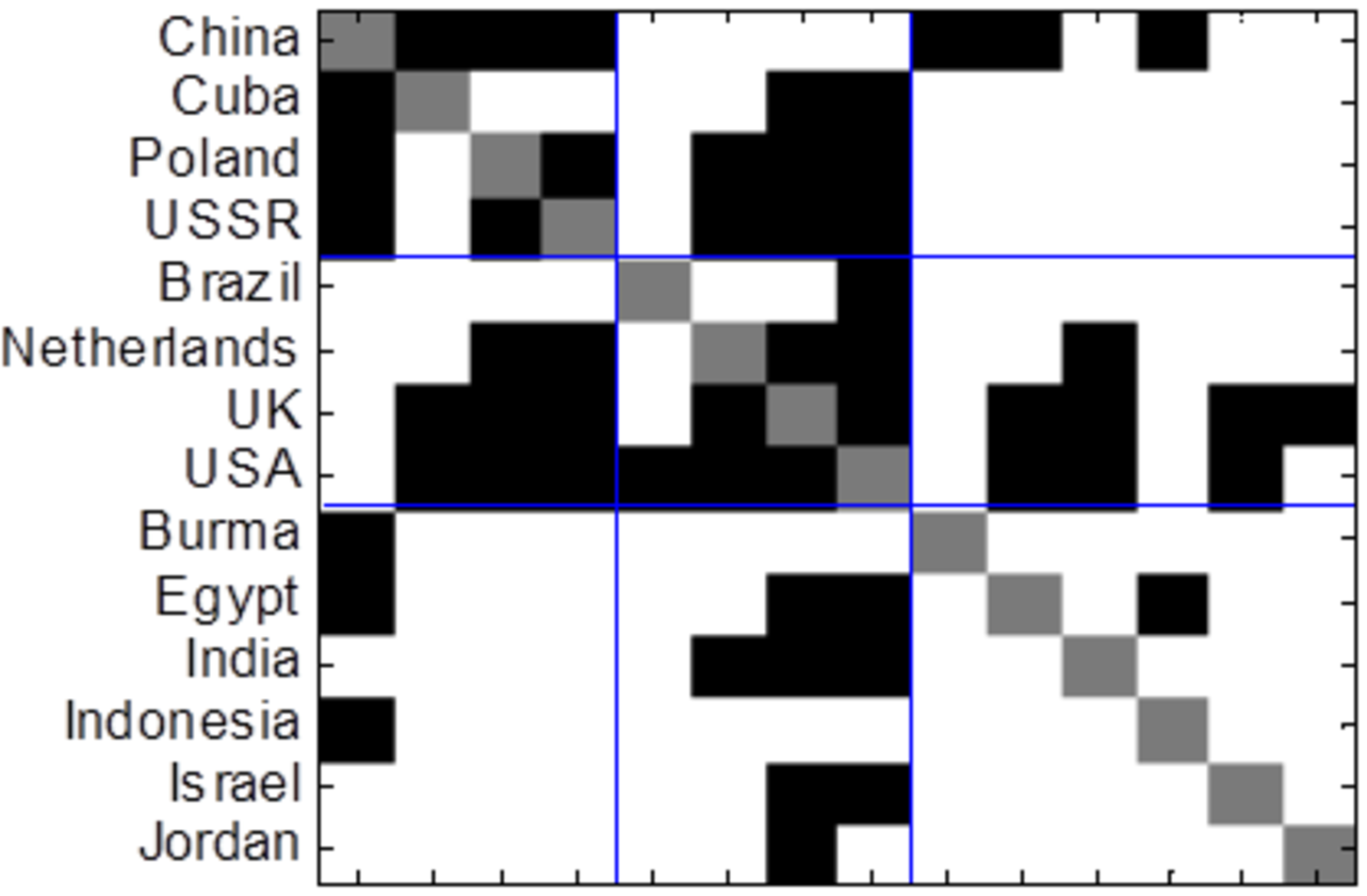}}\,\,
\subfigure[Bloc Membership]{\includegraphics[width=0.432\linewidth]{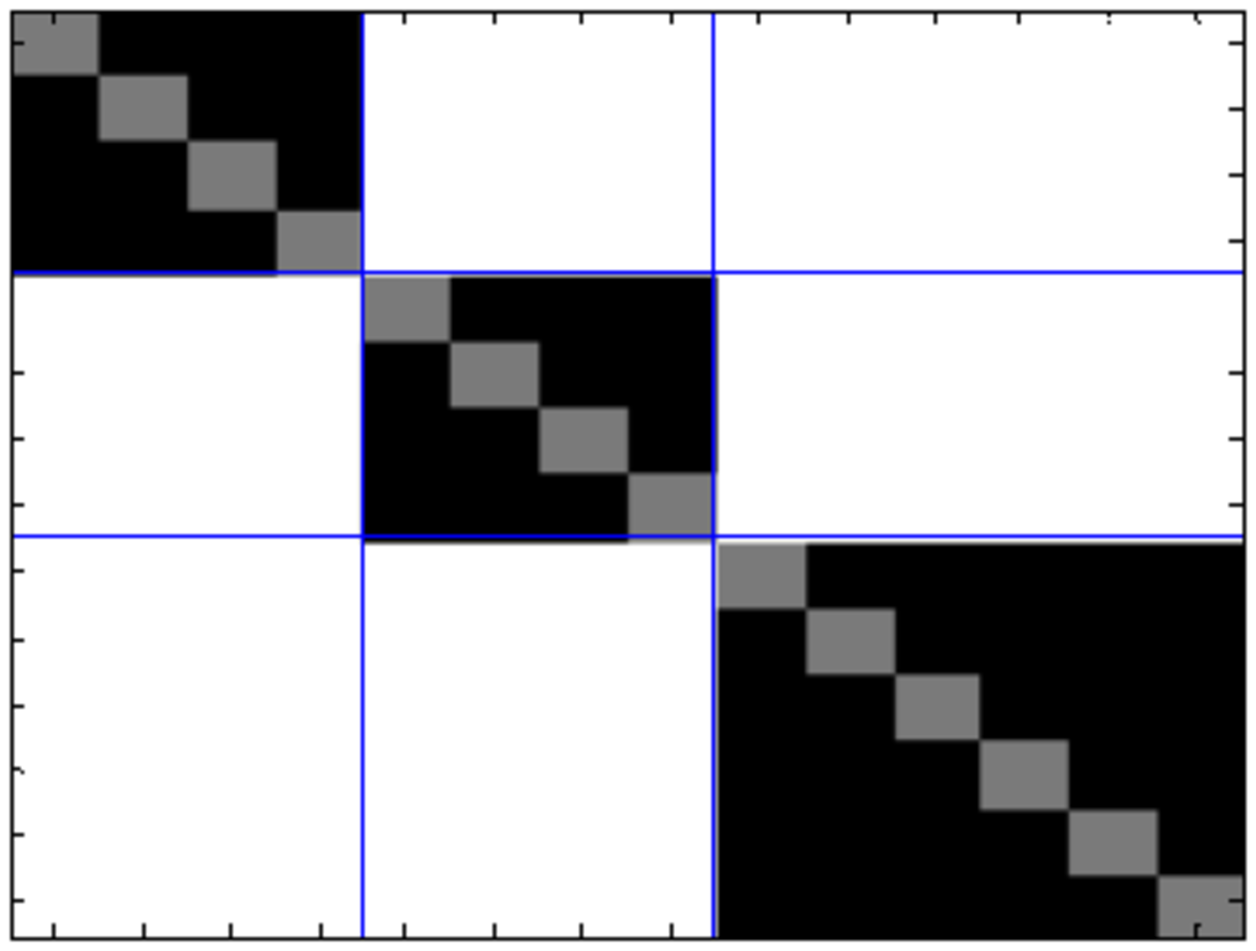}}

\vspace{-3mm}
\caption{The clustering results of four relations on the Nations date set. Black squares indicate an existing relation between the countries. Gray squares indicate missing data.}
\label{fig9}
\end{figure}

We compare our ROID method with several state-of-the-art approaches including WCP, the nonparametric Bayesian mehtod, IRM{\footnote{\url{http://www.psy.cmu.edu/~ckemp/code/irm.html}}}~\cite{kemp:irm}, the hidden variable discovery method, MRC~\cite{kok:spi} and RESCAL{\footnote{\url{http://www.mit.edu/~mnick/}}}~\cite{nickel:twm} and HaLRTC~\cite{liu:tcem} on these three data sets. Since WCP, WTucker and SHOOI yield very similar results on all three data sets, we only report the results of WCP. Then, we use the area under the precision-recall curve (AUC) as the evaluation metric to test the relation prediction performance as in~\cite{nickel:twm, jenatton:lfm}. In order to obtain comparable results to IRM, MRC and RESCAL, we follow their experimental requirements and preform 10-fold cross validation. For our ROID method, we set the multi-linear rank ${d}_{1}\!=\!{d}_{2}\!=\!{d}_{3}\!=\!35$ for both the Kinship and UMLS data sets, and ${d}_{1}\!=\!{d}_{2}\!=\!14$ and ${d}_{3}\!=\!10$ for the Nations data set, and the regularization parameter $\lambda=100$.

\begin{table}[t]
\footnotesize
\centering
\caption{Comparison of RSE results of WCP, RESCAL, HaLRTC, and our ROID method on three multi-relational data sets.}
\label{tab2}
\vspace{-2mm}

\begin{tabular} {l|cccc}
\hline
\ Datasets         & WCP   & RESCAL  & HaLRTC   & ROID\\
\hline
\ Nations         &0.3057  &0.3624   &0.2169   &\textbf{0.1773}\\
\ Kinship         &0.3236  &0.2404   &0.1683   &\textbf{0.1511}\\
\ UMLS	          &0.2004  &0.1861   &0.0907   &\textbf{0.0892}\\
\hline
\end{tabular}
\end{table}

We illustrate the experimental results of all these six methods on these three data sets, as shown in Fig.\ \ref{fig8}, from which we can see that our ROID method and HaLRTC consistently outperform the other four methods. The reason is that HaLRTC and our ROID method can more efficiently explore the impact of different relations to improve the accuracy of relation prediction. As WCP, RESCAL, HaLRTC, and our ROID method have similar AUC results, we also report the RSE results of these four methods in Table~\ref{tab2}. It is clear that our ROID method consistently performs better than the other three methods in terms of recovery accuracy. Moreover, we demonstrate the relation-based clustering capabilities of our ROID method on the Nations data set. We apply the \emph{K}-means clustering method with $K\!=\!3$ to the matrix $X_{m}$, and illustrate the results on four types of relationships in Fig.\ \ref{fig9}, from which we can see that similar results as in~\cite{nickel:twm, kemp:irm} are obtained.

\subsection{Running Time and Robustness Analysis}
Moreover, we present the comparison of the running time of related methods on three multi-relational data sets. In~\cite{nickel:twm}, it has been shown that WCP and RESCAL are much faster than MRC as well as IRM. Therefore, we only report the running time of WCP, RESCAL, HaLRTC and our ROID method on these three data sets with different ranks, as listed in Table~\ref{tab3}. It is clear that our ROID method is \emph{much faster} than WCP and RESCAL, and is more suitable for large-scale multi-relational data. RESCAL and WCP usually scale worse with regard to the ranks than our ROID method. In other words, with the increase of the given tensor ranks, the running time of RESCAL and WCP dramatically grows whereas that of our ROID method only changes slightly.

\begin{table}[t]
\scriptsize
\centering
\caption{The running time (seconds) comparison on three multi-relational data sets.}
\label{tab3}
\vspace{-2mm}

\begin{tabular}{cccccc}
\hline
{Data sets} & {Size} & {Methods} & Rank=10       & Rank=20     & Rank=40 \\
\hline
\multirow{4}{*}{Nations} &\multirow{4}{*}{14$\times$14$\times$56} &  WCP    &43.05 &186.79 &463.17\\
\multicolumn{2}{c}{\ } & RESCAL & 9.67 & 30.38 & 154.75\\
\multicolumn{2}{c}{\ } & ROID   & \textbf{1.23} & \textbf{2.46}  & \textbf{3.60}\\
\multicolumn{2}{c}{\ } & HaLRTC &\multicolumn{3}{c}{10.39}\\
\hline
\multirow{4}{*}{Kinship}  &\multirow{4}{*}{104$\times$104$\times$26} & WCP &156.70 &314.03 &488.82\\
\multicolumn{2}{c}{\ } & RESCAL & 18.93 & 36.22 & 116.85\\
\multicolumn{2}{c}{\ } & ROID   & \textbf{4.15} & \textbf{6.07}  & \textbf{10.24}\\
\multicolumn{2}{c}{\ } & HaLRTC &\multicolumn{3}{c}{46.98}\\
\hline
\multirow{4}{*}{UMLS}  &\multirow{4}{*}{135$\times$135$\times$49} & WCP &130.24 &262.36 &903.60\\
\multicolumn{2}{c}{\ } & RESCAL & 52.76 & 77.09 & 312.64\\
\multicolumn{2}{c}{\ } & ROID   & \textbf{11.18} & \textbf{13.40}  & \textbf{25.91}\\
\multicolumn{2}{c}{\ } & HaLRTC &\multicolumn{3}{c}{350.16}\\
\hline
\end{tabular}
\end{table}

We also evaluate the robustness of our ROID method against its parameters: the given tensor ranks and the regularization parameter $\lambda$ on these three real-world data sets, as shown in Fig.\ \ref{fig10}, from which we can see that our ROID method is \emph{robust} against its parameter variations, especially on the Kinship and UMLS date sets. Note that three rank parameters $d_{1}$, $d_{2}$ and $d_{3}$ are set to the lesser of the given multi-linear rank and the corresponding size of the tensor. The regularization parameter $\lambda$ is tuned from the grid $\{10^{0}, 10^{1}, \ldots, 10^{6}\}$.

\begin{figure}[t]
\centering
{\includegraphics[width=0.492\linewidth]{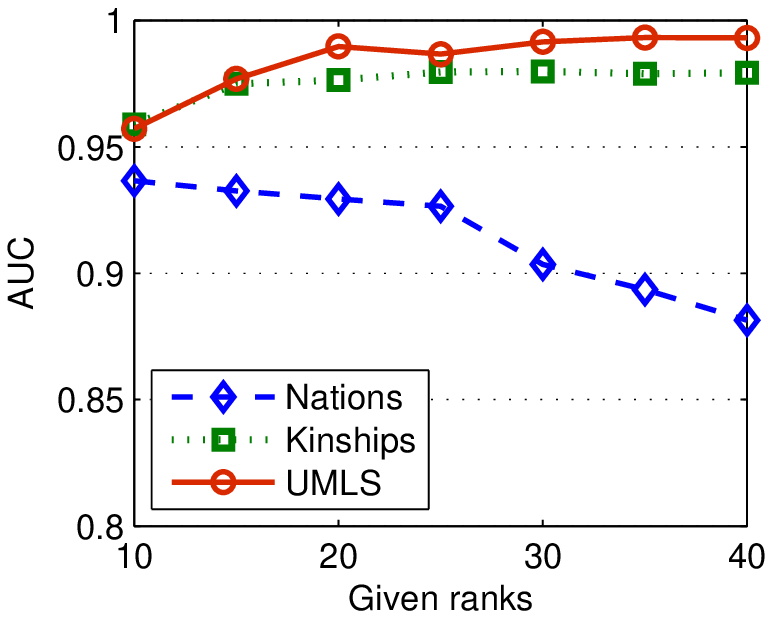}\label{fig101}}\,
{\includegraphics[width=0.492\linewidth]{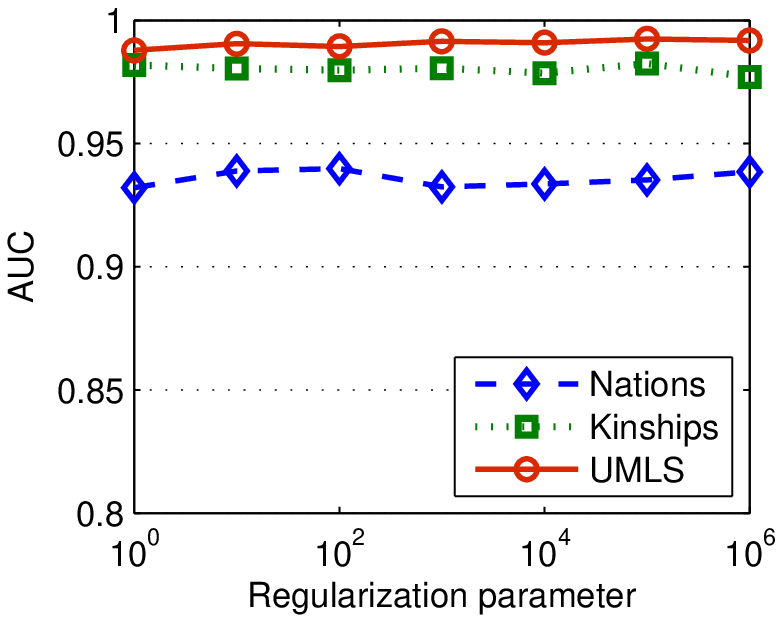}\label{fig102}}

\vspace{-2mm}
\caption{Link predication results of our ROID method against its parameters on the Nations, Kinship and UMLS date sets.}
\label{fig10}
\end{figure}

\section{Conclusions and Future Work}
In this paper we proposed a scalable ROID method and its graph regularized version for full or incomplete tensor analytics, such as multi-relational learning. First, we induced the equivalence relation of the Schatten $p$-norm ($0\!<\!p\!<\!\infty$) of a low multi-linear rank tensor and its core tensor. Then we presented a novel orthogonal tensor decomposition model with core tensor trace norm regularization. We also introduced a regularization version using graph Laplacians induced from the relationships and a sparse higher-order orthogonal iteration version. Finally, we developed two efficient ADMM algorithms to solve our problems. Moreover, we analyzed theoretically the local convergence of our algorithms. The convincing experimental results for real-world problems verified both the efficiency and effectiveness of our methods, especially from only a few observations.

Moreover, our ROID method can be extended to various higher-order tensor recovery and completion problems, such as higher-order robust principal component analysis (RPCA)~\cite{goldfarb:tr} and robust tensor completion. For future work, we are interested in exploring ways to regularize our model with other auxiliary information, such as semantic information contained in social network~\cite{nakatsuji:sdr}. Due to the unitary invariant property of norms as stated in Theorem 1 and the superiority of the Schatten-$p$ quasi-norm ($0\!<\!p\!<\!1$) to the trace norm, it would be an interesting research direction in future to investigate the more general core tensor Schatten quasi-norm regularization.

\appendices
\section{Proof of Theorem~\ref{theo1}:}
Before giving the proof of Theorem~\ref{theo1}, we will first present some properties of matrices and tensors in the following.

\begin{definition}
Let ${A}$ and ${B}$ be two matrices of size $m\times{n}$ and $p\times{q}$, respectively. The Kronecker product of both matrices $A$ and $B$ is an $mp\times{nq}$ matrix given by:
\vspace{-1mm}
\begin{displaymath}
{A}\otimes{B}=[{a}_{ij}{B}]_{mp\times{nq}}.
\end{displaymath}
\end{definition}

\begin{property}\label{prop1}
Let $A\in \mathbb{R}^{m\times p}$, $C\in \mathbb{R}^{p\times q}$, and $B\in \mathbb{R}^{n\times q}$, then
\vspace{-1mm}
\begin{displaymath}
\|ACB^{T}\|_{\mathcal{S}_{p}}=\|C\|_{\mathcal{S}_{p}}
\end{displaymath}
where both $A$ and $B$ are column-orthonormal, i.e., $A^{T}A\!=\!I_{p}$ and $B^{T}B\!=\!I_{q}$.
\end{property}

\begin{proof}
Let us denote the SVD of $C$ by $C\!=\!{U}{\Sigma}{V}^{T}$, then $ACB^{T}\!=\!(A{U}){\Sigma}(B{V})^{T}$. Since $(A{U})^{T}(A{U})\!=\!I_{p}$ and $(B{V})^{T}(B{V})\!=\!I_{q}$, $(A{U}){\Sigma}(B{V})^{T}$ is actually an SVD of $ACB^{T}$. According to the definition of the Schatten $p$-norm, we have $\|C\|_{\mathcal{S}_{p}}\!=\!\left(\textrm{Tr}({\Sigma^{p}})\right)^{1/p}\!=\!\|ACB^{T}\|_{\mathcal{S}_{p}}$.
\end{proof}

\begin{property}\label{prop2}
Let $A\!\in\! \mathbb{R}^{m\times n}$, $B\!\in\!\mathbb{R}^{p\times q}$, and ${C}$ and ${D}$ are two matrices of compatible sizes, then we have the following results:
\vspace{-1mm}
\begin{enumerate}
\item $(A\otimes B)\otimes C=A\otimes (B\otimes C)$.
\item $(A\otimes B)(C\otimes D)=(AC)\otimes (BD)$.
\item $(A\otimes B)^{T}=A^{T}\otimes B^{T}$.
\end{enumerate}
\end{property}

\begin{property}\label{prop3}
Let $\mathcal{X}=\mathcal{G}\times_{1}U\times_{2}V\times_{3}W$, where $\mathcal{X}\in \mathbb{R}^{{I_{1}}\times{I_{2}}\times{I_{3}}}$ and $\mathcal{G}\in \mathbb{R}^{{d_{1}}\times{d_{2}}\times{d_{3}}}$, then
\vspace{-1mm}
\begin{equation*}
\begin{split}
\mathcal{X}_{(1)}=U\mathcal{G}_{(1)}(&W\otimes V)^{T},\;\;\mathcal{X}_{(2)}=V\mathcal{G}_{(2)}(W\otimes U)^{T},\\
&\mathcal{X}_{(3)}=W\mathcal{G}_{(3)}(V\otimes U)^{T}.
\end{split}
\end{equation*}
\end{property}

\begin{proof}
Let ${P}_{1}=W\otimes V,\,{P}_{2}=W\otimes U,\,{P}_{3}=V\otimes U$.
According to Property~\ref{prop2}, we have
\vspace{-1mm}
\begin{displaymath}
\begin{split}
P^{T}_{1}P_{1}=&(W\otimes V)^{T}(W\otimes V)=(W^{T}\otimes V^{T})(W\otimes V),\\
=&(W^{T}W)\otimes(V^{T}V)=I_{3}\otimes I_{2}=\widetilde{I}_{1}
\end{split}
\end{displaymath}
where $I_{n}\!\in\! \mathbb{R}^{d_{n}\times d_{n}},\,n=1,\,2,\,3$, are all identity matrices, $\widetilde{I}_{1}\!\in\! \mathbb{R}^{J_{1}\times J_{1}}$ is also an identity matrix, and $J_{1}=\Pi_{j\neq 1}d_{j}$.

Similarly, we also have $P^{T}_{2}P_{2}\!=\!\widetilde{I}_{2}$ and $P^{T}_{3}P_{3}\!=\!\widetilde{I}_{3}$, where both $\widetilde{I}_{2}$ and $\widetilde{I}_{3}$ are identity matrices.

By Property~\ref{prop3}, we have
\vspace{-1mm}
\begin{equation*}
\|\mathcal{X}_{(1)}\|_{\mathcal{S}_{p}}=\|U\mathcal{G}_{(1)}(W\otimes V)^{T}\|_{\mathcal{S}_{p}}.
\end{equation*}
According to Property~\ref{prop1} and $P^{T}_{1}P_{1}=\widetilde{I}_{1}$ , we have
\vspace{-1mm}
\begin{equation*}
\|\mathcal{X}_{(1)}\|_{\mathcal{S}_{p}}=\|U\mathcal{G}_{(1)}(W\otimes V)^{T}\|_{\mathcal{S}_{p}}=\|\mathcal{G}_{(1)}\|_{\mathcal{S}_{p}}.
\end{equation*}
Similarly, we have $\|\mathcal{X}_{(2)}\|_{\mathcal{S}_{p}}\!=\!\|\mathcal{G}_{(2)}\|_{\mathcal{S}_{p}}$ and $\|\mathcal{X}_{(3)}\|_{\mathcal{S}_{p}}\!=\!\|\mathcal{G}_{(3)}\|_{\mathcal{S}_{p}}$. Hence, we have $\|\mathcal{X}\|_{\mathcal{S}_{p}}\!=\!\|\mathcal{G}\|_{\mathcal{S}_{p}}$.
\end{proof}

\section{Proof of Theorem~\ref{theo2}:}
\begin{proof}
The optimization problem (\ref{equ14}) with respect to $\mathcal{G}$ is written by
\vspace{-1mm}
\begin{equation}\label{equ36}
\begin{split}
\min_{\mathcal{G}}\,h(\mathcal{G})=\sum^{3}_{n=1}\frac{\rho^{k}}{2}\left\|\mathcal{G}_{(n)}-{G}^{k+1}_{n}+{Y}^{k}_{n}/\rho^{k}\right\|^{2}_{F}\\
+\frac{1}{2}\left\|\mathcal{X}^{k}-\mathcal{G}\!\times_{1}\!U\!\times_{2}\!V\!\times_{3}\!W\right\|^{2}_{F}.
\end{split}
\end{equation}
The above problem (\ref{equ36}) is a smooth convex optimization problem, thus we can obtain the derivative of the function $h$ in the following form:
\vspace{-1mm}
\begin{equation*}
\begin{split}
\frac{\partial{h}}{\partial\mathcal{G}}&=\left(\mathcal{G}-\mathcal{X}^{k}\!\times_{1}\!(U)^{T}\!\times_{2}\!(V)^{T}\!\times_{3}\!(W)^{T}\right)\\
&\;\;\;\;+\sum^{3}_{n=1}\rho^{k}\left(\mathcal{G}-\textup{refold}({G}^{k+1}_{n}-{Y}^{k}_{n}/\rho^{k})\right)\\
&=(3\rho^{k}+1)\mathcal{G}-\rho^{k}\sum^{3}_{n=1}\textup{refold}({G}^{k+1}_{n}-{Y}^{k}_{n}/\rho^{k})\\
&\;\;\;\;-\mathcal{X}^{k}\!\times_{1}\!(U)^{T}\!\times_{2}\!(V)^{T}\!\times_{3}\!(W)^{T}.
\end{split}
\end{equation*}

Let $\frac{\partial{h}}{\partial\mathcal{G}}$=0, the optimal solution of (\ref{equ36}) is given by
\vspace{-1mm}
\begin{equation*}
\begin{split}
\mathcal{G}=&\frac{1}{1+3\rho^{k}}\mathcal{X}^{k}\!\times_{1}\!(U)^{T}\!\times_{2}\!(V)^{T}\!\times_{3}\!(W)^{T}\\
&+\frac{\rho^{k}}{1+3\rho^{k}}\sum^{3}_{n=1}\textup{refold}({G}^{k+1}_{n}-{Y}^{k}_{n}/\rho^{k}).
\end{split}
\end{equation*}
\end{proof}

\section{Proof of Theorem~\ref{theo3}:}
\begin{proof}
Let
\vspace{-2mm}
\begin{equation}\label{equ37}
\begin{split}
f(\mathcal{G},U,V,W)=\frac{\rho^{k}}{2}\sum^{3}_{n=1}\left\|\mathcal{G}_{(n)}-{G}^{k+1}_{n}+{Y}^{k}_{n}/\rho^{k}\right\|^{2}_{F}\\
+\frac{1}{2}\left\|\mathcal{X}^{k}-\mathcal{G}\!\times_{1}\!U\!\times_{2}\!V\!\times_{3}\!W\right\|^{2}_{F},
\end{split}
\end{equation}
\begin{equation*}
\begin{split}
\;\;\mathcal{A}=\mathcal{X}^{k}\!\times_{1}\!(U)^{T}\!\times_{2}\!(V)^{T}\!\times_{3}\!(W)^{T},\\
\textup{and}\;\,\mathcal{B}=\sum^{3}_{n=1}\textup{refold}({G}^{k+1}_{n}-{Y}^{k}_{n}/\rho^{k}).\;\;\;\,
\end{split}
\end{equation*}
Then the closed-form solution of (\ref{equ37}) with respect to $\mathcal{G}$ can be obtained by (\ref{equ15}), and it can be rewritten as
\vspace{-2mm}
\begin{equation}\label{equ38}
\mathcal{G}=\frac{1}{1+3\rho^{k}}\mathcal{A}+\frac{\rho^{k}}{1+3\rho^{k}}\mathcal{B}.
\end{equation}

Using (\ref{equ38}) and according to the definitions of the tensors $\mathcal{A}$ and $\mathcal{B}$, we have
\begin{equation}\label{equ39}
\begin{split}
&\langle\mathcal{X}^{k},\mathcal{G}\!\times_{1}\!U\!\!\times_{2}\!V\!\!\times_{3}\!W\rangle\!=\!\left\langle\mathcal{A},\;\frac{1}{1\!+\!3\rho^{k}}(\mathcal{A}\!+\!\rho^{k}\mathcal{B})\right\rangle\\
=&\frac{1}{{1+{3}\rho^{k}}}{\|\mathcal{A}\|^{2}_{F}}+\frac{\rho^{k}}{{1+{3}\rho^{k}}}{\langle\mathcal{A},\;\mathcal{B}\rangle},
\end{split}
\end{equation}
\begin{equation}\label{equ40}
\begin{split}
&\langle\mathcal{G},\;\mathcal{B}\rangle=\left\langle\mathcal{B},\;\frac{1}{1+3\rho^{k}}(\mathcal{A}+\rho^{k}\mathcal{B})\right\rangle\\
=&\frac{\rho^{k}}{1+3\rho^{k}}\|\mathcal{B}\|^{2}_F+\frac{1}{1+3\rho^{k}}\langle\mathcal{A},\;\mathcal{B}\rangle.
\end{split}
\end{equation}
Hence,
\vspace{-2mm}
\begin{equation}\label{equ41}
\begin{aligned}[b]
&f(\mathcal{G},U,V,W)=\frac{\rho^{k}}{2}\!\sum^{3}_{n=1}\!\|{G}^{k+\!1}_{n}\!-\!{Y}^{k}_{n}/\rho^{k}\|^{2}_{F}\!+\!\frac{1}{2}\|\mathcal{X}^{k}\|^{2}_{F}\\
&-\langle\mathcal{X}^{k},\,\mathcal{G}\!\times_{1}\!\!U\!\!\times_{2}\!\!V\!\!\times_{3}\!\!W\rangle\!+\!\frac{3\rho^{k}\!+\!1}{2}\!\|\mathcal{G}\|^{2}_{F}\!-\!\rho^{k}\langle \mathcal{G},\,\mathcal{B}\rangle.
\end{aligned}
\end{equation}
Substituting (\ref{equ38}), (\ref{equ39}) and (\ref{equ40}) into (\ref{equ41}), then the cost function (\ref{equ41}) is rewritten in the following form,
\vspace{-2mm}
\begin{equation*}
\begin{aligned}[b]
f(\mathcal{G},U,V,W)\!=&\frac{\rho^{k}}{2}\!\sum^{3}_{n=1}\|{G}^{k+1}_{n}\!-\!{Y}^{k}_{n}/\rho^{k}\|^{2}_{F}\!+\!\frac{1}{2}\!\|\mathcal{X}^{k}\|^{2}_{F}\\
&\,-\langle\mathcal{A}+\rho^{k}\mathcal{B},\;\mathcal{G}\rangle+\frac{1+3\rho^{k}}{2}\|\mathcal{G}\|^{2}_{F}\\
=&c_{1}-\frac{1}{2(1+3\rho^{k})}\|\mathcal{A}+\rho^{k}\mathcal{B}\|^{2}_{F}
\end{aligned}
\end{equation*}
where $c_{1}\!=\!\|\mathcal{X}^{k}\|^{2}_{F}/2\!+\!(\rho^{k}/2)\!\sum^{3}_{n=1}\!\|{G}^{k+1}_{n}\!-\!{Y}^{k}_{n}/\rho^{k}\|^{2}_{F}$ is a constant, and $g(U,V,W):=\|\mathcal{A}+\rho^{k}\mathcal{B}\|^{2}_{F}$.
\end{proof}

\section{Proof of Theorem~\ref{theo4}:}
\begin{proof}
By (\ref{equ12}), (\ref{equ14}) and (\ref{equ21}), and $Y^{k+1}_{n}=Y^{k}_{n}+\rho^{k}(\mathcal{G}^{k+1}_{(n)}-G^{k+1}_{n})$, we have
\begin{equation}\label{equ42}
\begin{aligned}[b]
&\;0\in \partial\|G^{k+\!1}_{n}\|_{*}/(3\lambda)-Y^{k+\!1}_{n}+\rho^{k}(\mathcal{G}^{k+\!1}_{(n)}-\mathcal{G}^{k}_{(n)}),\\
&\!\!\!\sum^{3}_{n=1}\!\textup{refold}(\!Y^{k\!+\!1}_{n}\!)\!+\!\mathcal{G}^{k\!+\!1}\!\!-\!\mathcal{X}^{k\!+\!1}\!\!\times_{\!1}\!(\!U^{k\!+\!1}\!)^{\!T}\!\!\!\times_{\!2}\!(\!V^{k\!+\!1}\!)^{\!T}\!\!\!\times_{\!3}\!(\!W^{k\!+\!1}\!)^{\!T}\\
&\;\;+(\!\mathcal{X}^{k\!+\!1}\!-\!\mathcal{X}^{k})\!\times_{1}\!(\!U^{k\!+\!1}\!)^{T}\!\!\times_{2}\!(\!V^{k\!+\!1}\!)^{T}\!\!\times_{3}\!(\!W^{k\!+\!1}\!)^{T}\!\!=\!0,\\
&\mathcal{X}^{k+\!1}_{\Omega}\!=\!\mathcal{T}_{\Omega},\; \mathcal{X}^{k+\!1}_{\Omega^{C}}\!=\!(\mathcal{G}^{k+\!1}\!\times_{1}\!U^{k+\!1}\!\times_{2}\!V^{k+\!1}\!\times_{3}\!W^{k+\!1})_{\Omega^{C}},\\
&(\!U^{k\!+\!1}\!)^{T}U^{k\!+\!1}\!\!=\!I_{d_{1}},(\!V^{k\!+\!1}\!)^{T}V^{k\!+\!1}\!\!=\!I_{d_{2}},(\!W^{k\!+\!1}\!)^{T}W^{k\!+\!1}\!\!=\!I_{d_{3}}
\end{aligned}
\end{equation}
where $\Omega^{C}$ is the complement of $\Omega$. By Lemma~\ref{lemm1}, we have that the sequence $\{\mathscr{Z}^{k}\}$ is bounded, and $\mathscr{Z}^{k+1}-\mathscr{Z}^{k}\rightarrow 0$, such as $\mathcal{X}^{k+1}-\mathcal{X}^{k}\rightarrow 0$, $\mathcal{G}^{k+1}-\mathcal{G}^{k}\rightarrow 0$, $G^{k+1}_{n}-G^{k}_{n}\rightarrow 0$ and $Y^{k+1}_{n}-Y^{k}_{n}\rightarrow 0,\,n=1,\,2,\,3$. By the Bolzano-Weierstrass theorem, the bounded sequence $\{\mathscr{Z}^{k}\}$ must have a convergent subsequence $\{\mathscr{Z}^{k_{j}}\}$, and the limit point is denoted by $\mathscr{Z}^{\infty}=\lim_{j=\infty}\mathscr{Z}^{k_{j}}$. Moreover, $\{\mathscr{Z}^{k}\}$ lies in a compact set, thus $\mathscr{Z}^{\infty}$ is an accumulation point of $\{\mathscr{Z}^{k}\}$, where $\mathscr{Z}^{\infty}=(\{G^{\infty}_{n}\},\mathcal{G}^{\infty},U^{\infty},V^{\infty},W^{\infty},\mathcal{X}^{\infty},\{Y^{\infty}_{n}\})$.

Moreover, according to $Y^{k+1}_{n}-Y^{k}_{n}\rightarrow 0$ and $Y^{k+1}_{n}=Y^{k}_{n}+\rho^{k}(\mathcal{G}^{k+1}_{(n)}-G^{k+1}_{n})$, we have that $\mathcal{G}^{k+1}_{(n)}-G^{k+1}_{n}\rightarrow 0,\,n=1,\,2,\,3$. By (\ref{equ42}), we have
\begin{equation}\label{equ43}
\begin{aligned}[b]
& 0\in \partial\|G^{\infty}_{n}\|_{*}/(3\lambda)-Y^{\infty}_{n},\,\;\mathcal{G}^{\infty}_{(n)}=G^{\infty}_{n},\;n=1,2,3,\\
&\sum^{3}_{n=1}\!\textup{refold}(\!Y^{\!\infty}_{n}\!)\!+\!\mathcal{G}^{\infty}\!\!-\!\mathcal{X}^{\infty}\!\!\times_{\!1}\!\!(\!U^{\infty}\!)^{\!T}\!\!\!\times_{\!2}\!\!(\!V^{\infty}\!)^{\!T}\!\!\!\times_{\!3}\!\!(\!W^{\infty}\!)^{\!T}\!\!=\!0,\\
&\mathcal{X}^{\infty}_{\Omega}=\mathcal{T}_{\Omega},\;\,\mathcal{X}^{\infty}_{\Omega^{C}}=(\mathcal{G}^{\infty}\!\times_{1}\!U^{\infty}\!\times_{2}\!V^{\infty}\!\times_{3}\!W^{\infty})_{\Omega^{C}},\\
&(U^{\infty})^{T}U^{\infty}\!=\!I_{d_{1}},(V^{\infty})^{T}V^{\infty}\!=\!I_{d_{2}},(W^{\infty})^{T}W^{\infty}\!=\!I_{d_{3}}.
\end{aligned}
\end{equation}
It is easy to see that (\ref{equ43}) is the KKT conditions for (\ref{equ10}), that is, the accumulation point $(\{G^{\infty}_{n}\},\mathcal{G}^{\infty},U^{\infty},V^{\infty},W^{\infty},\mathcal{X}^{\infty})$ satisfies the KKT conditions of (\ref{equ10}). This completes the proof of Theorem~\ref{theo4}.
\end{proof}

\bibliographystyle{IEEEtran}
\bibliography{IEEEabrv,sigproc}

\begin{IEEEbiography}[{\includegraphics[width=1in,height=1.25in,clip,keepaspectratio]{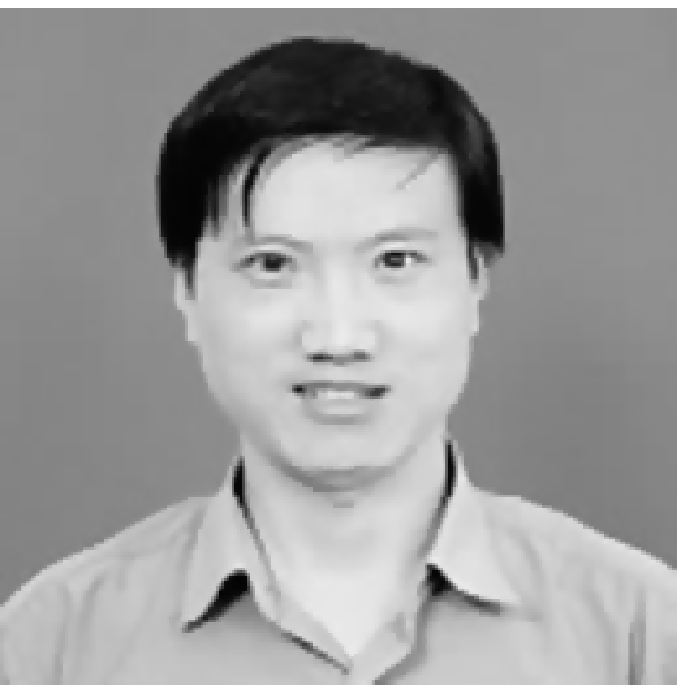}}]{Fanhua Shang} (M'14) received the Ph.D. degree in Circuits and Systems from Xidian University, Xi'an, China, in 2012.

He is currently a Post-Doctoral Research Fellow with the Department of Computer Science and Engineering, The Chinese University of Hong Kong. Prior to that, he was a Post-Doctoral Research Associate with the Department of Electrical and Computer Engineering, Duke University, Durham, NC, USA. His current research interests include machine learning, data mining, pattern recognition, and computer vision.
\end{IEEEbiography}

%

\begin{IEEEbiography}[{\includegraphics[width=1in,height=1.25in,clip,keepaspectratio]{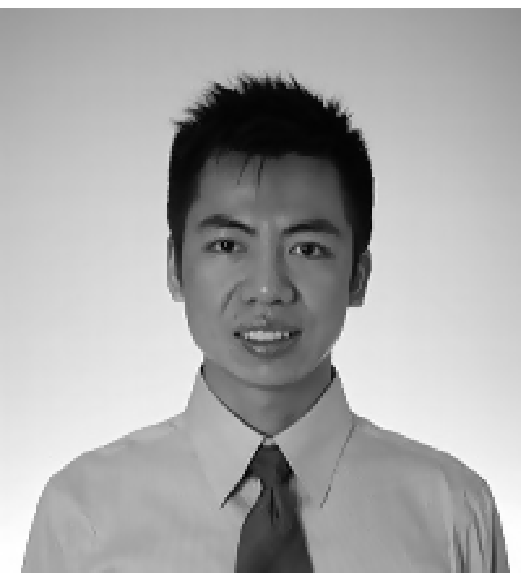}}]{James Cheng} is an Assistant Professor with the Department of Computer Science and Engineering, The Chinese University of Hong Kong, Hong Kong. His current research interests include distributed computing systems, large-scale network analysis, temporal networks, and big data.
\end{IEEEbiography}

\begin{IEEEbiography}[{\includegraphics[width=1in,height=1.25in,clip,keepaspectratio]{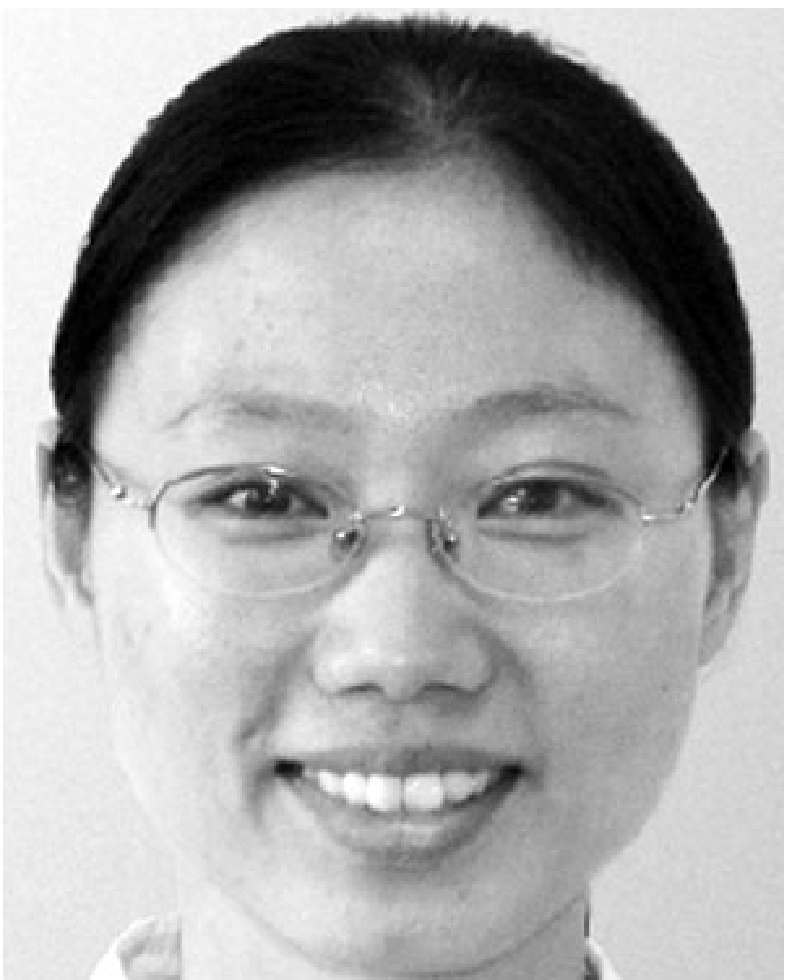}}]{Hong Cheng} is an Associate Professor with the Department of Systems Engineering and Engineering Management, The Chinese University of Hong Kong, Hong Kong. Her primary research interests include data mining, machine learning, and database systems.
\end{IEEEbiography}

\newpage
\onecolumn

\section*{Supplementary Materials}


In this supplementary material, we give the detailed ADMM algorithms for solving the core tensor trace norm regularized full tensor decomposition problem (7) and the sparse tensor HOOI problem (8).  In addition, we also provide the theoretical analysis of the relationship between (6) and (8), and the detailed complexity analysis of their algorithms.

\subsection*{Algorithm for Solving (7)}
In this supplementary material, we first give the details on how we can design an efficient ADMM algorithm, as outlined in Algorithm 3, for solving the core tensor trace norm regularized full tensor decomposition problem (7). Similar to (10), we also introduce three much smaller auxiliary variables ${G}_{n}\!\in\!\mathbb{R}^{{{d_{n}}\times{\Pi_{j\neq{n}}}{d_{j}}}}$, and reformulate (7) into the following equivalent form:
\begin{equation}
\begin{split}
&\min_{\mathcal{G},{U},{V},W,\{G_{\!n}\!\}}\sum^{3}_{n=1}\frac{\|G_{n}\|_{*}}{3\lambda}+\frac{1}{2}\|\mathcal{T}-\mathcal{G}\!\times_{1}\!U\!\times_{2}\!V\!\times_{3}\!W\|^{2}_{F},\\
&\quad\quad\; \textup{s.t.},\mathcal{G}_{(n)}=G_{n},U^{T}\!U\!=\!I_{d_{1}},V^{T}\!V\!=\!I_{d_{2}},W^{T}\!W\!=\!I_{d_{3}}.
\end{split}
\end{equation}

The partial augmented Lagrangian function of (44) is
\begin{equation}
\begin{split}
\mathcal{L}_{\rho}(\{{G}_{n}\},\mathcal{G},U,V,W,\{{Y}_{n}\})=&\sum^{3}_{n=1}\!\left(\frac{\|G_{n}\|_{*}}{3\lambda}+\langle{Y}_{n},\;{\mathcal{G}}_{(n)}-{G}_{n}\rangle+\frac{\rho}{2}\|\mathcal{G}_{(n)}-{G}_{n}\|^{2}_{F}\right)\\
&+\frac{1}{2}\|\mathcal{T}-\mathcal{G}\!\times_{1}\!U\!\times_{2}\!V\!\times_{3}\!W\|^{2}_{F},
\end{split}
\end{equation}
where ${Y}_{n}\!\in\!\mathbb{R}^{{{d_{n}}\times{\Pi_{j\neq{n}}}{d_{j}}}}$ are the matrices of Lagrange multipliers for $n\!=\!1,2,3$.\\

The updating rules for $\{{G}^{k+\!1}_{1}\!,{G}^{k+\!1}_{2}\!,{G}^{k+\!1}_{3}\}$ are the same as (13) in this paper. Next we discuss how we update $\{{U}^{k+1},{V}^{k+1},{W}^{k+1},\mathcal{G}^{k+1}\}$. The optimization problem (44) with respect to ${U}$, ${V}$, ${W}$ and $\mathcal{G}$ is formulated as follows:
\begin{equation}
\begin{split}
&\min_{\mathcal{G},U,V,W}\sum^{3}_{n=1}\frac{\rho^{k}}{2}\|\mathcal{G}_{(n)}\!-{G}^{k+1}_{n}+{Y}^{k}_{n}/\rho^{k}\|^{2}_{F}+\frac{1}{2}\|\mathcal{T}-\mathcal{G}\!\times_{1}\!U\!\!\times_{2}\!V\!\!\times_{3}\!W\|^{2}_{F},\\
&\quad \textup{s.t.},\,U^{T}U=I_{d_{1}},\,V^{T}V=I_{d_{2}},\,W^{T}W=I_{d_{3}}.
\end{split}
\end{equation}

Similar to Theorem 3 in this paper, the minimization problem with respect to ${U}$, ${V}$ and ${W}$ can be formulated as Theorem 6 shown below:
\begin{theorem}
Assume a real third-order tensor $\mathcal{T}$, then the minimization problem (46) is equivalent to the maximization, over these matrices $U$, $V$ and $W$ having orthonormal columns, of the function
\begin{equation}
g(U,\,V,\,W)=\|\mathcal{C}+\rho^{k}\mathcal{B}\|^{2}_{F},
\end{equation}
where $\mathcal{C}\!=\!\mathcal{T}\!\!\times_{1}\!U^{T}\!\!\times_{2}\!V^{T}\!\!\times_{3}\!W^{T}$.
\end{theorem}

\begin{IEEEproof}
According to Theorem 2, for any given matrices $U$, $V$ and $W$, the optimal core tensor $\mathcal{G}$ is given by
\begin{equation}
\mathcal{G}=\frac{\rho^{k}\mathcal{B}}{3\rho^{k}+1}+\frac{\mathcal{C}}{3\rho^{k}+1}.
\end{equation}

Thus, we have
\begin{equation*}
\begin{split}
f(U,V,W):=&\sum^{3}_{n=1}\frac{\rho^{k}}{2}\|\mathcal{G}_{(n)}\!-{G}^{k+1}_{n}+{Y}^{k}_{n}/\rho^{k}\|^{2}_{F}+\frac{1}{2}\|\mathcal{T}-\mathcal{G}\!\times_{1}\!U\!\!\times_{2}\!V\!\!\times_{3}\!W\|^{2}_{F}\\
^{a}\!\!=&\frac{3\rho^{k}}{2}\|\mathcal{G}\|^{2}_{F}+\frac{\rho^{k}}{2}\sum^{3}_{n=1}\|{G}^{k+1}_{n}-{Y}^{k}_{n}/\rho^{k}\|^{2}_{F}-\langle\mathcal{G}, \rho^{k}\mathcal{B}\rangle+\frac{1}{2}\|\mathcal{T}\|^{2}_{F}+\frac{1}{2}\|\mathcal{G}\|^{2}_{F}-\langle \mathcal{T}, \mathcal{G}\!\times_{1}\!U\!\!\times_{2}\!V\!\!\times_{3}\!W\rangle\\
^{b}\!\!=&\frac{\rho^{k}}{2}\sum^{3}_{n=1}\|{G}^{k+1}_{n}-{Y}^{k}_{n}/\rho^{k}\|^{2}_{F}+\frac{1}{2}\|\mathcal{T}\|^{2}_{F}+\frac{3\rho^{k}\!+\!1}{2}\|\frac{1}{3\rho^{k}\!+\!1}(\rho^{k}\mathcal{B}+\mathcal{C})\|^{2}_{F}-\langle\frac{1}{3\rho^{k}\!+\!1}(\rho^{k}\mathcal{B}+\mathcal{C}), \rho^{k}\mathcal{B}+\mathcal{C}\rangle\\
=&\frac{\rho^{k}}{2}\sum^{3}_{n=1}\|{G}^{k+1}_{n}-{Y}^{k}_{n}/\rho^{k}\|^{2}_{F}+\frac{1}{2}\|\mathcal{T}\|^{2}_{F}-\frac{1}{2(3\rho^{k}\!+\!1)}\|\rho^{k}\mathcal{B}+\mathcal{C}\|^{2}_{F},
\end{split}
\end{equation*}
where $\sum^{3}_{n=1}\!\frac{\rho^{k}}{2}\!\|{G}^{k+1}_{n}\!-\!{Y}^{k}_{n}/\rho^{k}\|^{2}_{F}\!+\!\frac{1}{2}\!\|\mathcal{T}\|^{2}_{F}$ is a constant, $g(U,V,W):=\|\rho^{k}\mathcal{B}+\mathcal{C}\|^{2}_{F}$, and the equality $^{a}\!\!=$ relies on the facts that $U$, $V$ and $W$ have orthonormal columns and the equality $^{b}\!\!=$ holds due to (48).
\end{IEEEproof}
~\\
By keeping $V^{k}$ and $W^{k}$ fixed, we have
\begin{equation}
\max_{U,U^{T}\!U=I_{d_{1}}}\!\!\|\mathcal{D}^{k}_{1}\!\times_{\!1}\!U^{T}\!+\!\rho^{k}\mathcal{B}\|^{2}_{F}\!=\!\|(\mathcal{D}^{k}_{1})^{T}_{(1)}U\!+\!\rho^{k}\mathcal{B}^{T}_{(1)}\|^{2}_{F},
\end{equation}
where $\mathcal{D}^{k}_{1}\!=\!\mathcal{T}\!\times_{2}\!(V^{k})^{T}\!\!\times_{3}\!(W^{k})^{T}$. This is actually the well-known orthogonal procrustes problem~\cite{nick:mpp}. Similar to (18) in this paper, we have
\begin{equation}
U^{k+1}=\textup{ORT}\left((\mathcal{D}^{k}_{1})_{(1)}\mathcal{B}^{T}_{(1)}\right).
\end{equation}
Repeating the above procedure for $V$ and $W$, we have
\begin{equation}
{V}^{k+1}=\textup{ORT}\left((\mathcal{D}^{k}_{2})_{(2)}\mathcal{B}^{T}_{(2)}\right),\;\;{W}^{k+1}=\textup{ORT}\left((\mathcal{D}^{k}_{3})_{(3)}\mathcal{B}^{T}_{(3)}\right),
\end{equation}
where $\mathcal{D}^{k}_{2}\!=\!\mathcal{T}\!\times_{1}\!(U^{k+1})^{T}\!\times_{3}\!(W^{k})^{T}$ and $\mathcal{D}^{k}_{3}\!=\!\mathcal{T}\!\times_{1}\!(U^{k+1})^{T}\!\times_{2}\!(V^{k+1})^{T}$.\\

After updating the matrices ${U}^{k+1}$, ${V}^{k+1}$ and ${W}^{k+1}$, then $\mathcal{G}$ is updated by
\begin{equation}
\mathcal{G}^{k+1}=\frac{\rho^{k}\sum^{3}_{n=1}\textup{refold}({G}^{k+1}_{n}-{Y}^{k}_{n}/\rho^{k})}{1+3\rho^{k}}+\frac{\mathcal{D}^{k}_{3}\!\times_{3}\!(W^{k+1})^{T}}{1+3\rho^{k}}.
\end{equation}

\begin{algorithm}[t]
\caption{ADMM for core tensor trace norm regularized full tensor decomposition problem (7)}
\label{alg3}
\renewcommand{\algorithmicrequire}{\textbf{Input:}}
\renewcommand{\algorithmicensure}{\textbf{Initialize:}}
\renewcommand{\algorithmicoutput}{\textbf{Output:}}
\begin{algorithmic}[1]
\REQUIRE $\mathcal{T}_{\Omega}$, multi-linear rank $({d_{1}},d_{2},{d_{3}})$ and $\textup{tol}$.
\WHILE {not converged}
\STATE {Update $G^{k+1}_{n}$ by (13) in the paper.}
\STATE {Update $U^{k+1}$, $V^{k+1}$, $W^{k+1}$ and $\mathcal{G}^{k+1}$ by (50), (51) and (52), respectively.}
\STATE {Update the multipliers $Y^{k+1}_{n}$ by $Y^{k+1}_{n}=Y^{k}_{n}+\rho^{k}(\mathcal{G}^{k+1}_{(n)}-G^{k+1}_{n}),\,n=1,2,3$.}
\STATE {Update $\rho^{k+1}$ by (27) in the paper.}
\STATE {Check the convergence condition, $\max\left(\|\mathcal{G}^{k+1}_{(n)}-G^{k+1}_{n}\|_{F}/\|\mathcal{T}\|_{F},\,n=1,2,3\right)<\textup{tol}$.}
\ENDWHILE
\OUTPUT $\mathcal{G}^{k+1}$, $U^{k+1}$, $V^{k+1}$ and $W^{k+1}$.
\end{algorithmic}
\end{algorithm}

~\\
\subsection*{Relationship between (6) and (8)}
~\\

When $\lambda\rightarrow\infty$, the model (6) degenerates to the following incomplete tensor Tucker decomposition model
\begin{equation}
\begin{split}
&\min_{\mathcal{G},U,V,W,\mathcal{X}}\frac{1}{2}\|\mathcal{X}-\mathcal{G}\!\times_{1}\!U\!\times_{2}\!V\!\times_{3}\!W\|^{2}_{F},\\
&\quad\;\textup{s.t.},\mathcal{X}_{\Omega}\!=\!\mathcal{T}_{\Omega},U^{T}U\!=\!I_{d_{1}},V^{T}V\!=\!I_{d_{2}},W^{T}W\!=\!I_{d_{3}}.
\end{split}
\end{equation}
~\\
Assume that $(\mathcal{G}^{*},U^{*},V^{*},W^{*},\mathcal{X}^{*})$ is a critical point (or stationary point) of (53). Then the Karush-Kuhn-Tucker (KKT) optimality conditions for (53) are given by
\begin{equation*}
\begin{split}
&\mathcal{X}^{*}_{\Omega}=\mathcal{T}_{\Omega},\;\,\mathcal{X}^{*}_{\Omega^{C}}=(\mathcal{G}^{*}\!\times_{1}\!U^{*}\!\times_{2}\!V^{*}\!\times_{3}\!W^{*})_{\Omega^{C}},\\
&(U^{*})^{T}U^{*}=I_{d_{1}},(V^{*})^{T}V^{*}=I_{d_{2}},(W^{*})^{T}W^{*}=I_{d_{3}}.
\end{split}
\end{equation*}
Thus, we have
\begin{equation*}
\frac{1}{2}\|\mathcal{X}^{*}-\mathcal{G}^{*}\!\times_{1}\!U^{*}\!\times_{2}\!V^{*}\!\times_{3}\!W^{*}\|^{2}_{F}
=\frac{1}{2}\|\mathcal{W}\ast(\mathcal{T}-\mathcal{G}^{*}\!\times_{1}\!U^{*}\!\times_{2}\!V^{*}\!\times_{3}\!W^{*})\|^{2}_{F}.
\end{equation*}

Let $\mathcal{Z}^{*}=\mathcal{G}^{*}\!\times_{1}\!U^{*}\!\times_{2}\!V^{*}\!\times_{3}\!W^{*}$, then
\begin{equation*}
\frac{1}{2}\|\mathcal{X}^{*}-\mathcal{G}^{*}\!\times_{1}\!U^{*}\!\times_{2}\!V^{*}\!\times_{3}\!W^{*}\|^{2}_{F}
=\frac{1}{2}\|\mathcal{W}\ast(\mathcal{Z}^{*}-\mathcal{T})\|^{2}_{F}.
\end{equation*}
It is easy to verify that $(\mathcal{G}^{*},U^{*},V^{*},W^{*},\mathcal{Z}^{*})$ satisfies the KKT optimality conditions of the following sparse tensor (in many practical applications, the incomplete tensors are very sparse) HOOI problem (i.e., (8) in this paper):
\begin{equation}
\begin{split}
&\min_{\mathcal{G},U,V,W,\mathcal{Z}}\;\frac{1}{2}\|\mathcal{W}\ast(\mathcal{Z}-\mathcal{T})\|^{2}_{F},\\
&\textup{s.t.},\mathcal{Z}=\mathcal{G}\!\times_{1}\!U\!\!\times_{2}\!V\!\!\times_{3}\!W,U^{T}U\!=I_{d_{1}},V^{T}V\!=I_{d_{2}},W^{T}W\!=I_{d_{3}}.
\end{split}
\end{equation}
That is, $(\mathcal{G}^{*},U^{*},V^{*},W^{*},\mathcal{Z}^{*})$ is a critical point of (54).\\

On the other hand, suppose that $(\mathcal{G}^{*},U^{*},V^{*},W^{*},\mathcal{Z}^{*})$ is a critical point of (54), and let $\mathcal{X}^{*}_{\Omega}=\mathcal{T}_{\Omega}$ and $\mathcal{X}^{*}_{\Omega^{C}}=(\mathcal{G}^{*}\!\times_{1}\!U^{*}\!\times_{2}\!V^{*}\!\times_{3}\!W^{*})_{\Omega^{C}}$, then we can know that $(\mathcal{G}^{*},U^{*},V^{*},W^{*},\mathcal{X}^{*})$ is also a critical point of (53). In this sense, the problem (54) (i.e., (8) in this paper) can be seen as a special case of (6) in this paper when $\lambda\rightarrow\infty$.\\

\subsection*{Algorithm for Solving (8)}
We also give the details on how we can design an efficient ADMM algorithm, as outlined in Algorithm 4, to solve the sparse tensor HOOI problem (8), whose partial augmented Lagrangian function is given as follows:
\begin{equation}
\mathcal{L}_{\rho}\left(\mathcal{G},U,V,W,\mathcal{Z},\mathcal{Y}\right)=\frac{1}{2}\|\mathcal{W}\ast(\mathcal{Z}-\mathcal{T})\|^{2}_{F}+\langle\mathcal{Y},\;\mathcal{Z}-\mathcal{G}\!\times_{1}\!U\!\times_{2}\!V\!\times_{3}\!W\rangle
+\frac{\rho}{2}\|\mathcal{Z}-\mathcal{G}\!\times_{1}\!U\!\times_{2}\!V\!\times_{3}\!W\|^{2}_{F},
\end{equation}
where $\mathcal{Y}$ is the tensor of Lagrange multipliers.\\

\begin{algorithm}[t]
\caption{ADMM for sparse tensor HOOI problem (8)}
\label{alg4}
\renewcommand{\algorithmicrequire}{\textbf{Input:}}
\renewcommand{\algorithmicensure}{\textbf{Initialize:}}
\renewcommand{\algorithmicoutput}{\textbf{Output:}}
\begin{algorithmic}[1]
\REQUIRE $\mathcal{T}_{\Omega}$, multi-linear rank $({d_{1}},d_{2},{d_{3}})$ and $\textup{tol}$.
\WHILE {not converged}
\STATE {Update $U^{k+1}$, $V^{k+1}$ and $W^{k+1}$ by (59) and (60), respectively.}
\STATE {Update $\mathcal{G}^{k+1}$ and $\mathcal{X}^{k+1}$ by (61) and (63), respectively.}
\STATE {Update the multiplier $\mathcal{Y}^{k+1}$ by $\mathcal{Y}^{k+1}=\mathcal{Y}^{k}+\rho^{k}(\mathcal{Z}^{k+1}-\mathcal{G}^{k+\!1}\!\!\times_{1}\!U^{k+\!1}\!\!\times_{2}\!V^{k+\!1}\!\!\times_{3}\!W^{k+\!1})$.}
\STATE {Update $\rho^{k+1}$ by (27) in this paper.}
\STATE {Check the convergence condition, $\|\mathcal{Z}^{k+1}-\mathcal{G}^{k+\!1}\!\!\times_{1}\!U^{k+\!1}\!\!\times_{2}\!V^{k+\!1}\!\!\times_{3}\!W^{k+\!1}\|_{F}/\|\mathcal{T}\|_{F}<\textup{tol}$.}
\ENDWHILE
\OUTPUT $\mathcal{G}^{k+1}$, $U^{k+1}$, $V^{k+1}$ and $W^{k+1}$.
\end{algorithmic}
\end{algorithm}

To update $\{{U}^{k+1},{V}^{k+1},{W}^{k+1},\mathcal{G}^{k+1}\}$, the optimization problem (8) with respect to ${U}$, ${V}$, ${W}$ and $\mathcal{G}$ is formulated as follows (i.e., (23) in the paper):
\begin{equation}
\begin{split}
&\min_{\mathcal{G},U,V,W}\frac{1}{2}\|\mathcal{Z}^{k}-\mathcal{G}\!\times_{1}\!U\!\times_{2}\!V\!\times_{3}\!W+\mathcal{Y}^{k}/\rho^{k}\|^{2}_{F},\\
&\;\;\;\;\textup{s.t.},\,U^{T}U=I_{d_{1}},\,V^{T}V=I_{d_{2}},\,W^{T}W=I_{d_{3}}.
\end{split}
\end{equation}

Analogous with Theorem 3 and Theorem 4.2 in~\cite{athauwer:hooi}, we first state that the minimization problem (56) can be formulated as Theorem 7 below:

\begin{theorem}
Given the real third-order tensors $\mathcal{Z}^{k}$ and $\mathcal{Y}^{k}$, then the minimization problem (56) is equivalent to the maximization (over the matrices $U$, $V$ and $W$ having orthonormal columns) of the following function
\vspace{-1mm}
\begin{equation}
g(U,\,V,\,W)=\|(\mathcal{Z}^{k}\!+\!\mathcal{Y}^{k}/\rho^{k})\!\times_{1}\!U^{T}\!\!\times_{2}\!V^{T}\!\!\times_{3}\!W^{T}\|^{2}_{F}.
\end{equation}
\end{theorem}

\begin{IEEEproof}
According to Theorem 4.1 in~\cite{athauwer:hooi}, for any given matrices $U$, $V$ and $W$, the optimal core tensor $\mathcal{G}$ is given by
\begin{equation*}
\mathcal{G}=(\mathcal{Z}^{k}\!+\!\mathcal{Y}^{k}/\rho^{k})\times_{1}\!U^{T}\!\times_{2}\!V^{T}\!\times_{3}\!W^{T}.
\end{equation*}
Thus, we have
\begin{equation*}
\begin{split}
f(U,V,W):=&\frac{1}{2}\|\mathcal{Z}^{k}-\mathcal{G}\!\times_{1}\!U\!\!\times_{2}\!V\!\!\times_{3}\!W+\mathcal{Y}^{k}/\rho^{k}\|^{2}_{F}\\
=&\frac{1}{2}\|\mathcal{Z}^{k}\!+\!\mathcal{Y}^{k}/\rho^{k}\|^{2}_{F}
-\langle\mathcal{Z}^{k}\!+\!\mathcal{Y}^{k}/\rho^{k},\,\mathcal{G}\!\times_{1}\!U\!\!\times_{2}\!V\!\!\times_{3}\!W\rangle+\frac{1}{2}\|\mathcal{G}\!\times_{1}\!U\!\!\times_{2}\!V\!\!\times_{3}\!W\|^{2}_{F}\\
^{a}\!\!=&\frac{1}{2}\|\mathcal{Z}^{k}\!+\!\mathcal{Y}^{k}/\rho^{k}\|^{2}_{F}-\langle(\mathcal{Z}^{k}\!+\!\mathcal{Y}^{k}/\rho^{k})\!\times_{1}\!U^{T}\!\!\times_{2}\!V^{T}\!\!\times_{3}\!W^{T},\,\mathcal{G}\rangle
+\frac{1}{2}\|\mathcal{G}\|^{2}_{F}\\
=&\frac{1}{2}\|\mathcal{Z}^{k}\!+\!\mathcal{Y}^{k}/\rho^{k}\|^{2}_{F}-\frac{1}{2}\|(\mathcal{Z}^{k}\!+\!\mathcal{Y}^{k}/\rho^{k})\!\times_{1}\!U^{T}\!\!\times_{2}\!V^{T}\!\!\times_{3}\!W^{T}\|^{2}_{F},
\end{split}
\end{equation*}
where $\|\mathcal{Z}^{k}\!+\!\mathcal{Y}^{k}/\rho^{k}\|^{2}_{F}$ is a constant, $g(U,V,W)=\|(\mathcal{Z}^{k}\!+\!\mathcal{Y}^{k}/\rho^{k})\!\times_{1}\!U^{T}\!\!\times_{2}\!V^{T}\!\!\times_{3}\!W^{T}\|^{2}_{F}$, and the equality $^{a}\!\!=$ relies on the facts that $U$, $V$ and $W$ have orthonormal columns.
\end{IEEEproof}
~\\

Imagine that the matrices $V$ and $W$ are fixed and that the optimization problem (57) is merely a quadratic function of the unknown matrix $U$. Thus, we have
\begin{equation}
\max_{U,\,U^{T}U=I_{d_{1}}}\|\mathcal{H}^{k}_{1}\!\times_{1}\!U^{T}\|^{2}_{F}=\|(\mathcal{H}^{k}_{1})^{T}_{(1)}U\|^{2}_{F},
\end{equation}
where $\mathcal{H}^{k}_{1}\!=\!(\mathcal{Z}^{k}\!+\!\mathcal{Y}^{k}/\rho^{k})\!\!\times_{2}\!(V^{k})^{T}\!\!\times_{3}\!\!(W^{k})^{T}$.

From (58), we can know that each column of the desire variable ${U}^{k+\!1}$ is the orthogonal basis for the dominant subspace of $(\mathcal{H}^{k}_{1})_{(1)}$. Hence, we use the partial singular value decomposition to compute ${U}^{k+\!1}$ as follows:
\begin{equation}
U^{k+\!1}=\textup{SVDs}((\mathcal{H}^{k}_{1})_{(1)}, d_{1}),
\end{equation}
where the columns of $U^{k+\!1}$ are the left singular vectors of $(\mathcal{H}^{k}_{1})_{(1)}$ corresponding to the top $d_{1}$ singular values.
Repeating the above procedure for $V$ and $W$, we have
\begin{equation}
V^{k+\!1}=\textup{SVDs}((\mathcal{H}^{k}_{2})_{(2)}, d_{2})\;\;\textup{and}\;\;W^{k+\!1}=\textup{SVDs}((\mathcal{H}^{k}_{3})_{(3)}, d_{3}),
\end{equation}
where $\mathcal{H}^{k}_{2}\!=\!(\mathcal{Z}^{k}\!+\!\mathcal{Y}^{k}/\rho^{k})\!\times_{1}\!(U^{k+\!1})^{T}\!\!\times_{3}\!(W^{k})^{T}$ and $\mathcal{H}^{k}_{3}\!=\!(\mathcal{Z}^{k}\!+\!\mathcal{Y}^{k}/\rho^{k})\!\times_{1}\!(U^{k+\!1})^{T}\!\!\times_{2}\!(V^{k+\!1})^{T}$.\\

Given the updated matrices ${U}^{k+1}$, ${V}^{k+1}$ and ${W}^{k+1}$, $\mathcal{G}^{k+1}$ is updated by
\begin{equation}
\mathcal{G}^{k+1}=\mathcal{H}^{k}_{3}\!\times_{3}\!(W^{k+\!1})^{T}.
\end{equation}
\\

To update $\mathcal{Z}^{k+1}$, the optimization problem (8) with respect to $\mathcal{Z}$ is formulated as follows (i.e., (24) in the paper):
\begin{equation}
\min_{\mathcal{Z}}\frac{1}{2}\|\mathcal{W}\ast(\mathcal{Z}-\mathcal{T})\|^{2}_{F}+\frac{\rho^{k}}{2}\|\mathcal{Z}\!-\!\mathcal{G}^{k+\!1}\!\!\times_{1}\!U^{k+\!1}\!\!\times_{2}\!V^{k+\!1}\!\!\times_{3}\!W^{k+\!1}\!\!+\!\mathcal{Y}^{k}/\rho^{k}\|^{2}_{F}.
\end{equation}
Since (62) is a least squares problem, and its optimal solution is given by
\begin{equation}
\mathcal{Z}^{k+1}\!=\!\mathcal{P}_{\Omega}\!\left(\frac{\mathcal{T}+\rho^{k}\mathcal{G}^{k+\!1}\!\!\times_{1}\!U^{k+\!1}\!\!\times_{2}\!V^{k+\!1}\!\!\times_{3}\!W^{k+\!1}-\mathcal{Y}^{k}}{1+\rho^{k}}\right)+\mathcal{P}^{\perp}_{\Omega}\!\left(\mathcal{G}^{k+\!1}\!\!\times_{1}\!U^{k+\!1}\!\!\times_{2}\!V^{k+\!1}\!\!\times_{3}\!W^{k+\!1}\!-\frac{\mathcal{Y}^{k}}{\rho^{k}}\right).
\end{equation}
~\\

\subsection*{Detailed Complexity Analysis}
In this part, we present the detailed complexity analysis of Algorithm 1 and Algorithm 4, as listed in Table 4.

\begin{table}[t]
\renewcommand{\arraystretch}{1.5}
\centering
\caption{Detailed complexity analysis of Algorithm 1 and Algorithm 4.}
\label{tab1}
\begin{tabular} {c|c|c|c|c|c}
\hline
\multicolumn{3}{c|}{Algorithm 1}  & \multicolumn{3}{c}{Algorithm 4}\\
\hline\hline
(18) & $\!\mathcal{X}^{k}\!\times_{2}\!(V^{k})^{T}\!\!\times_{3}\!(W^{k})^{T}$   &$O\!\left(d_{2}\Pi_{j}I_{j}\right)$ &(59)  & $\!(\mathcal{Z}^{k}\!+\!\mathcal{Y}^{k}\!/\!\rho^{k})\!\!\times_{2}\!(V^{k})^{T}\!\!\times_{3}\!\!(W^{k})^{T}$  & $O\!\left(d_{2}\Pi_{j}I_{j}\right)$\\
\hline
\multirow{2}{*}{(19)} & {$\!\mathcal{X}^{k}\!\times_{1}\!(U^{k+1})^{T}\!\times_{3}\!(W^{k})^{T}$} & \multirow{2}{*}{$O\!\left(2d_{1}\Pi_{j}I_{j}\right)$} & \multirow{2}{*}{(60)} & {$\!\!(\mathcal{Z}^{k}\!+\!\mathcal{Y}^{k}\!/\!\rho^{k})\!\times_{1}\!\!(U^{k\!+\!1})^{T}\!\!\times_{3}\!\!(W^{k})^{T}\!$} & \multirow{2}{*}{$O\!\left(2d_{1}\Pi_{j}I_{j}\right)$} \\
\cline{2-2}\cline{5-5}
& {$\!\mathcal{X}^{k}\!\times_{1}\!(U^{k\!+\!1})^{T}\!\times_{2}\!(V^{k\!+\!1})^{T}$} & & &  {$\!\!(\mathcal{Z}^{k}\!\!+\!\mathcal{Y}^{k}\!/\!\rho^{k})\!\times_{1}\!\!(U^{k\!+\!1})^{T}\!\!\!\times_{2}\!\!(V^{k\!+\!1})^{T}$}\!\!\\
\hline
\!\!(18,19)\!\! & $\textup{ORT}\left((\mathcal{M}^{k}_{n})_{(n)}\mathcal{B}^{T}_{(n)}\right)$ & \!\!$O\!\left(\sum_{n}\! I_{n}(\Pi_{j}d_{j}\!+\!d^{2}_{n})\right)$\!\!\! &\!\!(59,60)\!\! & $\textup{SVDs}((\mathcal{H}^{k}_{n})_{(n)}, d_{n})$ & $\!\!O\!\left(\sum_{n}\!(\min\{I_{n},\prod_{j\neq n}\!d_{j}\})^{2}d_{n}\!\right)\!\!$ \\
\hline
\!\!(20)\!\! & $\mathcal{M}^{k}_{3}\!\times_{3}\!(W^{k+1})^{T}$ & $O\!\left(I_{3}\Pi_{j}d_{j}\right)$ &\!\!(61)\!\! & $\mathcal{H}^{k}_{3}\!\times_{3}\!(W^{k+\!1})^{T}$ & $O\!\left(I_{3}\Pi_{j}d_{j}\right)$ \\
\hline
(22) & \!\!$\mathcal{G}^{k\!+\!1}\!\!\times_{1}\!\!U^{k\!+\!1}\!\!\times_{2}\!\!V^{k\!+\!1}\!\!\times_{3}\!\!W^{k\!+\!1}$\!\! & $O\!\left(d_{3}\Pi_{j}I_{j}\right)$ &(63) & $\mathcal{G}^{k+\!1}\!\!\times_{1}\!U^{k+\!1}\!\!\times_{2}\!V^{k+\!1}\!\!\times_{3}\!W^{k+\!1}$ & $O\!\left(d_{3}\Pi_{j}I_{j}\right)$ \\
\hline
(13) & $\textup{SVT}_{1/(3\lambda\rho^{k})}(\mathcal{G}^{k}_{(n)}\!+\!Y^{k}_{n}\!/\!\rho^{k})$ & $O\!\left(\sum_{n}\!{d^{2}_{n}}{\Pi_{j\neq{n}}}d_{j}\right)$ & & & \\
\hline
\multicolumn{3}{c|}{$O\!\left((2d_{1}\!+\!d_{2}\!+\!d_{3})\Pi_{j}I_{j}\right)$} & \multicolumn{3}{c}{$O\!\left((2d_{1}\!+\!d_{2}\!+\!d_{3})\Pi_{j}I_{j}\right)$}\\
\hline
\end{tabular}
\end{table}

\end{document}